\documentclass[11pt,a4paper]{article}

\usepackage[T1]{fontenc}
\usepackage[utf8]{inputenc}
\usepackage{lmodern} 
\usepackage[margin=1in]{geometry}
\usepackage{microtype}
\usepackage{newpxtext,newpxmath}

\usepackage[most]{tcolorbox}

\newtcolorbox{resultbox}{
  colback=gray!4,
  colframe=black!55,
  boxrule=0.5pt,
  arc=2pt,
  left=3pt,
  right=3pt,
  top=2pt,
  bottom=2pt,
  before skip=8pt,
  after skip=8pt
}

\usepackage{wrapfig}
\usepackage{caption}

\usepackage{amsmath,mathtools,bm}
\usepackage{amsthm}            

\newcommand{\act}{\mathcal{A}}
\newcommand{\defer}{\bot}
\newcommand{\absent}{0}
\newcommand{\present}{1}
\newcommand{\etaLoc}{\boldsymbol{\eta}}
\newcommand{\muGlob}{\boldsymbol{\mu}}
\newcommand{\compat}[1]{\mathsf{Comp}(#1)}
\DeclareMathOperator{\Softmax}{Softmax}
\newtheorem{proposition}{Proposition}

\usepackage{graphicx}
\graphicspath{{images/}} 
\usepackage{booktabs,threeparttable,tabularx,array,colortbl,multirow}
\usepackage{xcolor}
\definecolor{lightgrayrow}{gray}{0.93}

\definecolor{inc_red}{RGB}{255, 230, 230}
\definecolor{epi_org}{RGB}{255, 245, 230}
\definecolor{coh_grn}{RGB}{235, 250, 235}
\definecolor{op_blu}{RGB}{235, 242, 255}

\usepackage{tikz}
\usetikzlibrary{arrows.meta,positioning,calc,shadows.blur}
\usepackage[edges]{forest}
\definecolor{L2DRoot}{HTML}{1B4965}
\definecolor{L2DBranch}{HTML}{E7F0FA}
\definecolor{L2DLeaf}{HTML}{F6FAFF}
\colorlet{L2DLine}{L2DRoot}
\forestset{
  l2dTaxonomy/.style={
    for tree={
      grow'=east, draw=L2DLine, edge={-Stealth,L2DLine,line width=.8pt},
      rounded corners=2pt, align=center, font=\sffamily\small,
      inner xsep=4pt, inner ysep=4pt, l sep=12mm, s sep=2.5mm,
      parent anchor=east, child anchor=west, anchor=west,
      minimum height=7mm, text width=34mm
    },
    where level=0{fill=L2DRoot, text=white, font=\sffamily\bfseries\footnotesize, text width=27mm}{},
    where level=1{fill=L2DBranch, font=\sffamily\bfseries\footnotesize, text width=40mm}{},
    where level>=2{fill=L2DLeaf, font=\sffamily\footnotesize, text width=36mm}{},
  },
}

\usepackage[most]{tcolorbox}
\tcbset{boxrule=0.6pt,arc=2pt}

\usepackage{enumitem}
\setlist{leftmargin=*, itemsep=0.25em, topsep=0.25em}

\usepackage{bbm}    
\usepackage{pifont} 

\usepackage[numbers,sort&compress]{natbib}          
\usepackage{xcolor}
\definecolor{linkblue}{HTML}{0A66C2} 
\usepackage{hyperref}
\hypersetup{
  colorlinks=true,
  linkcolor=linkblue,   
  citecolor=linkblue,   
  urlcolor=linkblue,    
  filecolor=linkblue
}

\usepackage{caption}
\usepackage{subcaption}
\newif\ifblindreview
\blindreviewfalse 

\IfFileExists{orcidlink.sty}{\usepackage{orcidlink}}{%
  \newcommand{\orcidlink}[1]{\href{https://orcid.org/#1}{\textsuperscript{\tiny ORCID}}}
}

\usepackage{authblk}

\title{\vspace{-8pt}\textbf{Coherent Hierarchical Multi-Label Learning to Defer for Medical Imaging}\vspace{4pt}}
\ifblindreview
  \author{\textbf{Anonymous authors}}
  \affil{}
\else
  \author[1*]{Joshua Strong\,\orcidlink{0009-0009-2348-1955}}
  \author[1]{Pramit Saha\,\orcidlink{0009-0006-8090-1969}}
  \author[1]{Emma Sun\,\orcidlink{0009-0001-3816-5391}}
  \author[2]{Helen Higham\,\orcidlink{0000-0001-5796-0377}}
  \author[1]{J.\ Alison Noble\,\orcidlink{0000-0002-3060-3772}}
  \affil[1]{Department of Engineering Science, University of Oxford, UK \\}
  \affil[2]{Nuffield Division of Anaesthetics, University of Oxford, The John Radcliffe Hospital, UK}
  \affil[*]{Correspondence: \texttt{joshua.strong@eng.ox.ac.uk}}
  \affil[ ]{Project Webpage: \url{https://josh-strong.github.io/coherent-l2d-site/}}
\fi
\date{\vspace{-8pt}\small Preprint — \today}

\newcommand{\keywords}[1]{%
  \vspace{0.5em}\noindent\textbf{Keywords:} #1
}

\theoremstyle{plain}
\newtheorem{theorem}{Theorem}

\theoremstyle{definition}
\newtheorem{definition}{Definition}

\theoremstyle{remark}
\newtheorem{remark}{Remark}
\newtheorem{corollary}[theorem]{Corollary}

\usepackage{array,tabularx,threeparttable}

\newcolumntype{L}[1]{>{\raggedright\arraybackslash}m{#1}}
\newcolumntype{Y}{>{\raggedright\arraybackslash}X}

\definecolor{inc_red_text}{RGB}{145,48,48}
\definecolor{epi_org_text}{RGB}{142,97,22}
\definecolor{coh_grn_text}{RGB}{33,97,57}
\definecolor{op_blu_text}{RGB}{41,78,132}

\newcommand{\StatusTax}{\cellcolor{inc_red}\strut\textbf{\textcolor{inc_red_text}{Taxonomic contradiction}}}
\newcommand{\StatusDed}{\cellcolor{epi_org}\strut\textbf{\textcolor{epi_org_text}{Deductive defect}}}
\newcommand{\StatusCoh}{\cellcolor{coh_grn}\strut\textbf{\textcolor{coh_grn_text}{Coherent}}}
\newcommand{\StatusDel}{\cellcolor{op_blu}\strut\textbf{\textcolor{op_blu_text}{Delegation violation}}}

\usepackage[capitalize,noabbrev]{cleveref}
\usepackage{makecell}
\Crefname{equation}{Eq.}{Eqs.}

\begin{document}
\maketitle

\begin{abstract}
\noindent Learning to Defer (L2D) enables a model to predict autonomously or defer to an expert, but prior work largely assumes flat label spaces. We study the first L2D setting with hierarchical multi-label decisions, motivated by medical-imaging workflows in which findings are organised by clinical taxonomies. In this setting, deferral is a delegation action rather than a label assignment, so treating it as an independent per-label decision can produce \emph{deferral incoherence}, including taxonomic contradictions, delegation violations, and deferrals of labels already implied by the model’s own assertions. We formalise coherent hierarchical deferral under a Selective-Exclusion handoff contract, characterise the Bayes-optimal coherent deferral rule, and show that even nodewise Bayes L2D can be action-incoherent. We then propose two remedies: exact coherent projection, a dynamic-programming decoder over the coherent action set, and Taxonomic Belief Propagation (TBP) with Recursive Policy Optimisation (RPO), a contract-aware joint action model trained through the same recursion used at inference. Across real-reader and controlled-expert medical-imaging benchmarks, naïve binary-relevance L2D exhibits non-trivial incoherence. Projection removes it exactly, and fast TBP+RPO drives incoherence near zero while retaining strong utility.
\end{abstract}


\section{Introduction}

Deep learning performs strongly on many medical-imaging tasks, but is typically considered as an assistive tool rather than an autonomous decision-maker \cite{allen20212020,european2022current,eltawil2023analyzing}. Learning to Defer (L2D) captures this human--AI setting by allowing a model either to predict or to defer to an expert \cite{mozannar2020consistent}. Existing L2D methods, however, largely assume flat label spaces, where labels are treated independently and prediction/deferral decisions are made per label without hierarchical constraints \cite{Strong2025-rr}. Medical-imaging decisions are often hierarchical and multi-label: broad findings such as \emph{Lung Opacity} relate taxonomically to more specific subtypes, and multiple findings may co-occur \cite{Langlotz2006-wv,bos2006snomed,irvin2019chexpert}. We study, to our knowledge, the first L2D formulation in which deferral is a coherent action over a hierarchical multi-label taxonomy rather than an independent reject option for each label. This exposes a failure mode invisible to both flat L2D and standard hierarchical multi-label classification (HMLC): \textbf{a system may be label-satisfiable yet still yield an incoherent delegation to the human}.

The key new phenomenon is \emph{deferral incoherence}. A naïve but natural first baseline applies Binary Relevance (BR) \cite{zhang2018binary}, treating each node as an independent defer-or-predict problem. In a taxonomy, however, independent actions can be \emph{incoherent}. A model may assert a child while ruling out its parent, defer a parent while autonomously asserting a child, or defer a node whose value is already entailed by an asserted ancestor. These are not simply label-consistency errors,  but failures of the human--AI handoff. For a clinical reader, such action vectors can be confusing even when individual nodewise predictions appear reasonable, because they obscure which findings the model is asserting autonomously and which findings it is handing off for review.

We formalise coherent hierarchical deferral under a \emph{Selective-Exclusion contract}, motivated by assistive imaging workflows in which deferring an internal finding transfers responsibility for that finding to a clinical reader. We use ``clinical reader'' to denote the human responsible for the imaging decision in the workflow; in practice, a radiologist or other clinician. We develop two remedies: exact coherent projection, an inference-only dynamic-programming decoder, and Taxonomic Belief Propagation (TBP) with Recursive Policy Optimisation (RPO), a contract-aware joint action model. We evaluate on VinDr-CXR and CheXpert as primary real-reader CXR evidence, with PadChest and ADPv2 as controlled-expert stress tests for taxonomy scale and modality.

\textbf{Relation to nearby settings.}
Flat L2D studies when to defer but not how deferral interacts with a taxonomy. HMLC studies hierarchical label consistency but not expert handoff actions. Hierarchical selective classification allows abstention or backoff in a hierarchy, but typically in single-label settings and without modelling expert-provided labels on deferred nodes. Our setting combines all three ingredients: multi-label hierarchy, ternary prediction/deferral actions, and system utility under expert delegation. An extended related literature review is given in Appendix \ref{extended_lit}.

\textbf{Contributions.}
First, we introduce hierarchical multi-label L2D as a structured human--AI handoff problem, where deferral is a delegation action over a taxonomy rather than an independent third label value. Second, we formalise \emph{deferral incoherence} through taxonomic satisfiability, workflow admissibility, and deductive closure, and characterise the Bayes-optimal coherent deferral rule. This yields a key negative result: even oracle nodewise Bayes L2D can produce delegation violations and deductive defects, so incoherence is not just a training artefact. Third, we give two contract-aware solutions: exact coherent projection, which guarantees zero incoherence at inference, and TBP+RPO, which learns a coherent-support joint action model through the deployed hierarchy recursion. Across four medical-imaging benchmarks, BR-L2D is consistently incoherent, projection removes incoherence exactly, and fast TBP+RPO gives a near-coherent high-utility operating point.

\section{Background}
\label{sec:background}
Clinical vocabularies and reporting taxonomies organise imaging findings into broad concepts and more specific subtypes. HMLC abstracts this structure as a taxonomy $\mathcal T=(\mathcal V,\mathcal E)$, where $\mathcal V$ are labels (nodes) and $\mathcal E$ are parent--child relations. Let $y_j\in\{0,1\}$ denote the label at node $j$, and let $pa(j)$ denote its parent. For each edge $(pa(j)\to j)$, $y_j=1\Rightarrow y_{pa(j)}=1$. Thus a positive child implies a positive parent, and an absent parent implies absent descendants. HMLC methods enforce or exploit these constraints over label values \cite{chen2019deep,giunchiglia2020coherent}.

L2D instead concerns \emph{actions}. For input $x$, the system either predicts autonomously or defers to an expert, and is evaluated by system risk \cite{madras2018predict,mozannar2020consistent}. A natural first hierarchical multi-label solution we propose is via Binary Relevance (BR): each node $j\in\mathcal V$ is treated as an independent ternary decision with action space $\mathcal A=\{\absent,\present,\defer\}$ where $0$ and $1 $ are autonomous predictions and $\defer$ represents deferral. With local action distribution $\pi_j(\cdot\mid x)$ and any single-concept L2D surrogate $\mathcal L_\phi$, BR optimises
\vspace{-1pt}
\begin{equation}
\label{eq:j_br_bg}
\mathcal J_{\mathrm{BR}}(x,y,m)
=
\sum\limits_{j\in\mathcal V} \,\,
\mathcal L_\phi\big(\pi_j(\cdot\mid x), y_j, m_j\big).
\end{equation}
\vspace{-2pt}
This reduction is scalable but ignores ancestor--descendant coupling between actions, which is precisely where deferral incoherence arises. We discuss this next.

\section{Deferral Coherence}
\label{sec:theory}

Once internal nodes may be deferred, label consistency is insufficient: actions must be compatible with the taxonomy, admissible as a handoff, and closed under the model's own logical implications. We formalise these in this Section. Appendix~\ref{app:intuitive_coherence} gives an informal walk-through to assist the reader.

\subsection{Setup and Label-Space Assumptions}
\label{sec:theory_setup}

We consider a tree taxonomy $\mathcal T=(\mathcal V,\mathcal E)$ with parent map $pa(\cdot)$, children $C(p)$ for each node $p$, and node-wise binary labels $y_v\in\{0,1\}$. Figure~\ref{fig:lung_opacity_subtree} shows a schematic opacity subtree used for examples.

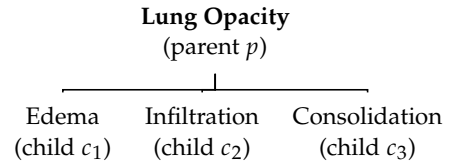
\begin{wrapfigure}{r}{0.36\textwidth}
\vspace{-2em}
\centering
\begin{forest}
for tree={
    draw=none,
    font=\footnotesize,
    edge={semithick},
    align=center,
    grow=south,
    parent anchor=south,
    child anchor=north,
    l sep=7pt,
    s sep=5pt,
    edge path={
      \noexpand\path[\forestoption{edge}]
        (!u.parent anchor) -- ++(0,-6pt) -| (.child anchor)\forestoption{edge label};
    },
}
[{\textbf{Lung Opacity}\\(parent $p$)}
  [{Edema\\(child $c_1$)}]
  [{Infiltration\\(child $c_2$)}]
  [{Consolidation\\(child $c_3$)}]
]
\end{forest}
\caption{Illustrative subtree used throughout Section~\ref{sec:theory}.}
\label{fig:lung_opacity_subtree}
\end{wrapfigure}

The label space obeys the standard hierarchy constraint of \textbf{upward implication}: for every edge $(p\to c)\in\mathcal E$,
\begin{equation}
\label{eq:upward}
y_c=1 \implies y_p=1.
\end{equation}
Its contrapositive gives \textbf{downward implication}: $y_p=0 \implies y_c=0.$

We also adopt the \textbf{open-world assumption}
\begin{equation}
\label{eq:openworld}
y_p=1 \nRightarrow \exists c\in C(p)\text{ such that }y_c=1.
\end{equation}
Thus, a positive parent need not be fully explained by the listed children.

\subsection{Taxonomic Satisfiability}
\label{sec:satisfiable}

In L2D, each node is a defer-or-predict decision: the model outputs
$a_v \in \act := \{\absent,\present,\defer\}$
for each $v\in\mathcal V$, where $\absent$ and $\present$ are autonomous predictions and $\defer$ hands the decision to the expert. Thus $\defer$ is not a third value of $y_v$; it leaves the binary label unspecified by the model.

\begin{definition}[Compatibility set and taxonomic satisfiability]
\label{def:compatibility}
Let $y=(y_v)_{v\in\mathcal V}$ be a binary label vector, and let
\[
\mathcal Y_{\mathcal T}
:=
\left\{
y\in\{0,1\}^{|\mathcal V|} :
\forall (p\to c)\in\mathcal E,\; y_c=1 \Rightarrow y_p=1
\right\}
\]
be the taxonomy-consistent label space under upward implication (\cref{eq:upward}). For actions $a\in\act^{|\mathcal V|}$, define
\[
\compat{a}
:=
\left\{
y\in\mathcal Y_{\mathcal T} :
\forall v\in\mathcal V,\; a_v\neq\defer \Rightarrow y_v=a_v
\right\}.
\]
Thus $\compat{a}$ is the set of hierarchy-consistent labelings still possible after respecting every non-deferred model assertion. We call $a$ \emph{taxonomically satisfiable} if $\compat{a}\neq\emptyset$.
\end{definition}

\begin{definition}\colorbox{inc_red}{(Taxonomic contradiction)}
\label{def:tax_contra}
An action vector $a$ is a \emph{taxonomic contradiction} if $\compat{a}=\emptyset$.
\end{definition}

For example, \textit{Lung Opacity} absent with \textit{Edema} present is a taxonomic contradiction: the child assertion requires the parent to be present under \cref{eq:upward}.

For a parent and its children, satisfiability reduces to a simple local test:
\begin{proposition}[Local satisfiability criterion]
\label{prop:local_satisfiable}
Let $p$ be a parent node with children $C(p)$. A local configuration $(a_p,\{a_c\}_{c\in C(p)})$ is taxonomically satisfiable if and only if it does \emph{not} assert a positive child under an asserted negative parent: $ \neg\Big(a_p=\absent \ \wedge\ \exists c\in C(p): a_c=\present\Big).$
\end{proposition}

\subsection{Workflow Admissibility: Deferring a Parent}
\label{sec:contract}

Taxonomic satisfiability is necessary but not sufficient: we also need a rule for what the model may do when it defers an internal concept. We adopt a contract motivated by assistive reading workflows and refined with clinical input from a coauthor.

\begin{definition}[Selective-Exclusion deferral contract]
\label{def:selective_exclusion}
Under \emph{Selective-Exclusion} (SE), if the system defers a parent concept, it may exclude specific subtypes but may not assert any subtype as present. Formally, for any edge $(p\to c)$,
$a_p=\defer \Rightarrow a_c\in\{\absent,\defer\}.$
Together with taxonomic satisfiability, this implies that any positive child assertion must have a positive parent assertion.
\end{definition}

This encodes a handoff semantics: deferring a parent delegates that finding to the clinical reader, while a positive child assertion already commits the parent to being present. Negative child assertions remain admissible under the open-world assumption.

\begin{definition}\colorbox{op_blu}{(Delegation / contract violation)}
\label{def:delegation_violation}
A local configuration $(a_p,\{a_c\}_{c\in C(p)})$ is a \emph{delegation violation} under SE if $a_p=\defer \ \wedge\ \exists c\in C(p): a_c=\present.$
\end{definition}

For example, deferring \textit{Lung Opacity} while asserting \textit{Consolidation} present is a delegation violation, since the child already commits the deferred parent.

\begin{remark}[Alternative semantics: strong subtree handoff]
\label{rmk:a}
SE is not the only possible contract. A stricter alternative enforces \emph{subtree handoff}: if a parent is deferred, all descendants must also be deferred. This may suit workflows with full responsibility transfer, but is more restrictive. We derive it in Appendix~\ref{app:semA} and evaluate it in Appendix~\ref{sensitivity_analysis}.
\end{remark}

\subsection{Deductive Closure and Deductive Defects}
\label{sec:deductive}

Even a taxonomically satisfiable, contract-admissible configuration may be defective if it defers a label already determined by the model's own assertions.

\begin{definition}[Taxonomic entailment from actions]
\label{def:entailment}
For an action vector $a$ and node $v\in\mathcal V$, write
$a \models_{\mathcal T} (v=b)$ iff
$\forall y\in\compat{a},\; y_v=b$.
Thus $(v=b)$ is \emph{entailed} when every taxonomy-consistent completion of the model's assertions gives node $v$ value $b$.
\end{definition}

\begin{definition}[Deductive closure]
\label{def:deductive_closure}
An action vector $a$ is \emph{deductively closed} if every entailed label value is explicitly asserted:
$a \models_{\mathcal T} (v=b)
\Rightarrow a_v=b$ for all $v\in\mathcal V$ and $b\in\{0,1\}$.
\end{definition}

\begin{definition}\colorbox{epi_org}{(Deductive defect)}
\label{def:deductive_defect}
A configuration is a \emph{deductive defect} if it is taxonomically satisfiable and contract-admissible, but fails to be deductively closed.
\end{definition}

The key practical case follows directly from downward implication.
\begin{corollary}[Downward implication yields deductive defects]
\label{cor:downward_contradiction}
Suppose $a_p=\absent$ for a parent node $p$, and the action configuration is taxonomically satisfiable. Then for every child $c\in C(p)$, $a \models_{\mathcal T}(c=\absent).$
Consequently, any configuration with $a_p=\absent$ and some $a_c=\defer$ is a deductive defect.
\end{corollary}

For example, asserting \textit{Lung Opacity} absent while deferring \textit{Infiltration} is a deductive defect, because the parent-absent assertion already entails that \textit{Infiltration} is absent.

\subsection{Complete Local Classification of the Decision Space}
\label{sec:classification}

\begin{table*}[h]
\centering
\caption{Complete local classification of parent--children handoff decisions under Selective-Exclusion. Within each parent block, the \textbf{Status} column reports the first failed coherence criterion: taxonomic satisfiability, contract admissibility, or deductive closure.}
\label{tab:coherence}
\footnotesize
\setlength{\tabcolsep}{4pt}
\renewcommand{\arraystretch}{0.96}
\begin{threeparttable}
\begin{tabularx}{\textwidth}{@{}l l Y@{}}
\toprule
\textbf{Child pattern \(\mathbf{a}_C\)} & \textbf{Status} & \textbf{Interpretation} \\
\midrule
\multicolumn{3}{@{}l}{\textbf{Parent absent} (\(a_p=\absent\))} \\
\(\exists c:\, a_c=\present\) & \StatusTax & Positive child under asserted negative parent \\
\(\forall c:\, a_c\neq\present,\ \exists c:\, a_c=\defer\) & \StatusDed & Deferred child already entailed absent \\
\(\forall c:\, a_c=\absent\) & \StatusCoh & Complete, deductively closed exclusion of the subtree \\
\multicolumn{3}{@{}l}{\textbf{Parent present} (\(a_p=\present\))} \\
\(\exists c:\, a_c=\present\) & \StatusCoh & Consistent subtype assertion under a positive parent \\
\(\forall c:\, a_c=\absent\) & \StatusCoh & Open-world case: unmodelled subtypes may explain the parent \\
\(\forall c:\, a_c\neq\present,\ \exists c:\, a_c=\defer\) & \StatusCoh & Selective delegation at the subtype level \\
\multicolumn{3}{@{}l}{\textbf{Parent deferred} (\(a_p=\defer\))} \\
\(\exists c:\, a_c=\present\) & \StatusDel & Positive child commits deferred parent present \\
\(\forall c:\, a_c\in\{\absent,\defer\}\) & \StatusCoh & Parent deferred while children are only excluded or deferred \\
\bottomrule
\end{tabularx}
\end{threeparttable}
\end{table*}

The three coherence requirements give a complete local classification of parent--children handoff decisions under SE. We first define coherence globally, then give the edge-level characterisation used by our empirical incoherence metrics; Table~\ref{tab:coherence} expands the same logic to parent neighbourhoods.

\begin{definition}[Deferral coherence]
\label{def:coherence}
An action vector $a\in\act^{|\mathcal V|}$ is \emph{deferral coherent} under SE if it is (i) taxonomically satisfiable, (ii) contract-admissible, and (iii) deductively closed.
\end{definition}

\begin{definition}[Deferral incoherence]
\label{def:deferral_incoherence}
An action vector is \emph{deferral-incoherent} under a handoff contract if it is not deferral coherent. Under SE on a tree, Proposition~\ref{prop:local_characterization} shows that all local failures are captured by taxonomic contradictions, delegation violations, and deductive defects.
\end{definition}

\begin{proposition}[Local characterisation of deferral coherence]
\label{prop:local_characterization}
Assume a tree taxonomy and the open-world choice Eq.~\eqref{eq:openworld}. An action vector $a\in\act^{|\mathcal V|}$ is deferral coherent under SE iff every edge $(p\to c)$ avoids:
\begin{align*}
a_p=\absent \ \wedge\ a_c=\present
&\qquad \colorbox{inc_red}{\text{(taxonomic contradiction),}}\\
a_p=\defer \ \wedge\ a_c=\present
&\qquad \colorbox{op_blu}{\text{(delegation violation),}}\\
a_p=\absent \ \wedge\ a_c=\defer
&\qquad \colorbox{epi_org}{\text{(deductive defect).}}
\end{align*}
\end{proposition}

Thus immediate-edge checks are sufficient for global coherence in our tree setting: avoiding these patterns on every edge implies taxonomic satisfiability, contract admissibility, and deductive closure on the full taxonomy.

\subsection{Bayes-Optimal Coherent Deferral}
\label{sec:bayes_coherent}

This subsection makes precise why hierarchical L2D is not nodewise L2D applied independently. Flat L2D admits a pointwise Bayes rule: defer when the expert's conditional correctness exceeds the classifier's Bayes confidence \cite[Equation 6]{mozannar2020consistent}. Here the same conditional-risk principle applies, but actions must be coherent.

Let $Y_v$ and $M_v$ denote the ground-truth and expert labels at node $v$. With optional weights $w_v\ge0$ and deferral costs $\lambda_v\ge0$, define
\[
\begin{aligned}
\rho_v(\absent\mid x)&=w_vP(Y_v=1\mid x),&
\rho_v(\present\mid x)&=w_vP(Y_v=0\mid x),\\
\rho_v(\defer\mid x)&=w_v\{P(M_v\neq Y_v\mid x)+\lambda_v\}.&
\end{aligned}
\]
Without coherence constraints, the nodewise Bayes action is
$a_v^{\mathrm{BR},B}(x)\in\arg\min_{a\in\act}\rho_v(a\mid x)$. For \(w_v>0\), this rule defers when
\[
P(M_v=Y_v\mid x)
\ge
\max\{P(Y_v=0\mid x),P(Y_v=1\mid x)\}+\lambda_v .
\]
When \(\lambda_v=0\), this reduces to the usual flat-L2D condition that the expert's conditional correctness exceeds the classifier's Bayes confidence.
Let $\mathcal C_{\mathrm{SE}}$ be the SE-coherent action set and
$\Gamma_{\mathrm{SE}}(\absent)=\{\absent\}$,
$\Gamma_{\mathrm{SE}}(\present)=\act$,
$\Gamma_{\mathrm{SE}}(\defer)=\{\absent,\defer\}$.

\begin{resultbox}
\textbf{Key result: Bayes-optimal coherent deferral.}
The Bayes action is the minimum-risk coherent action vector:
\begin{equation}
\label{eq:bayes_coherent}
a_{\mathrm{SE}}^B(x)
\in
\arg\min_{a\in\mathcal C_{\mathrm{SE}}}
\sum_{v\in\mathcal V}\rho_v(a_v\mid x).
\end{equation}
For a tree, this is computed by the subtree recursion
\[
B_v(i\mid x)
=
\min_{j\in\Gamma_{\mathrm{SE}}(i)}
\left[
\rho_v(j\mid x)+
\sum_{u\in C(v)}B_u(j\mid x)
\right],
\]
where $B_v(i\mid x)$ is the minimum risk over the subtree rooted at \(v\), including \(v\), given parent action \(i\); roots minimise over \(j\in\act\).
\end{resultbox}

Thus optimal decisions are \emph{not local}: an internal action trades off its own risk with downstream feasible-subtree value.

\begin{proposition}[Nodewise Bayes L2D can be deferral-incoherent]
\label{prop:br_bayes_incoherent}
For an edge $(p\to c)$ with $Y_c=1\Rightarrow Y_p=1$, consider the zero-cost case \(\lambda_p=\lambda_c=0\), assume \(w_p,w_c>0\), and assume no ties and let $\pi_v(x)=P(Y_v=1\mid x)$ and $q_v(x)=P(M_v=Y_v\mid x)$. The independent nodewise Bayes rule
\[
a_v^{\mathrm{BR},B}(x)
\in
\arg\max_{a\in\{\absent,\present,\defer\}}
\{1-\pi_v(x),\,\pi_v(x),\,q_v(x)\}
\]
cannot produce $a_p=\absent,a_c=\present$, but can produce both $a_p=\defer,a_c=\present$ and $a_p=\absent,a_c=\defer$. Hence even oracle-calibrated BR-L2D can be deferral-incoherent under SE. Proof is given in Appendix~\ref{app:proof_br_bayes_incoherent}.
\end{proposition}
 
\section{Contract-Aware Decoding and Training}
\label{sec:method}

Deferral coherence turns hierarchical L2D into a constrained action problem. We use two estimators: an exact decoder over the coherent action set, and a joint action model whose support is coherent by construction. Unlike standard HMLC corrections, which constrain binary label values, our constraints act on ternary handoff actions. This matters because a configuration can be label-satisfiable while still being workflow-inadmissible or deductively incomplete.

\textbf{Setup and local primitives.}
Each node $v\in\mathcal V$ takes action $a_v\in\act=\{\absent,\present,\defer\}$. The network first outputs a local three-way primitive distribution
\begin{equation}
\label{eq:local_primitive}
\etaLoc_v(a\mid x)
:=
\Softmax(f_\theta^{(v)}(x))_a,
\qquad
a\in\act.
\end{equation}
These primitives are local defer/predict preferences before ancestor-dependent feasibility constraints are imposed; by themselves, they do not define a coherent joint policy.

\subsection{Exact Coherent Projection}
\label{sec:projection}

BR+Projection is an inference-only structured baseline: training remains standard BR-L2D, and coherence is imposed at test time by exact constrained decoding. Let
\[
\Gamma_{\mathrm{SE}}(\absent)=\{\absent\},\qquad
\Gamma_{\mathrm{SE}}(\present)=\act,\qquad
\Gamma_{\mathrm{SE}}(\defer)=\{\absent,\defer\}.
\]
Then
\[
\mathcal C_{\mathrm{SE}}
=
\left\{
a \in \act^{|\mathcal V|} :
\forall (p \to c)\in\mathcal E,\; a_c \in \Gamma_{\mathrm{SE}}(a_p)
\right\}.
\]
Given local primitives, projection solves the coherent MAP problem
\begin{equation}
\label{eq:proj_full_map_main}
\hat{\mathbf a}^{\mathrm{Proj}}(x)
\in
\arg\max_{\mathbf a \in \mathcal C_{\mathrm{SE}}}
\sum_{v \in \mathcal V}\log \etaLoc_v(a_v \mid x).
\end{equation}
Thus projection is not a heuristic repair step: it is exact MAP decoding over the contract-defined coherent action space.

On a tree, \eqref{eq:proj_full_map_main} is solved by dynamic programming. With $C(v)$ denoting the children of $v$,
\[
F_v(i)=
\max_{j\in\Gamma_{\mathrm{SE}}(i)}
\left[
\log \etaLoc_v(j\mid x)
+
\sum_{u\in C(v)}F_u(j)
\right],
\]
with roots maximised over $j\in\act$ and standard backtracking recovering the action vector. A single decode costs $O(|\mathcal V||\act|^2)$, effectively linear since $|\act|=3$.

For budgeted evaluation, deferrals are ranked under the projection model rather than by local BR marginals. For node $t$ and action $a^\star$, define
\[
V_t(a^\star\mid x)
=
\max_{\mathbf a\in\mathcal C_{\mathrm{SE}}:\,a_t=a^\star}
\sum_{v\in\mathcal V}\log \etaLoc_v(a_v\mid x),
\]
and rank by
\[
s_{\mathrm{Proj}}(x,t)
=
V_t(\defer\mid x)
-
\max\{V_t(\absent\mid x),V_t(\present\mid x)\}.
\]
Selected deferrals are clamped and decoded by the same tree DP; feasibility closure and runtime details are in Appendix~\ref{app:exp_details_projection}. Projection therefore isolates the effect of exact contract enforcement for an already-trained local L2D model.

\subsection{Taxonomic Belief Propagation}
\label{sec:tbp}

Our second estimator builds the contract into the model. Taxonomic Belief Propagation (TBP) maps local primitives $\{\etaLoc_v\}$ to unconditional marginals $\{\muGlob_v\}$ through a tree-structured action model. The transition law is over actions, not labels: the parent action determines the admissible child actions.

For each non-root node $v$, define
\[
\mathbf T_v(x)\in[0,1]^{3\times 3},
\qquad
(\mathbf T_v)_{ij}
=
P(a_v=j \mid a_{pa(v)}=i, x),
\]
with rows and columns ordered as $[\absent,\present,\defer]$. Under SE,
\[
a_{pa(v)}=\absent \Rightarrow a_v=\absent,\qquad
a_{pa(v)}=\present \Rightarrow a_v\sim \etaLoc_v(\cdot\mid x),\qquad
a_{pa(v)}=\defer \Rightarrow a_v\in\{\absent,\defer\}.
\]
Thus
\begin{equation}
\label{eq:transition_kernel}
\mathbf T_v(x)=
\begin{bmatrix}
1 & 0 & 0 \\
\etaLoc_v(\absent\mid x) & \etaLoc_v(\present\mid x) & \etaLoc_v(\defer\mid x) \\
\alpha_v(x) & 0 & 1-\alpha_v(x)
\end{bmatrix},
\qquad
\alpha_v(x)
=
\frac{\etaLoc_v(\absent\mid x)}
{\etaLoc_v(\absent\mid x)+\etaLoc_v(\defer\mid x)}.
\end{equation}
The deferred-parent row renormalises the local primitive over the admissible face $\{\absent,\defer\}$; Appendix~\ref{norm_proof} shows this is the unique reverse-KL projection onto that face.

For roots, $\muGlob_r(x)=\etaLoc_r(\cdot\mid x)$; for non-roots,
\begin{equation}
\label{eq:tbp_recursion}
\muGlob_v(x)=\muGlob_{pa(v)}(x)^\top \mathbf T_v(x).
\end{equation}
Equivalently, TBP induces the joint action model
\begin{equation}
\label{eq:tbp_joint}
P_{\mathrm{TBP}}(\mathbf a\mid x)
=
\etaLoc_r(a_r\mid x)
\prod_{v\neq r}\mathbf T_v(x)_{a_{pa(v)},a_v}.
\end{equation}
The structural zeros in \eqref{eq:transition_kernel} assign forbidden parent--child action pairs zero probability, so the joint model has coherent support rather than relying on post-hoc repair. Appendix~\ref{app:theory} gives the normalisation, constraint-enforcement, monotonicity, and reverse-KL proofs; exact TBP MAP decoding appears in Appendix~\ref{app:tbp_exactdecode}.

\subsection{Why Hierarchical Deferral Is Not Separable}
\label{sec:nonseparable}

The Bayes recursion in Section~\ref{sec:bayes_coherent} shows that ancestor actions have downstream option value. Under SE, asserting a parent absent forces descendants absent, while deferring a parent forbids positive child assertions. Thus the locally Bayes-optimal parent action need not match the parent action in the Bayes-optimal coherent policy.

\begin{proposition}[Contract-induced non-separability]
\label{prop:nonseparable}
Under SE, there exist parent--child defer problems for which the locally Bayes-optimal parent action differs from the parent action induced by the Bayes-optimal coherent policy.
\end{proposition}

\noindent\emph{Proof sketch.}
Choose $x$ such that parent deferral has slightly lower local risk than parent-present, but parent-present unlocks a much lower-risk child-present action that is forbidden after parent deferral. Appendix~\ref{app:rpo_why_improves} gives a constructive example.

This is why independent nodewise argmax decisions are insufficient. Projection performs structured decoding over $\mathcal C_{\mathrm{SE}}$, while RPO trains local primitives through the same ancestor--descendant gating used at inference.

\subsection{Recursive Policy Optimisation}
\label{sec:rpo}

TBP changes the deployed policy: inference uses the composed marginals $\{\muGlob_v(x)\}$ rather than the local primitives $\{\etaLoc_v(x)\}$. Since SE makes hierarchical deferral non-separable, locally trained primitives need not optimise the final contract-aware policy.

We train in two stages. Stage I learns local primitives with the BR objective; Stage II activates TBP and fine-tunes through the composed marginals:
\begin{equation}
\label{eq:j_rpo_stages}
\begin{aligned}
\mathcal J_{\mathrm{StageI}}
&=
\sum_{v\in\mathcal V}
\mathcal L_\phi\big(\etaLoc_v(x), y_v, m_v\big),
\qquad
\mathcal J_{\mathrm{RPO}}
=
\sum_{v\in\mathcal V}
\mathcal L_\phi\big(\muGlob_v(x), y_v, m_v\big).
\end{aligned}
\end{equation}
Stage II lets descendant losses update ancestor primitives, so the model can learn the downstream cost of deferring, asserting, or excluding a broad finding. The objective remains a sum of per-node losses on TBP-composed marginals rather than the full structured risk of a particular budgeted decoder; the continuation baseline isolates RPO from simply training longer.

\textbf{Three TBP objects.}
The TBP joint model has coherent support by construction. The fast marginal decoder used in the main table is efficient but not coherence-guaranteed. Exact TBP MAP decoding is exactly coherent but slower. Appendix~\ref{app:tbp_exactdecode} separates these operating points.

\section{Experiments}
\label{sec:experiments}

\begin{table*}[h]
\centering

\begin{minipage}[t]{0.50\textwidth}
\vspace{0pt}
\centering
\fontsize{7.5}{8.5}\selectfont
\captionof{table}{\textbf{Primary results on VinDr-CXR, CheXpert, PadChest, and ADPv2.} Higher is better for utility metrics and lower is better for incoherence. Best means bolded, second-best underlined. Asterisks denote significance by paired Wilcoxon signed-rank test.}
\label{tab:main_results}
\setlength{\tabcolsep}{2pt}
\renewcommand{\arraystretch}{1.05}
\begin{tabular}{@{}lccccccc@{}}
\toprule
\textbf{Method}
& \makecell{\textbf{BalAcc}\\$\uparrow$}
& \makecell{\textbf{F1-I}\\$\uparrow$}
& \makecell{\textbf{F1-L}\\$\uparrow$}
& \makecell{\colorbox{inc_red}{\strut \textbf{Tax.}}\\$\downarrow$}
& \makecell{\colorbox{epi_org}{\strut \textbf{Ded.}}\\$\downarrow$}
& \makecell{\colorbox{op_blu}{\strut \textbf{Del.}}\\$\downarrow$}
& \makecell{\colorbox{black!10}{\strut \textbf{Any}}\\$\downarrow$} \\
\midrule
\multicolumn{8}{c}{-- \textbf{VinDr-CXR} \cite{PhysioNet-vindr-cxr-1.0.0}\;(\textit{real-reader}, large tax., pooled 15 runs) --} \\
\midrule
BR-L2D          & $.852$ & $.512$             & $.749$             & $\approx 0$        & $.088$ & $.019$ & $.107$ \\
BR (cont.)      & $.847$             & $\underline{.514}$ & $\underline{.763}$ & $\approx 0$        & $.115$ & $.024$ & $.138$ \\
\rowcolor{gray!10}
BR+Proj.        & $\mathbf{.853^*}$    & $.513$             & $.749$             & $\mathbf{=0}$      & $\mathbf{=0}$ & $\mathbf{=0}$ & $\mathbf{=0}$ \\
\rowcolor{gray!10}
BR+RPO          & $\mathbf{.853^*}$             & $\mathbf{.515^*}$    & $\mathbf{.768^*}$    & $\mathbf{=0}$      & $\approx 0$ & $\approx 0$ & $\underline{.004}$ \\
\midrule
\multicolumn{8}{c}{-- \textbf{CheXpert} \cite{irvin2019chexpert}\;(\textit{real-reader}, small tax., pooled 20 runs) --} \\
\midrule
BR-L2D          & $.811$             & $.641$             & $.734$             & $\mathbf{=0}$      & $.011$ & $.038$ & $.050$ \\
BR (cont.)      & $\underline{.823}$ & $\underline{.648}$ & $\underline{.749}$ & $\mathbf{=0}$      & $.014$ & $.036$ & $.050$ \\
\rowcolor{gray!10}
BR+Proj.        & $.812$             & $.642$             & $.735$             & $\mathbf{=0}$      & $\mathbf{=0}$ & $\mathbf{=0}$ & $\mathbf{=0}$ \\
\rowcolor{gray!10}
BR+RPO          & $\mathbf{.841^*}$  & $\mathbf{.652^*}$  & $\mathbf{.764^*}$  & $\mathbf{=0}$      & $\underline{\approx 0}$ & $\underline{\approx 0}$ & $\underline{\approx 0}$ \\
\midrule

\multicolumn{8}{c}{-- \textbf{PadChest} \cite{BUSTOS2020101797}\;(\textit{controlled-expert}, xlarge tax.) --} \\
\midrule
BR-L2D          & $.679$             & $.230$             & $.421$             & $\underline{\approx 0}$ & $.122$ & $.002$ & $.124$ \\
BR (cont.)      & $.671$             & $.218$             & $.410$             & $.000$             & $.129$ & $.002$ & $.131$ \\
\rowcolor{gray!10}
BR+Proj.        & $\underline{.692}$ & $\underline{.240}$ & $\underline{.430}$ & $\mathbf{=0}$      & $\mathbf{=0}$ & $\mathbf{=0}$ & $\mathbf{=0}$ \\
\rowcolor{gray!10}
BR+RPO          & $\mathbf{.721^*}$    & $\mathbf{.272^*}$    & $\mathbf{.437^*}$    & $\mathbf{=0}$      & $\underline{\approx 0}$ & $\underline{\approx 0}$ & $\underline{\approx 0}$ \\
\midrule
\multicolumn{8}{c}{-- \textbf{ADPv2} \cite{yang2025adpv2}\;(\textit{controlled-expert}, medium tax.) --} \\
\midrule
BR-L2D          & $.896$             & $.863$             & $.878$             & $\approx 0$             & $.037$             & $.032$             & $.068$ \\
BR (cont.)      & $\mathbf{.900^*}$    & $\mathbf{.868^*}$    & $\mathbf{.882^*}$    & $\underline{\approx 0}$ & $.030$             & $.032$             & $.062$ \\
\rowcolor{gray!10}
BR+Proj.        & $.895$             & $.861$             & $.876$             & $\mathbf{=0}$           & $\mathbf{=0}$      & $\mathbf{=0}$      & $\mathbf{=0}$ \\
\rowcolor{gray!10}
BR+RPO          & $\mathbf{.900^*}$    & $\mathbf{.868^*}$    & $\mathbf{.882^*}$    & $\mathbf{=0}$           & $\underline{\approx 0}$ & $\underline{.008}$ & $\underline{.009}$ \\
\bottomrule
\end{tabular}
\end{minipage}
\hfill
\begin{minipage}[t]{0.47\textwidth}
\vspace{0pt}
\centering
\footnotesize

\includegraphics[width=\linewidth]{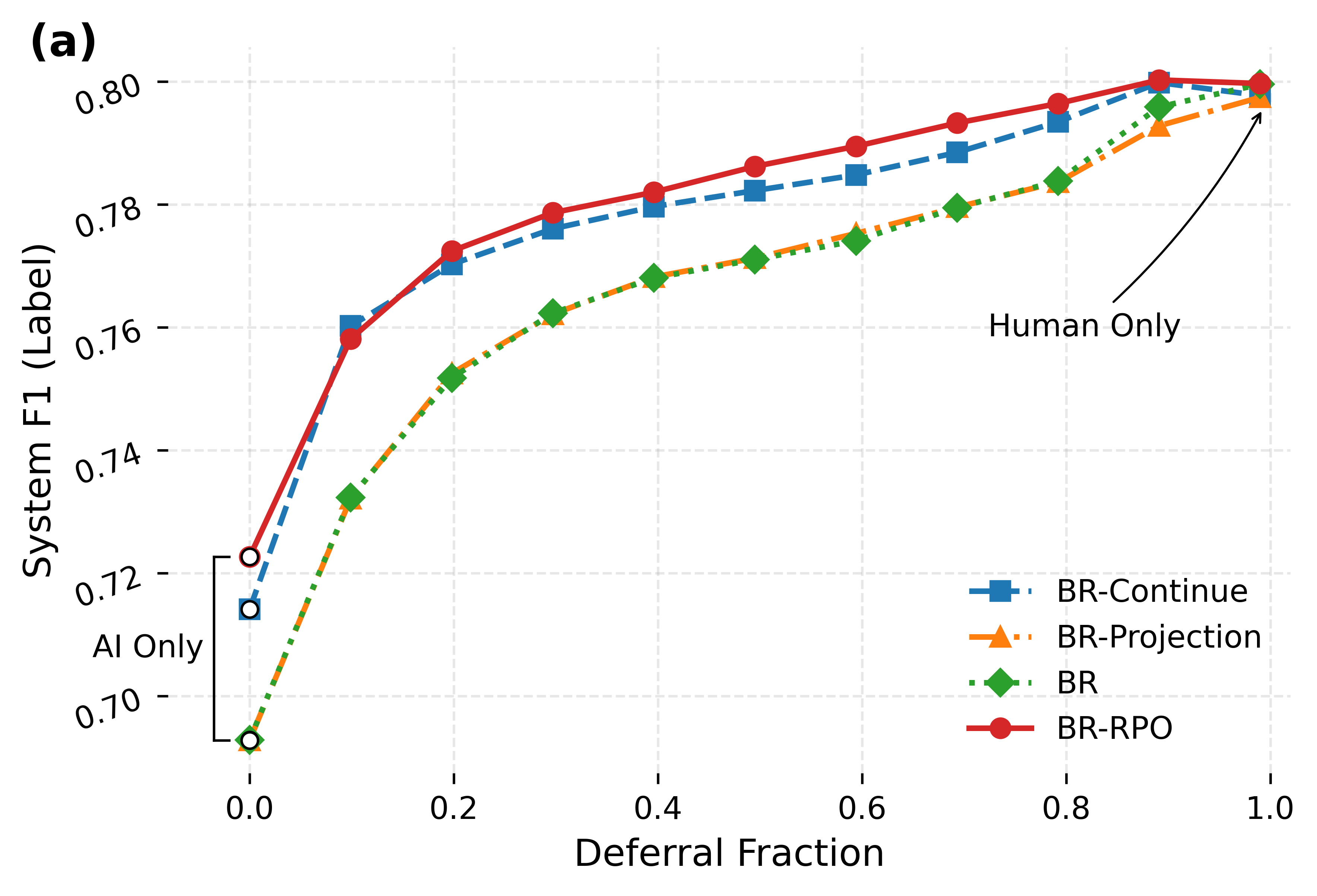}\\[-0.5ex]
\vspace{-3mm}
\includegraphics[width=\linewidth]{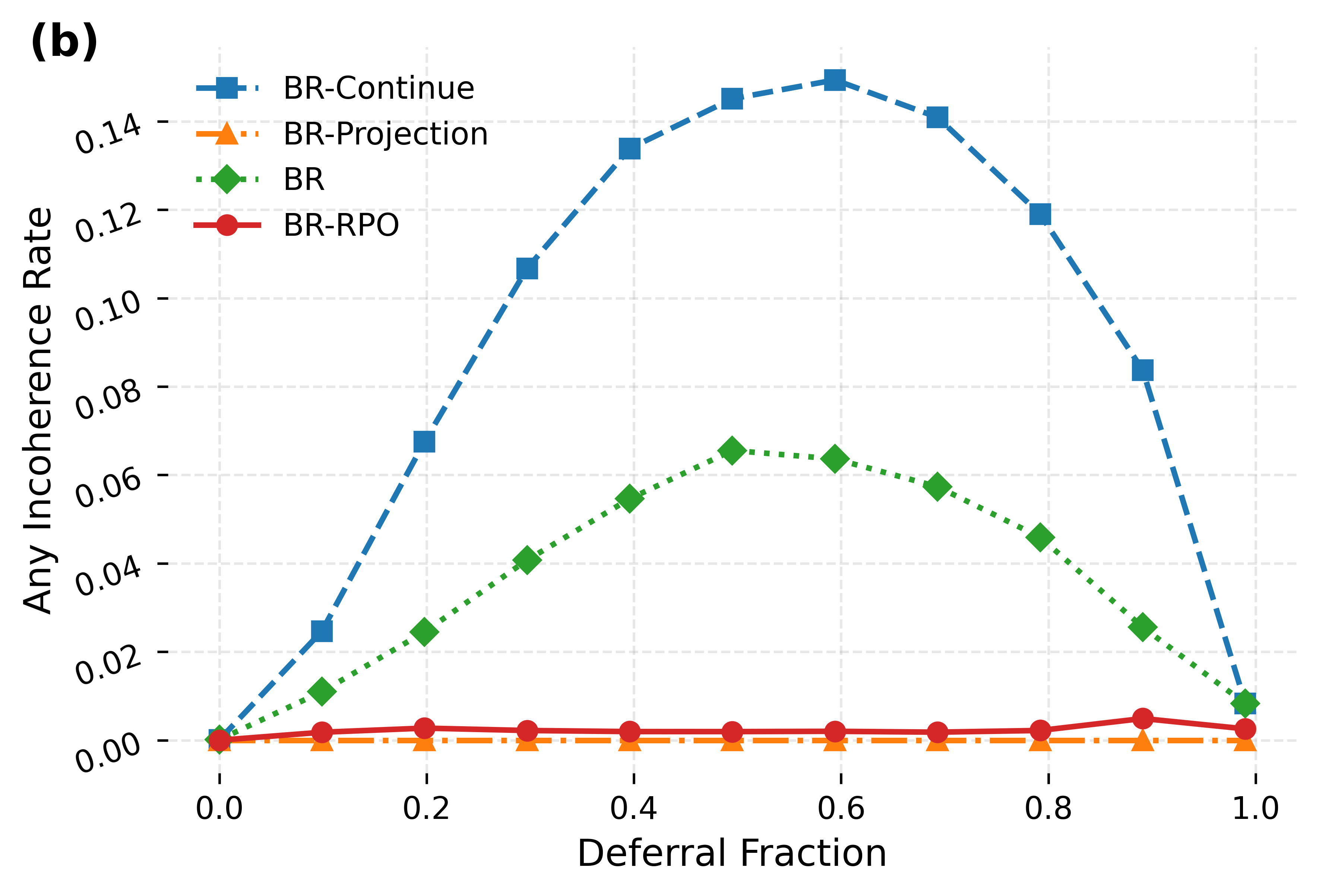}\\[-0.5ex]

\captionof{figure}{\textbf{Representative VinDr-CXR deferral curves.}
Mean over 5 seeds when deferring to radiologist R9.
\textbf{(a)} System F1 score across the node-level deferral budget.
\textbf{(b)} Any incoherence rate over the same budget sweep.}
\label{fig:vindr_r9_curves}

\end{minipage}
\end{table*}

We test three theory-derived claims. First, independent nodewise L2D should produce deferral-specific incoherence: Proposition~\ref{prop:br_bayes_incoherent} shows that even oracle nodewise Bayes decisions can yield delegation violations and deductive defects. Second, exact projection should provide a zero-incoherence operating point by decoding over the coherent action set characterised in Proposition~\ref{prop:local_characterization} and solving the constrained MAP problem in Eq.~\eqref{eq:proj_full_map_main}. Third, TBP+RPO should improve the utility--coherence trade-off beyond matched continued training: TBP assigns zero mass to forbidden parent--child transitions by construction, while RPO trains through the deployed recursion rather than independent local heads. These questions are deployment-relevant because, in assisted reading, a useful system must be accurate after deferral and provide a coherent clinical handoff specifying which findings it has asserted and which it has delegated.

\textbf{Datasets, methods, and evaluation.}
VinDr-CXR \cite{PhysioNet-vindr-cxr-1.0.0} and CheXpert \cite{irvin2019chexpert} provide the primary real-reader CXR evidence through retrospective radiologist annotations. PadChest \cite{BUSTOS2020101797} and ADPv2 \cite{yang2025adpv2} are controlled-expert stress tests for larger taxonomies, varied expert behaviour, and a non-CXR modality. We compare nodewise \textbf{BR-L2D}, matched \textbf{BR (cont.)}, exact \textbf{BR+Projection}, and \textbf{BR+RPO}. BR+RPO uses the fast marginal decoder in the main table. Exact coherent TBP MAP decoding appears in Appendix~\ref{app:tbp_exactdecode}. For each method, we sweep a global node-level deferral budget, fill deferred nodes with expert labels, and report areas under system balanced-accuracy and F1 curves for both label- and instance-wise deferral (abbreviated F1-I and F1-L, respectively). Because assisted interpretation also requires a comprehensible handoff, we report any incoherence and decompose it into taxonomic contradictions, deductive defects, and delegation violations. By Proposition~\ref{prop:local_characterization}, immediate-edge checks are sufficient in our tree setting. Raw upward-closure violations were extremely rare in the real-reader VinDr-CXR and CheXpert annotations, so preprocessing mainly adds implied ancestors rather than sanitising substantial expert incoherence. We leave explicit handling of frequent human--taxonomy disagreement to future prospective deployments.

\textbf{Main results.}
Table~\ref{tab:main_results} shows that deferral incoherence is systematic. Naïve BR-L2D is incoherent on every dataset, with deductive defects and delegation violations dominating ordinary taxonomic contradictions. This confirms the central theory: the main failures are not standard HMLC label violations, but deferral-specific handoff failures driven by expert correctness. Exact projection removes incoherence without retraining and remains utility-competitive, giving a simple zero-incoherence operating point for settings where an existing L2D model must be made contract-compliant. TBP+RPO instead learns through the deployed hierarchy recursion and drives incoherence near zero while retaining strong utility, giving a fast operating point for repeated deployment-time inference.

\textbf{Why this matters operationally.}
The incoherence decomposition is the key diagnostic. If the failures were only parent-absent/child-present contradictions, ordinary hierarchy-aware label correction would be sufficient. Instead, the dominant errors are deductive defects and delegation violations: the model defers labels already implied by its own assertions, or asks the reader to decide a parent while autonomously asserting a child. These are handoff failures rather than ordinary classification errors. A clinical reader receiving such an action vector cannot cleanly tell which parts of the taxonomy the model has taken responsibility for and which parts have been delegated. This is why hierarchical L2D must constrain the ternary action space, not only the binary label space.

\textbf{Representative curves and deployment choices.}
Figure~\ref{fig:vindr_r9_curves} illustrates the same behaviour across the full budget sweep. In Fig.~\ref{fig:vindr_r9_curves}(b), the BR baselines show a characteristic deferral-incoherence curve: incoherence rises as more nodes are deferred, peaks at intermediate budgets, and then falls as the action vector becomes dominated by deferral. The peak occurs at roughly 60\% of node-level decisions deferred, where handoff structure matters most because the system is mixing autonomous assertions, exclusions, and delegated findings. BR continuation is worst in this regime, showing that longer local training can amplify rather than repair incoherent handoffs; BR-L2D is also substantially incoherent. In contrast, BR+Projection is identically zero across the entire sweep, as expected from exact decoding over $\mathcal C_{\mathrm{SE}}$, and BR+RPO remains practically zero while preserving the stronger utility curve in Fig.~\ref{fig:vindr_r9_curves}(a).

The two proposed methods therefore serve different deployment needs. \textbf{BR+Projection} is the safest retrofit option when deterministic zero incoherence is required and a trained local L2D model already exists. \textbf{TBP+RPO} is the preferred high-throughput option when fast inference is needed. The TBP joint model has coherent support, RPO trains through the hierarchy used at inference, and exact TBP MAP decoding can be used when deterministic zero incoherence is required.

\textbf{Ablations.}
The appendix addresses the main alternative explanations and deployment variants, summarised in Table~\ref{tab:ablation_summary}. Matched continuation tests whether RPO is merely longer training; Stage-I sensitivity tests whether the effect depends on BR pretraining; hierarchy-aware losses test whether ordinary HMLC objectives remove ternary handoff defects; strong-subtree handoff tests contract sensitivity; multi-expert routing tests whether the action-space issue persists with expert-indexed deferrals; exact TBP MAP decoding separates the coherent joint model from the fast marginal decoder; and runtime experiments quantify the cost of coherent inference. These results support the central claim: the empirical gains come from modelling the semantics of human--AI handoff in the action space rather than from a particular pretraining loss, expert choice, or contract variant.

\begin{table}[t]
\centering
\caption{\textbf{Ablation summary.} Each ablation probes an alternative explanation or deployment concern.}
\label{tab:ablation_summary}

\scriptsize
\setlength{\tabcolsep}{4pt}
\renewcommand{\arraystretch}{0.95}

\rowcolors{2}{gray!6}{white}

\begin{tabularx}{\textwidth}{
@{}>{\raggedright\arraybackslash}p{0.21\textwidth}
   >{\raggedright\arraybackslash}p{0.27\textwidth}
   >{\raggedright\arraybackslash}X@{}}
\toprule
\textbf{Ablation} & \textbf{Question} & \textbf{Takeaway} \\
\midrule

\textbf{BR continuation}
& Is RPO just longer training?
& \textbf{\emph{No.}} Continued BR training does not consistently reduce incoherence and can increase it; RPO trains through deployed TBP recursion. App.~\ref{app:stage1_sensitivity}, \ref{app:exp_details_stats}. \\

\textbf{Stage-I initialisation}
& Does RPO depend on BR pretraining?
& RPO remains beneficial across tested initialisations, including BR and hierarchy-aware losses. App.~\ref{app:stage1_sensitivity}. \\

\textbf{Hierarchy-aware losses}
& Would standard HMLC objectives solve this?
& Label-consistency losses address binary labels, not delegation violations or deductive defects in ternary action spaces. App.~\ref{app:stage1_sensitivity}. \\

\textbf{Strong-subtree handoff}
& Is Selective-Exclusion the only valid contract?
& \textbf{\emph{No.}} Stricter subtree handoff is feasible and exactly decodable, but shifts the utility--coherence trade-off. App.~\ref{app:semA}, \ref{sensitivity_analysis}. \\

\textbf{Multi-expert deferral}
& Does coherence matter with several experts?
& The same action-space coherence issue persists with expert-indexed deferral actions. App.~\ref{apd:multi-expert-extension}. \\

\textbf{Exact TBP MAP}
& Is residual RPO incoherence from TBP itself?
& \textbf{\emph{No.}} The TBP joint model is coherent; residual incoherence comes from the fast marginal decoder. Exact MAP removes it at some utility cost. App.~\ref{app:tbp_exactdecode}. \\

\textbf{Runtime}
& Are coherent methods practical?
& Projection/TBP decoding scales linearly in taxonomy size; exact budgeted ranking is costlier, motivating fast RPO decoding. App.~\ref{app:exp_details_runtime}. \\

\bottomrule
\end{tabularx}
\end{table}

\section{Conclusion, Limitations, and Future Work}
We introduced hierarchical multi-label L2D as a structured human--AI handoff problem. In medical taxonomies, deferral is not abstention or a third label value, but delegation to a clinical reader. This creates \emph{deferral incoherence}: a ternary action vector may be label-satisfiable yet incoherent as a handoff. Our theoretical and empirical results show that this is substantive. Under Selective-Exclusion, the Bayes-optimal coherent policy is a structured optimisation over $\mathcal C_{\mathrm{SE}}$, while independent nodewise Bayes L2D can still produce delegation violations and deductive defects. Empirically, BR-L2D is incoherent across real-reader and controlled-expert benchmarks, with failures dominated by delegation and deductive closure rather than ordinary taxonomic contradiction. Thus, hierarchical L2D must model the semantics of delegation, not only of labels.

The proposed methods provide practical operating points. Exact projection retrofits an already-trained local L2D system with guaranteed zero incoherence. TBP+RPO builds the handoff contract into the model, trains through the deployed hierarchy recursion, and enables fast near-coherent inference with strong utility. Together, these results show that coherent hierarchical deferral is both theoretically necessary and practically achievable.

\textbf{Limitations and future work.}
Our formulation studies only tree taxonomies, intentionally reflecting the common structure of medical imaging taxonomies. We discuss extensions to DAGs in Appendix \ref{app:dag_extension}. Future work will prospectively study this framework in a real hospital setting.

\bibliographystyle{ACM-Reference-Format}
\bibliography{sample-base}

\renewcommand{\contentsname}{Contents of Appendix}
\tableofcontents
\addtocontents{toc}{\protect\setcounter{tocdepth}{3}} 
\clearpage

\section{Extended Literature Review}
\label{extended_lit}
\textbf{Learning to Defer and selective prediction.}
Learning to Defer (L2D) formalises human--AI collaboration as an \emph{allocation} decision: the model either predicts or defers to an expert to optimise system-level utility rather than standalone accuracy \cite{mozannar2020consistent}. \citet{Strong2025-rr} provide a recent survey of the area. L2D is closely related to selective classification / learning with a reject option \cite{5222035,1054406,el2010foundations,geifman2017selective,geifman2019selectivenet}, but differs in that deferral routes to an external decision-maker with its own error profile rather than simply abstaining. Beyond the original consistent-surrogate formulation \cite{mozannar2020consistent}, later work studies alternative optimisation and deployment regimes, including post-hoc estimators that learn deferral policies on top of a trained predictor \cite{narasimhan2022post} and exact / constrained formulations that optimise deferral subject to explicit capacity constraints \cite{mozannar2023should}.

\textbf{Extensions of L2D.}
Several extensions generalise L2D, addressing routing among multiple experts \cite{verma2023learning}, coverage/workload constraints \cite{ponomarev2024simple,alves2024costsensitive}, structured two-training pipelines (assuming a frozen predictor) \cite{mao2023two,montreuil2024two}, and missing expert annotations \cite{nguyen2025probabilistic}. Most closely adjacent to our training curriculum, \cite{montreuil2024two} develop two-stage L2D formulations for \emph{multi-task} settings (learning a deferral mechanism on top of a multi-head predictor) and further study querying \emph{top-$k$} experts and related guarantees \cite{montreuil2025askaskklearningtodefer}. These works primarily treat tasks as multiple outputs that share representation and routing constraints. Our focus is different: the outputs themselves are coupled by a \emph{taxonomy}, and coherence is defined by logical satisfiability, an explicit parent-deferral contract, and deductive closure.

\textbf{Hierarchical multi-label prediction.}
Hierarchical multi-label classification (HMLC) exploits taxonomies via conditional factorisation and/or constraint-based objectives to improve calibration and enforce hierarchy-consistent predictions \cite{chen2019deep,giunchiglia2020coherent}. This literature typically constrains \emph{label values} (present/absent) but does not model the \emph{delegation action} introduced by L2D. As a result, directly extending L2D with Binary Relevance can produce workflow-incoherent configurations (e.g., deferring a parent while asserting a child) even when the underlying label predictions are hierarchy-consistent.

\textbf{Hierarchy-aware abstention and coherence.}
Recent work has begun to study hierarchy-aware abstention in structured label spaces. \citet{goren2024hierarchical} formulate hierarchical selective classification, where a classifier may back off to an ancestor node instead of fully rejecting an uncertain example. Their setting is related in motivation, but differs from ours because it is single-label, does not involve expert deferral, and does not address coherence of joint ternary actions in a multi-label taxonomy. Our contribution sits at this intersection. We define \emph{deferral coherence} as the conjunction of (i) taxonomic satisfiability of assertions, (ii) workflow admissibility under an explicit contract (Selective-Exclusion), and (iii) deductive closure. We then characterise the Bayes-optimal coherent deferral rule over the constrained action set and show that independent nodewise Bayes L2D can still be action-incoherent. Algorithmically, we enforce these semantics using contract-specialised tree MAP decoding and a contract-indexed top-down Markov recursion optimised end-to-end with Recursive Policy Optimisation. Our method complements prior hierarchy-aware prediction losses \cite{chen2019deep,giunchiglia2020coherent} by extending the structured semantics from binary labels to the ternary prediction/deferral action space used in L2D.

\section{An Informal Walk-Through of Deferral Coherence}
\label{app:intuitive_coherence}

This appendix is an informal companion to Section~\ref{sec:theory}. The main text gives the formal definitions and propositions. Here we describe the same ideas in plain language. It can be read either before or after Section~\ref{sec:theory}. A useful picture to keep in mind is the simple subtree in Figure~\ref{fig:lung_opacity_subtree}: a parent concept such as \emph{Lung Opacity}, with children such as \emph{Edema}, \emph{Infiltration}, and \emph{Consolidation}.

At each node, the model chooses an \emph{action}: predict absence (\(\absent\)), predict presence (\(\present\)), or defer to the expert (\(\defer\)). The key point is that \(\defer\) is \emph{not} a third medical label. It is a workflow choice meaning that the model is not taking responsibility for that node and is instead handing the decision to the expert.

\textbf{Coherence lives in the action space, not just the label space.}
In ordinary hierarchical multi-label classification, the main concern is logical consistency of labels: a child should not be present if its parent is absent. Once deferral is introduced, that is no longer sufficient. We must also ask whether the \emph{handoff} itself is well-formed. A configuration can be logically possible as a set of labels and still correspond to an inconsistent deferral decision.

\textbf{A simple three-question checklist.}
For any parent and its children, coherence can be understood by asking:
\begin{enumerate}[leftmargin=1.4em]
    \item \textbf{Could the model's assertions all be true together?} If not, we have a \colorbox{inc_red}{\strut taxonomic contradiction}.
    \item \textbf{Is the model deferring a parent while still making a child-present claim that already commits the parent?} If yes, we have a \colorbox{op_blu}{\strut delegation violation}.
    \item \textbf{Is the model deferring a node whose value is already implied by its own earlier assertions?} If yes, we have a \colorbox{epi_org}{\strut deductive defect}.
\end{enumerate}

\textbf{Taxonomic contradiction: saying something impossible.}
The familiar bad case is \emph{parent absent, child present}. For example, \emph{Lung Opacity = absent} and \emph{Edema = present} cannot both be true. Another way to say this is that, after leaving deferred nodes undecided, there should still remain at least one hierarchy-consistent way the world could look. If no such completion exists, the action vector is taxonomically impossible.

\textbf{Delegation violation: deferring a question while also answering it indirectly.}
Under our Selective-Exclusion contract, deferring a parent still allows the model to \emph{exclude} specific children, but it does \emph{not} allow the model to assert any child as present. The reason is operational rather than purely logical: saying \emph{Edema = present} already commits \emph{Lung Opacity = present}. So the model cannot honestly say ``the expert should decide the parent'' while simultaneously making a positive subtype assertion underneath it. Equivalently, Selective-Exclusion permits statements like ``I am not deciding Lung Opacity overall, but I can already exclude Edema,'' while forbidding statements like ``I defer Lung Opacity, but I assert Edema is present.''

\textbf{Deductive defect: deferring something the model has already settled itself.}
Suppose the model says \emph{Lung Opacity = absent}. Then every child under that node must also be absent. In that case, a choice such as \emph{Edema = defer} is not contradictory, but it is still defective: the model is handing off a question whose answer is already fixed by its own parent assertion. Equivalently, deductive closure says that when all hierarchy-consistent completions compatible with the model's current assertions agree on a node, the model should state that value explicitly rather than leave it deferred.

\textbf{Why ``parent present, all children absent'' can still be coherent.}
Section~\ref{sec:theory} adopts an \emph{open-world} view of the taxonomy: the listed children need not exhaust all possible ways that a parent can be present. So a configuration such as \emph{Lung Opacity = present} with all listed children marked absent is still allowed. Intuitively, the parent may be present for a reason that is real but not captured by the modelled child set. This is why the theory only forbids specific bad patterns. It does not require every positive parent to be ``explained'' by one positive child.

\begin{table}[t]
\centering
\caption{Informal reading of the local parent--children patterns under Selective-Exclusion.}
\label{tab:intuitive_coherence}
\small
\setlength{\tabcolsep}{5pt}
\renewcommand{\arraystretch}{1.14}
\begin{tabular}{@{}p{2.0cm}p{3.3cm}p{6.0cm}@{}}
\toprule
\textbf{Parent action} & \textbf{Children} & \textbf{Interpretation} \\
\midrule
Absent & At least one child present &
\colorbox{inc_red}{\strut Impossible}: the model asserts a positive child under a negative parent. \\

Absent & All children absent &
\colorbox{coh_grn}{\strut Coherent}: the model fully closes that part of the tree. \\

Absent & No child present, but at least one child deferred &
\colorbox{epi_org}{\strut Incomplete reasoning}: those deferred children are already forced absent by the parent decision. \\

Present & Any mix of absent / present / defer &
\colorbox{coh_grn}{\strut Coherent}: once the parent is present, the model may exclude, assert, or defer children. \\

Deferred & At least one child present &
\colorbox{op_blu}{\strut Bad handoff}: the model says the expert should decide the parent, but a positive child assertion already decides it. \\

Deferred & Every child absent or deferred &
\colorbox{coh_grn}{\strut Coherent}: the model defers the parent while still being allowed to selectively exclude or defer children. \\
\bottomrule
\end{tabular}
\end{table}

\textbf{Why this makes the problem structured rather than nodewise independent.}
The action chosen at a parent changes what is allowed below it. If the parent is marked absent, the subtree is forced absent. If the parent is deferred, children can no longer be asserted present under Selective-Exclusion. So the best action at a node cannot always be chosen independently of its descendants. This is the intuition behind the non-separability result in Section~\ref{sec:nonseparable}: a locally sensible defer decision at an internal node can remove valuable options lower in the tree.

\textbf{Why even nodewise Bayes decisions can fail.}
In flat L2D, the Bayes rule at one label compares the model's confidence with the expert's probability of being correct. In a hierarchy, doing this independently at every node can still produce an incoherent handoff. For example, the parent may locally prefer deferral because the expert is strong at the parent, while the child may locally prefer an autonomous positive assertion because the model is strong at the child. This gives a configuration such as \emph{Lung Opacity = defer} and \emph{Edema = present}, which is a delegation violation. Thus, incoherence is not only a training error. It can arise from solving the wrong, nodewise Bayes problem.

\textbf{Take-home message.}
Deferral coherence is best thought of as three common-sense rules acting together:
\begin{enumerate}[leftmargin=1.4em]
    \item do not assert an impossible parent--child combination;
    \item do not defer a question while answering it indirectly through a descendant;
    \item do not defer a label that your own assertions have already determined.
\end{enumerate}
Section~\ref{sec:theory} makes these rules precise, and Section~\ref{sec:method} then builds estimators whose predictions respect this coherent decision space.

\section{Alternative Deferral Semantics: Strong Subtree Handoff}
\label{app:semA}

Our main method adopts the \emph{Selective-Exclusion} contract (Definition~\ref{def:selective_exclusion}), which allows selectively excluding subtypes under parent deferral. For completeness, we also consider a stricter workflow semantics in which deferral at a parent triggers a \emph{subtree handoff}: if a parent is deferred then all descendants must be deferred.

\textbf{Modified transition kernel.}
Let $\etaLoc_v(a\mid x)=\mathbb{P}_\theta(a_v=a \mid a_{pa(v)}=\present,x)$ denote the local primitive as in the main text. Under strong subtree handoff, the deferral row becomes deterministic:
\begin{equation}
\mathbf{T}^{(A)}_v(x)=
\begin{bmatrix}
1 & 0 & 0 \\
\etaLoc_v(\absent) & \etaLoc_v(\present) & \etaLoc_v(\defer) \\
0 & 0 & 1
\end{bmatrix},
\end{equation}
which enforces $a_{pa(v)}=\defer \Rightarrow a_v=\defer$ for every edge.

\textbf{Resulting recursion.}
With $\muGlob_v=\muGlob_{pa(v)}^\top \mathbf{T}^{(A)}_v$, the marginals become:
\begin{align}
\muGlob_v(\present) &= \muGlob_{pa(v)}(\present)\,\etaLoc_v(\present),\\
\muGlob_v(\absent) &= \muGlob_{pa(v)}(\absent) + \muGlob_{pa(v)}(\present)\,\etaLoc_v(\absent),\\
\muGlob_v(\defer) &= \muGlob_{pa(v)}(\defer) + \muGlob_{pa(v)}(\present)\,\etaLoc_v(\defer).
\end{align}
This semantics guarantees monotone deferral down the hierarchy, at the cost of disallowing selective exclusion.

The Bayes-optimal coherent decoder under SSH has the same form as Eq.~\eqref{eq:bayes_coherent}, but with the SSH feasible set. Equivalently, the tree dynamic program uses
\[
\Gamma_{\mathrm{SSH}}(\absent)=\{\absent\},\qquad
\Gamma_{\mathrm{SSH}}(\present)=\act,\qquad
\Gamma_{\mathrm{SSH}}(\defer)=\{\defer\}.
\]
Thus changing the handoff contract changes the admissible subtree actions, but not the conditional-risk principle.

\subsection{Contract Sensitivity Analysis}
\label{sensitivity_analysis}

We evaluate the strong subtree handoff (SSH) contract as a stricter alternative to Selective-Exclusion, where deferring an internal node forces all descendants to be deferred. Tables~\ref{tab:ssh_main_results} and~\ref{tab:ssh_defect_decomp} show that the same coherent-decoding machinery applies under SSH: projection again gives zero incoherence, while RPO remains near-coherent and typically preserves strong utility.

\begin{table*}[h!]
\centering
\scriptsize
\caption{
\textbf{Balanced utility on VinDr-CXR and CheXpert under the strong subtree handoff (SSH) contract.}
Values are mean $\pm$ standard deviation for the four utility metrics over the SSH BR-family runs: $15$ runs for VinDr-CXR pooled across R8/R9/R10 with AU-BalAcc stopping, and $20$ runs for CheXpert pooled across bc1/bc2/bc3/bc5 with AU-SysF1-I stopping. The final AU-Neigh-Any column reports pooled means only. Higher is better for the first four columns and lower is better for AU-Neigh-Any. Best means are bolded and second-best means are underlined, using the unrounded pooled means to break near-ties. In each dataset block, the last two rows are the two proposed coherent methods: BR+Projection and BR+RPO, with BR+RPO evaluated under SSH.
}
\label{tab:ssh_main_results}
\setlength{\tabcolsep}{4pt}
\renewcommand{\arraystretch}{1.10}
\begin{tabular}{@{}lccccc@{}}
\toprule
\textbf{Method}
& \makecell{\textbf{AU-BalAcc}\\$\uparrow$}
& \makecell{\textbf{AU-F1-Lmacro}\\$\uparrow$}
& \makecell{\textbf{AU-SysF1-I}\\$\uparrow$}
& \makecell{\textbf{AU-SysF1-L}\\$\uparrow$}
& \makecell{\textbf{AU-Neigh-Any}\\$\downarrow$} \\
\midrule
\multicolumn{6}{c}{\textbf{VinDr-CXR} \cite{PhysioNet-vindr-cxr-1.0.0}} \\
\midrule
BR-L2D
    & $0.8525 \pm 0.0127$
    & $0.4695 \pm 0.0180$
    & $0.5121 \pm 0.0260$
    & $0.7494 \pm 0.0156$
    & $0.107348$ \\
BR (cont.)
    & $0.8473 \pm 0.0118$
    & $\underline{0.4780 \pm 0.0191}$
    & $\underline{0.5135 \pm 0.0268}$
    & $\underline{0.7627 \pm 0.0154}$
    & $0.138449$ \\
\rowcolor{gray!10}
BR+Projection
    & $\underline{0.8530 \pm 0.0125}$
    & $0.4706 \pm 0.0180$
    & $0.5129 \pm 0.0258$
    & $0.7493 \pm 0.0156$
    & $\mathbf{0.000000}$ \\
\rowcolor{gray!10}
BR+RPO
    & $\mathbf{0.8553 \pm 0.0157}$
    & $\mathbf{0.5190 \pm 0.0297}$
    & $\mathbf{0.5214 \pm 0.0269}$
    & $\mathbf{0.7791 \pm 0.0168}$
    & $\underline{0.000477}$ \\
\midrule
\multicolumn{6}{c}{\textbf{CheXpert} \cite{irvin2019chexpert}} \\
\midrule
BR-L2D
    & $0.8108 \pm 0.0084$
    & $0.4936 \pm 0.0325$
    & $0.6469 \pm 0.0320$
    & $0.7361 \pm 0.0140$
    & $0.078582$ \\
BR (cont.)
    & $\underline{0.8224 \pm 0.0175}$
    & $\underline{0.5109 \pm 0.0370}$
    & $\underline{0.6552 \pm 0.0347}$
    & $\underline{0.7518 \pm 0.0269}$
    & $0.072095$ \\
\rowcolor{gray!10}
BR+Projection
    & $0.8104 \pm 0.0076$
    & $0.4943 \pm 0.0332$
    & $0.6483 \pm 0.0322$
    & $0.7363 \pm 0.0136$
    & $\mathbf{0.000000}$ \\
\rowcolor{gray!10}
BR+RPO
    & $\mathbf{0.8373 \pm 0.0095}$
    & $\mathbf{0.5475 \pm 0.0315}$
    & $\mathbf{0.6657 \pm 0.0328}$
    & $\mathbf{0.7740 \pm 0.0138}$
    & $\underline{0.000647}$ \\
\bottomrule
\end{tabular}
\end{table*}

\begin{table*}[h!]
\centering
\scriptsize
\caption{
\textbf{Budget-swept decomposition of deferral incoherence on VinDr-CXR and CheXpert under the strong subtree handoff (SSH) contract.}
Values are pooled AUC means over the same SSH BR-family runs as Table~\ref{tab:ssh_main_results}: $15$ VinDr-CXR runs (AU-BalAcc-stopped) and $20$ CheXpert runs (AU-SysF1-I-stopped). The left block is edge-weighted (EW), i.e.\ a micro-average over $(\text{example}, \text{edge})$ pairs; the right block is the coarser neighbourhood-partition (NP) view over $(\text{example}, \text{parent-neighbourhood})$ pairs. Lower is better. In each dataset block, the last two rows are the two proposed coherent methods: BR+Projection and BR+RPO, with BR+RPO evaluated under SSH.
}
\label{tab:ssh_defect_decomp}
\setlength{\tabcolsep}{4pt}
\renewcommand{\arraystretch}{1.10}
\begin{tabular}{@{}lcccccccc@{}}
\toprule
\textbf{Method}
& \multicolumn{4}{c}{\textbf{Edge-Weighted}} 
& \multicolumn{4}{c}{\textbf{Neighbourhood-Partition}} \\
\cmidrule(lr){2-5}\cmidrule(lr){6-9}
& \makecell{\colorbox{inc_red}{\strut \textbf{Tax.}}\\$\downarrow$}
& \makecell{\colorbox{epi_org}{\strut \textbf{Ded.}}\\$\downarrow$}
& \makecell{\colorbox{op_blu}{\strut \textbf{Del.}}\\$\downarrow$}
& \makecell{\colorbox{black!10}{\strut \textbf{Any}}\\$\downarrow$}
& \makecell{\colorbox{inc_red}{\strut \textbf{Tax.}}\\$\downarrow$}
& \makecell{\colorbox{epi_org}{\strut \textbf{Ded.}}\\$\downarrow$}
& \makecell{\colorbox{op_blu}{\strut \textbf{Del.}}\\$\downarrow$}
& \makecell{\colorbox{black!10}{\strut \textbf{Any}}\\$\downarrow$} \\
\midrule
\multicolumn{9}{c}{\textbf{VinDr-CXR} \cite{PhysioNet-vindr-cxr-1.0.0}} \\
\midrule
BR-L2D
    & $0.000010$ & $0.043205$ & $0.010970$ & $0.054185$
    & $0.000030$ & $0.088128$ & $0.019190$ & $0.107348$ \\
BR (cont.)
    & $0.000007$ & $0.061516$ & $0.014627$ & $0.076150$
    & $0.000020$ & $0.114878$ & $0.023551$ & $0.138449$ \\
\rowcolor{gray!10}
BR+Projection
    & $\mathbf{0.000000}$ & $\mathbf{0.000000}$ & $\mathbf{0.000000}$ & $\mathbf{0.000000}$
    & $\mathbf{0.000000}$ & $\mathbf{0.000000}$ & $\mathbf{0.000000}$ & $\mathbf{0.000000}$ \\
\rowcolor{gray!10}
BR+RPO
    & $\mathbf{0.000000}$
    & $0.000183$
    & $\mathbf{0.000000}$
    & $0.000183$
    & $\mathbf{0.000000}$
    & $0.000477$
    & $\mathbf{0.000000}$
    & $0.000477$ \\
\midrule
\multicolumn{9}{c}{\textbf{CheXpert} \cite{irvin2019chexpert}} \\
\midrule
BR-L2D
    & $\mathbf{0.000000}$ & $0.011519$ & $0.034789$ & $0.046308$
    & $\mathbf{0.000000}$ & $0.021758$ & $0.056823$ & $0.078582$ \\
BR (cont.)
    & $0.000009$ & $0.014994$ & $0.031695$ & $0.046698$
    & $0.000017$ & $0.025054$ & $0.047024$ & $0.072095$ \\
\rowcolor{gray!10}
BR+Projection
    & $\mathbf{0.000000}$ & $\mathbf{0.000000}$ & $\mathbf{0.000000}$ & $\mathbf{0.000000}$
    & $\mathbf{0.000000}$ & $\mathbf{0.000000}$ & $\mathbf{0.000000}$ & $\mathbf{0.000000}$ \\
\rowcolor{gray!10}
BR+RPO
    & $\mathbf{0.000000}$ & $0.000343$ & $\mathbf{0.000000}$ & $0.000343$
    & $\mathbf{0.000000}$ & $0.000647$ & $\mathbf{0.000000}$ & $0.000647$ \\
\bottomrule
\end{tabular}
\end{table*}

\section{Proof of Proposition~\ref{prop:br_bayes_incoherent}}
\label{app:proof_br_bayes_incoherent}

\begin{proof}
For any edge $(p\to c)$, upward implication gives
\[
Y_c=1 \Rightarrow Y_p=1,
\qquad\text{hence}\qquad
\pi_c(x):=P(Y_c=1\mid x)\le P(Y_p=1\mid x)=:\pi_p(x).
\]
Assume no ties. If the nodewise Bayes rule predicts $a_c=\present$, then
\[
\pi_c(x) > 1-\pi_c(x),
\]
so $\pi_c(x)>1/2$. By $\pi_c(x)\le \pi_p(x)$, we have $\pi_p(x)>1/2$, so the parent absent action cannot be locally Bayes-optimal. Hence the nodewise Bayes rule cannot produce $a_p=\absent,\ a_c=\present$.

The two deferral-specific failures can occur by construction. For a delegation violation, take
\[
\pi_p=.70,\qquad \pi_c=.60,\qquad q_p=.80,\qquad q_c=.40,
\]
which is compatible with upward implication, e.g.\
\[
P(Y_p=1,Y_c=1)=.60,\quad
P(Y_p=1,Y_c=0)=.10,\quad
P(Y_p=0,Y_c=0)=.30.
\]
Then the parent Bayes scores are $(1-\pi_p,\pi_p,q_p)=(.30,.70,.80)$, so $a_p=\defer$, while the child scores are $(.40,.60,.40)$, so $a_c=\present$. Thus $(a_p,a_c)=(\defer,\present)$ is a delegation violation.

For a deductive defect, take
\[
\pi_p=.20,\qquad \pi_c=.10,\qquad q_p=.10,\qquad q_c=.95,
\]
compatible for example with
\[
P(Y_p=1,Y_c=1)=.10,\quad
P(Y_p=1,Y_c=0)=.10,\quad
P(Y_p=0,Y_c=0)=.80.
\]
The parent scores are $(.80,.20,.10)$, so $a_p=\absent$, while the child scores are $(.90,.10,.95)$, so $a_c=\defer$. Thus $(a_p,a_c)=(\absent,\defer)$ is a deductive defect. Therefore, even oracle nodewise Bayes L2D can be deferral-incoherent under SE: it may either violate the SE contract through a delegation violation or fail deductive closure through a deductive defect.
\end{proof}

\section{Validity and Structural Properties of Taxonomic Belief Propagation}
\label{app:theory}

In this appendix, we prove basic validity and structural properties of the
\emph{Taxonomic Belief Propagation} (TBP) mechanism in the main paper
(§\ref{sec:tbp}) under the \textbf{Selective-Exclusion} contract
(§\ref{def:selective_exclusion}). Throughout, we consider a single edge
$(pa(v)\to v)$ in a tree taxonomy.

\subsection{Setup and notation}

Let $\muGlob_{pa(v)}(x)\in\Delta^2$ denote the \emph{unconditional} marginal action
distribution of the parent node, and let $\etaLoc_{v}(\cdot\mid x)\in\Delta^2$
denote the \emph{local primitive} distribution at node $v$ conditioned on the
parent being asserted present, as in Eq.~\eqref{eq:local_primitive}. That is,
for $\act=\{\absent,\present,\defer\}$,
\begin{align}
\sum_{a\in\act}\muGlob_{pa(v)}(a\mid x) &= 1,\qquad \muGlob_{pa(v)}(a\mid x)\ge 0,\\
\sum_{a\in\act}\etaLoc_{v}(a\mid x) &= 1,\qquad \etaLoc_{v}(a\mid x)\ge 0.
\end{align}

Under Selective-Exclusion, when the parent is deferred, the child may not be
asserted present; the remaining mass is split between $\absent$ and $\defer$ by
renormalising $\etaLoc_v$ onto $\{\absent,\defer\}$. Define
\begin{equation}
\label{eq:app_alpha}
\alpha_v(x)
:= \mathbb{P}(a_v=\absent \mid a_{pa(v)}=\defer, x)
= \frac{\etaLoc_v(\absent\mid x)}{\etaLoc_v(\absent\mid x)+\etaLoc_v(\defer\mid x)},
\end{equation}
with the convention $\alpha_v(x)=1$ when
$\etaLoc_v(\absent\mid x)+\etaLoc_v(\defer\mid x)=0$.
By construction, $\alpha_v(x)\in[0,1]$.

\subsection{TBP transition kernel and recursion (Selective-Exclusion)}

Define the Selective-Exclusion transition kernel $\mathbf T_v(x)\in[0,1]^{3\times 3}$
(rows/cols ordered as $[\absent,\present,\defer]$) by Eq.~\eqref{eq:transition_kernel}:
\begin{equation}
\label{eq:app_transition_kernel}
\mathbf T_v(x)=
\begin{bmatrix}
1 & 0 & 0 \\
\etaLoc_v(\absent\mid x) & \etaLoc_v(\present\mid x) & \etaLoc_v(\defer\mid x) \\
\alpha_v(x) & 0 & 1-\alpha_v(x)
\end{bmatrix}.
\end{equation}
TBP defines the child marginal by passing the parent marginal through the kernel,
as in Eq.~\eqref{eq:tbp_recursion}:
\begin{equation}
\label{eq:app_tbp_recursion}
\muGlob_{v}(x) \;=\; \muGlob_{pa(v)}(x)^\top\, \mathbf T_v(x).
\end{equation}
Expanding \eqref{eq:app_tbp_recursion} yields exactly the main-paper equations:
\begin{align}
\muGlob_v(\present\mid x)
&= \muGlob_{pa(v)}(\present\mid x)\,\etaLoc_v(\present\mid x),
\label{eq:app_prop_present}\\[4pt]
\muGlob_v(\absent\mid x)
&= \muGlob_{pa(v)}(\absent\mid x)
 + \muGlob_{pa(v)}(\present\mid x)\,\etaLoc_v(\absent\mid x)
 + \muGlob_{pa(v)}(\defer\mid x)\,\alpha_v(x),
\label{eq:app_prop_absent}\\[4pt]
\muGlob_v(\defer\mid x)
&= \muGlob_{pa(v)}(\present\mid x)\,\etaLoc_v(\defer\mid x)
 + \muGlob_{pa(v)}(\defer\mid x)\,[1-\alpha_v(x)].
\label{eq:app_prop_defer}
\end{align}

\subsection{Property (i): Normalisation}

\begin{theorem}[Normalisation]
\label{thm:tbp_norm}
For every $x$, the TBP output $\muGlob_v(x)$ is a valid probability distribution:
\(
\sum_{a\in\act}\muGlob_v(a\mid x)=1.
\)
\end{theorem}

\begin{proof}
Sum Eqs.~\eqref{eq:app_prop_present}--\eqref{eq:app_prop_defer}:
\begin{align}
&\muGlob_v(\absent\mid x)+\muGlob_v(\present\mid x)+\muGlob_v(\defer\mid x) \notag\\
&=
\muGlob_{pa(v)}(\absent\mid x)
+ \muGlob_{pa(v)}(\present\mid x)\big(\etaLoc_v(\absent\mid x)+\etaLoc_v(\present\mid x)+\etaLoc_v(\defer\mid x)\big) \notag\\
&\quad
+ \muGlob_{pa(v)}(\defer\mid x)\big(\alpha_v(x) + (1-\alpha_v(x))\big) \notag\\
&=
\muGlob_{pa(v)}(\absent\mid x)
+ \muGlob_{pa(v)}(\present\mid x)\cdot 1
+ \muGlob_{pa(v)}(\defer\mid x)\cdot 1 \notag\\
&= 1,
\end{align}
using that $\etaLoc_v(\cdot\mid x)$ and $\muGlob_{pa(v)}(\cdot\mid x)$ are normalised.
\end{proof}

\subsection{Property (ii): Nonnegativity and bounds}

\begin{theorem}[Bounds]
\label{thm:tbp_bounds}
For all $a\in\act$ and all $x$, $0\le \muGlob_v(a\mid x)\le 1$.
\end{theorem}

\begin{proof}
All terms in Eqs.~\eqref{eq:app_prop_present}--\eqref{eq:app_prop_defer} are sums
and products of nonnegative quantities, and $\alpha_v(x)\in[0,1]$, so
$\muGlob_v(a\mid x)\ge 0$ for each $a$. By §\ref{thm:tbp_norm}, the components sum
to $1$, hence no single component can exceed $1$.
\end{proof}

\subsection{Property (iii): Contract and coherence constraints are enforced}

\begin{theorem}[Coherence constraints as zero-probability transitions]
\label{thm:tbp_constraints}
The TBP kernel enforces the local coherence constraints induced by taxonomic satisfiability, deductive closure, and the Selective-Exclusion contract:
\begin{enumerate}[leftmargin=*]
\item If the parent is asserted absent, then the child is forced absent:
\(
\mathbb P(a_v=\present\mid a_{pa(v)}=\absent,x)=
\mathbb P(a_v=\defer\mid a_{pa(v)}=\absent,x)=0.
\)
\item If the parent is deferred, the child cannot be asserted present:
\(
\mathbb P(a_v=\present\mid a_{pa(v)}=\defer,x)=0.
\)
\end{enumerate}
\end{theorem}

\begin{proof}
Both claims follow immediately from the corresponding rows of the transition
matrix \eqref{eq:app_transition_kernel}. The first row is $[1,0,0]$, and the
third row is $[\alpha_v(x),0,1-\alpha_v(x)]$.
\end{proof}

\begin{proposition}[Coherent handoff preserves label consistency]
\label{prop:handoff_preserves_consistency}
Let the completed system label be
\[
\hat y^{\mathrm{sys}}_v(a,m)=
\begin{cases}
a_v, & a_v\in\{\absent,\present\},\\
m_v, & a_v=\defer.
\end{cases}
\]
If $a\in\mathcal C_{\mathrm{SE}}$ and the expert label vector $m\in\mathcal Y_{\mathcal T}$, then $\hat y^{\mathrm{sys}}(a,m)\in\mathcal Y_{\mathcal T}$.
\end{proposition}
\begin{proof}
For any edge $(p\to c)$, if $a_p=\absent$, coherence forces $a_c=\absent$. If $a_p=\present$, the completed parent is positive. If $a_p=\defer$, coherence gives $a_c\in\{\absent,\defer\}$; in the deferred-child case both labels come from the expert, whose labels are hierarchy-consistent. Thus no completed positive child can occur under a completed negative parent.
\end{proof}

\subsection{Monotonicity-style implications}

\begin{theorem}[Upward-implication-style marginal bound]
\label{thm:tbp_upward}
For every $x$,
\(
\muGlob_v(\present\mid x)\le \muGlob_{pa(v)}(\present\mid x).
\)
\end{theorem}

\begin{proof}
By \eqref{eq:app_prop_present},
\(
\muGlob_v(\present\mid x)=\muGlob_{pa(v)}(\present\mid x)\,\etaLoc_v(\present\mid x).
\)
Since $\etaLoc_v(\present\mid x)\in[0,1]$, the inequality follows.
\end{proof}

\begin{theorem}[Effect of parent deferral mass under Selective-Exclusion]
\label{thm:tbp_parent_defer_split}
Under Selective-Exclusion, the contribution of parent deferral mass to the child marginal is split between the child's \(\absent\) and \(\defer\) outcomes. Algebraically, in Eqs.~\eqref{eq:app_prop_absent}--\eqref{eq:app_prop_defer}, the coefficient of \(\muGlob_{pa(v)}(\defer\mid x)\) is \(\alpha_v(x)\) in \(\muGlob_v(\absent\mid x)\) and \(1-\alpha_v(x)\) in \(\muGlob_v(\defer\mid x)\). Thus parent deferral mass is redistributed between child absence and child deferral with weights \(\alpha_v(x)\) and \(1-\alpha_v(x)\), respectively.
\end{theorem}

\begin{proof}
This follows directly from the expanded TBP recursion:
\[
\muGlob_v(\absent\mid x)
=
\muGlob_{pa(v)}(\absent\mid x)
+
\muGlob_{pa(v)}(\present\mid x)\etaLoc_v(\absent\mid x)
+
\muGlob_{pa(v)}(\defer\mid x)\alpha_v(x),
\]
and
\[
\muGlob_v(\defer\mid x)
=
\muGlob_{pa(v)}(\present\mid x)\etaLoc_v(\defer\mid x)
+
\muGlob_{pa(v)}(\defer\mid x)(1-\alpha_v(x)).
\]
The coefficients multiplying \(\muGlob_{pa(v)}(\defer\mid x)\) are therefore \(\alpha_v(x)\) and \(1-\alpha_v(x)\). Since \(\alpha_v(x)\in[0,1]\), parent deferral mass is split between child \(\absent\) and child \(\defer\). This is an algebraic decomposition of the recursion, not an independent perturbation of one coordinate of the parent marginal off the probability simplex.
\end{proof}

\subsection{Reverse-KL characterisation of the deferred-parent row}
\label{norm_proof}
\begin{proposition}[Selective-Exclusion renormalisation as an exact KL projection]
\label{prop:alpha_kl_projection}
Fix a non-root node $v$ and input $x$. Let
\[
\etaLoc_v(\cdot\mid x)\in\Delta^2
\]
be the local primitive distribution over
\(
\act=\{\absent,\present,\defer\}
\),
with coordinates ordered as $[\absent,\present,\defer]$. Assume
\[
\etaLoc_v(\absent\mid x)+\etaLoc_v(\defer\mid x)>0,
\]
which holds in our parameterisation since $\etaLoc_v$ is produced by a softmax. Under Selective-Exclusion, when the parent is deferred, the admissible child-action distributions are
\[
\mathcal Q_v^{\defer}(x)
:=
\{q\in\Delta^2:\; q(\present)=0\}.
\]
Then the unique solution of
\begin{equation}
\label{eq:kl_proj_problem}
q_v^\star(\cdot\mid x)
\in
\arg\min_{q\in\mathcal Q_v^{\defer}(x)}
D_{\mathrm{KL}}\!\bigl(q \,\|\, \etaLoc_v(\cdot\mid x)\bigr)
\end{equation}
is
\begin{equation}
\label{eq:kl_proj_solution}
q_v^\star(\cdot\mid x)
=
\left[
\frac{\etaLoc_v(\absent\mid x)}
{\etaLoc_v(\absent\mid x)+\etaLoc_v(\defer\mid x)},
\;
0,
\;
\frac{\etaLoc_v(\defer\mid x)}
{\etaLoc_v(\absent\mid x)+\etaLoc_v(\defer\mid x)}
\right].
\end{equation}
Equivalently,
\[
q_v^\star(\absent\mid x)=\alpha_v(x),
\qquad
q_v^\star(\defer\mid x)=1-\alpha_v(x),
\]
where $\alpha_v(x)$ is defined in Eq.~\eqref{eq:app_alpha}.
\end{proposition}

\begin{proof}
Write
\[
\eta_0:=\etaLoc_v(\absent\mid x),\qquad
\eta_1:=\etaLoc_v(\present\mid x),\qquad
\eta_\bot:=\etaLoc_v(\defer\mid x).
\]
Any feasible $q\in\mathcal Q_v^{\defer}(x)$ has the form
\[
q=(q_0,0,q_\bot),
\qquad
q_0\ge 0,\quad q_\bot\ge 0,\quad q_0+q_\bot=1.
\]
Substituting this into the reverse KL objective gives
\[
D_{\mathrm{KL}}(q\|\etaLoc_v)
=
q_0\log\frac{q_0}{\eta_0}
+
q_\bot\log\frac{q_\bot}{\eta_\bot},
\]
since the middle term is
\(
q_\present\log\!\bigl(q_\present/\eta_1\bigr)=0
\)
by feasibility.

We therefore solve
\[
\min_{q_0,q_\bot\ge 0,\; q_0+q_\bot=1}
\left\{
q_0\log\frac{q_0}{\eta_0}
+
q_\bot\log\frac{q_\bot}{\eta_\bot}
\right\}.
\]
Introducing a Lagrange multiplier $\lambda$ for the constraint $q_0+q_\bot=1$, the Lagrangian is
\[
\mathcal L(q_0,q_\bot,\lambda)
=
q_0\log\frac{q_0}{\eta_0}
+
q_\bot\log\frac{q_\bot}{\eta_\bot}
+
\lambda(q_0+q_\bot-1).
\]
The first-order optimality conditions are
\[
\frac{\partial \mathcal L}{\partial q_0}
=
\log\frac{q_0}{\eta_0}+1+\lambda
=
0,
\qquad
\frac{\partial \mathcal L}{\partial q_\bot}
=
\log\frac{q_\bot}{\eta_\bot}+1+\lambda
=
0.
\]
Subtracting the two equations yields
\[
\log\frac{q_0}{\eta_0}
=
\log\frac{q_\bot}{\eta_\bot},
\]
and hence
\[
\frac{q_0}{\eta_0}
=
\frac{q_\bot}{\eta_\bot}
=
c
\]
for some constant $c>0$. Using the normalisation constraint,
\[
q_0+q_\bot
=
c(\eta_0+\eta_\bot)
=
1,
\]
so
\[
c=\frac{1}{\eta_0+\eta_\bot}.
\]
Therefore
\[
q_0^\star=\frac{\eta_0}{\eta_0+\eta_\bot},
\qquad
q_\bot^\star=\frac{\eta_\bot}{\eta_0+\eta_\bot},
\]
which proves Eq.~\eqref{eq:kl_proj_solution}.

Finally, since $\eta_0,\eta_\bot>0$ under the softmax parameterisation, the map
\[
(q_0,q_\bot)\mapsto
q_0\log\frac{q_0}{\eta_0}
+
q_\bot\log\frac{q_\bot}{\eta_\bot}
\]
is strictly convex on the feasible set. Hence the stationary point above is the unique global minimiser of \eqref{eq:kl_proj_problem}. By Eq.~\eqref{eq:app_alpha}, its $\absent$ coordinate is exactly $\alpha_v(x)$.
\end{proof}

\subsection{Summary}

The results above show that TBP under Selective-Exclusion (i) produces valid
normalised marginals, (ii) enforces the local coherence constraints by assigning zero probability
to forbidden transitions, and (iii) preserves an upward-implication-style
monotonicity bound on the \emph{present} action while explicitly characterising
how parent deferral probability is redistributed between child $\absent$ and
child $\defer$ via $\alpha_v(x)$. Moreover, this redistribution is not
heuristic: $\alpha_v(x)$ is exactly the unique reverse-KL projection of the
local primitive onto the admissible face that excludes the $\present$ action
under parent deferral.

\section{Why end-to-end TBP fine-tuning can improve system utility}
\label{app:rpo_why_improves}

This appendix expands the non-separability argument from Section~\ref{sec:nonseparable}. The Bayes recursion in Section~\ref{sec:bayes_coherent} shows that an internal-node action should account for downstream subtree value, not only local risk. Stage~II fine-tuning through TBP (RPO) can therefore improve \emph{system utility} even when coherence defects are already rare after Stage~I. At a high level, Stage~II can help for two independent reasons:

\begin{enumerate}[leftmargin=*]
\item \textbf{Contract-induced non-separability (gating).}
Under a deferral contract (Selective-Exclusion), internal-node actions change the feasible action space of descendants.
Therefore, the Bayes-optimal \emph{coherent} policy is generally \emph{not} separable across nodes, whereas many Stage~I objectives are locally trained (e.g., BR-style L2D, per-node surrogates, or masked classification losses).

\item \textbf{Composition/semantics mismatch.}
Stage~I may optimise predictions under a hierarchy semantics that differs from TBP's \emph{contract-constrained ternary} composition.
Even when Stage~I is hierarchy-aware (e.g., constraint/closure-based HMLC surrogates such as MCLoss), the resulting parameters need not be stationary for the TBP-induced marginals that are used at inference in our deferral system.
\end{enumerate}

We formalise these points below and give an explicit two-node counterexample with strict improvement.

\subsection{A contract-constrained system risk}
Consider a taxonomy with a single edge $(p \to c)$ and labels $(y_p,y_c)\in\{0,1\}^2$ satisfying upward implication $y_c=1\Rightarrow y_p=1$.
A policy selects actions $a_p,a_c\in\act:=\{\absent,\present,\defer\}$.
Under Selective-Exclusion (Def.~\ref{def:selective_exclusion}), if $a_p=\defer$ then $a_c\in\{\absent,\defer\}$ (in particular, $a_c\neq\present$).

We evaluate system performance with a simple (yet standard) ``predict vs.\ defer'' risk: when the system asserts $\absent$ or $\present$, it incurs $0$--$1$ error against the ground-truth; when it defers, it incurs the expert's $0$--$1$ error plus a deferral cost $\lambda>0$.
Formally, for node $v\in\{p,c\}$, ground truth $y_v$, and expert prediction $m_v$, define
\begin{equation}
\label{eq:app_node_risk}
\ell(a_v,y_v,m_v)
=
\begin{cases}
\mathbb I[a_v \neq y_v], & a_v\in\{\absent,\present\},\\
\mathbb I[m_v \neq y_v] + \lambda, & a_v=\defer.
\end{cases}
\end{equation}
The (expected) system risk is the sum over nodes:
\begin{equation}
\label{eq:app_sys_risk}
R(\pi)
:= \mathbb E_{(x,y,m)}\big[\,\ell(a_p,y_p,m_p) + \ell(a_c,y_c,m_c)\,\big],
\end{equation}
where $(a_p,a_c)$ are drawn (or decoded) from the policy $\pi(\cdot\mid x)$ subject to Selective-Exclusion.

\subsection{Stage~I vs.\ Stage~II objectives: a loss-agnostic view}
Let $\{\etaLoc_v(\cdot\mid x)\}_{v\in\mathcal V}$ denote local primitives produced by a network (Eq.~\ref{eq:local_primitive} in the main text).
TBP defines a deterministic, differentiable composition map
\begin{equation}
\label{eq:app_comp_map}
\muGlob(\cdot\mid x) \;=\; F_{\text{TBP}}\big(\etaLoc(\cdot\mid x)\big),
\end{equation}
via the transition kernels and recursion in Eq.~\eqref{eq:transition_kernel}--\eqref{eq:tbp_recursion}.
Stage~I training with an arbitrary surrogate can be abstracted as optimising some objective of the form
\begin{equation}
\label{eq:app_stage1_general}
\mathcal J_{\text{base}}(\theta)
\;=\;
\mathbb E\Big[\; \sum_{v\in\mathcal V} \mathcal L^{(I)}_v\big(\etaLoc_v(x;\theta),y_v,m_v\big)\;\Big],
\end{equation}
where $\mathcal L^{(I)}_v$ may be an L2D surrogate, a hierarchy-aware classification loss, a constraint-based objective, or any mixture thereof.
Stage~II (RPO) instead optimises the \emph{inference-time} TBP-induced marginals:
\begin{equation}
\label{eq:app_stage2_general}
\mathcal J_{\text{RPO}}(\theta)
\;=\;
\mathbb E\Big[\; \sum_{v\in\mathcal V} \mathcal L^{(II)}_v\big(\muGlob_v(x;\theta),y_v,m_v\big)\;\Big],
\qquad \muGlob = F_{\text{TBP}}(\etaLoc).
\end{equation}
Even when $\mathcal L^{(I)}_v=\mathcal L^{(II)}_v$ (same surrogate), the objectives differ because $F_{\text{TBP}}$ couples nodes and modifies the effective decision distribution.
Hence, Stage~I optima need not be stationary for Stage~II.

\subsection{Myopic (separable) vs.\ contract-aware optimality}
A separable training objective (e.g., BR-style L2D) effectively optimises a nodewise criterion
\begin{equation}
\label{eq:app_br_myopic}
\min_{\pi_p,\pi_c}\;
\mathbb E\big[\ell(a_p,y_p,m_p)\big]
+
\mathbb E\big[\ell(a_c,y_c,m_c)\big],
\end{equation}
which treats the choice of $a_p$ as independent of the consequences for the feasible set of $a_c$.
In contrast, under Selective-Exclusion, the optimal decision at $p$ must account for the downstream best response:
\begin{equation}
\label{eq:app_value_compare}
\underbrace{\mathbb E[\ell(a_p,y_p,m_p)\mid x]}_{\text{parent risk}}
+
\underbrace{\min_{a_c\in\mathcal A_c(a_p)}\mathbb E[\ell(a_c,y_c,m_c)\mid x]}_{\text{best achievable child risk given }a_p},
\end{equation}
where $\mathcal A_c(a_p)=\act$ if $a_p=\present$ and $\mathcal A_c(a_p)=\{\absent,\defer\}$ if $a_p=\defer$ (and $\mathcal A_c(a_p)=\{\absent\}$ if $a_p=\absent$ under TBP's satisfiability row).
Thus, deferring an internal node can have a non-local cost by forcing suboptimal actions downstream.

\subsection{A fixed-$x$ counterexample with strict utility gain}
We now construct an explicit conditional distribution at a single input $x^\star$ where the locally optimal choice at $p$ is to defer, but the globally optimal coherent policy at that same $x^\star$ asserts $p=\present$ to unlock a much better decision at $c$.
All inequalities are strict.

\begin{proposition}[Option value of commitment under Selective-Exclusion]
\label{prop:option_value_commitment}
There exist an input $x^\star$, a conditional label distribution satisfying $y_c=1\Rightarrow y_p=1$, an expert $m$, a deferral cost $\lambda>0$, and a model class such that:
(i) the nodewise Bayes-optimal action at $p$ conditional on $x^\star$ (minimising $\mathbb E[\ell(a_p,y_p,m_p)\mid x^\star]$) is to defer ($a_p=\defer$),
but (ii) the globally optimal Selective-Exclusion-admissible policy at that same $x^\star$ attains strictly lower system risk by asserting $a_p=\present$.
\end{proposition}

\begin{proof}
Fix an input $x^\star$. Conditional on $x^\star$, let
\[
\mathbb P\big((y_p,y_c)=(1,1)\mid x^\star\big)=0.9,
\qquad
\mathbb P\big((y_p,y_c)=(0,0)\mid x^\star\big)=0.1.
\]
This satisfies $y_c=1\Rightarrow y_p=1$.

Let the parent expert be perfect, so deferring at $p$ incurs conditional risk $\lambda=0.05$. Let the child expert have error rate $0.4$ conditional on $x^\star$, so child deferral incurs risk $0.4+\lambda=0.45$. Assume the model class can represent:
(a) a child action $a_c=\present$ with conditional risk $0.1$ at $x^\star$;
(b) a parent action $a_p=\present$ with conditional risk $0.1$ at $x^\star$.

\textbf{(i) Local conditional optimum at the parent.}
At $x^\star$, asserting $a_p=\present$ has conditional risk $0.1$, while deferring has conditional risk $0.05$.
Hence the nodewise Bayes-optimal parent action conditional on $x^\star$ is
\[
\arg\min_{a_p}\mathbb E[\ell(a_p,y_p,m_p)\mid x^\star]=\defer.
\]

\textbf{(ii) Global coherent optimum differs at the same $x^\star$.}
Under Selective-Exclusion, if $a_p=\defer$ then the child cannot assert $\present$.
At $x^\star$, the child therefore chooses between asserting $\absent$ (risk $0.9$) and deferring (risk $0.45$), so the best admissible child action is $a_c=\defer$.
The resulting coherent conditional system risk is
\[
R_{\text{defer-parent}}(x^\star)=0.05+0.45=0.50.
\]

If instead the parent asserts $a_p=\present$, then the child may assert $\present$ as well. The coherent policy
\[
a_p=\present,\qquad a_c=\present
\]
therefore has conditional system risk
\[
R_{\text{assert-parent}}(x^\star)=0.10+0.10=0.20.
\]
Thus
\[
R_{\text{assert-parent}}(x^\star)=0.20 < 0.50 = R_{\text{defer-parent}}(x^\star),
\]
even though deferring at the parent is locally optimal conditional on the same $x^\star$.
\end{proof}

\subsection{Interpretation and link to RPO}
Proposition~\ref{prop:option_value_commitment} isolates the \emph{contract-induced non-separability} behind Stage~II gains:
deferring an internal node can be locally optimal (because the expert is strong), yet globally suboptimal because it removes downstream autonomy under the contract, forcing costly deferrals or errors at descendants.

TBP makes descendant marginals depend on ancestor marginals (e.g., for $(p\to c)$,
$\muGlob_c(\present\mid x)=\muGlob_p(\present\mid x)\,\etaLoc_c(\present\mid x)$.
Hence, when RPO optimises a loss on the \emph{unconditional} marginals, gradients from descendant losses backpropagate into ancestor decisions.
In the above construction, child loss incurred at $x^\star$ induces a gradient signal that increases $\muGlob_p(\present\mid x^\star)$, because raising the parent's ``present'' marginal enlarges the feasible child action set and eliminates forced child deferral.

\section{Dataset details}

\subsection{Dataset Licenses and Access Terms}
\label{app:dataset_licenses}

All datasets used in this work were accessed under their respective research-use terms. VinDr-CXR v1.0.0 is a restricted PhysioNet resource released under the PhysioNet Credentialed Health Data License 1.5.0 and the PhysioNet Credentialed Health Data Use Agreement 1.5.0; access requires credentialed PhysioNet status, completion of the required CITI training, and signing the project DUA. CheXpert is distributed through Stanford AIMI/Stanford ML Group; its dataset documentation describes the terms as personal, non-commercial, non-clinical research use. PadChest is distributed under the BIMCV-PadChest Dataset Research Use Agreement, which permits research use without charge, prohibits sale and redistribution, requires approval for non-academic use, and states that the dataset is for research use only. ADPv2 is distributed through Zenodo in three parts under the Creative Commons Attribution-NonCommercial 4.0 International license.

We do not redistribute any medical images, labels, or dataset files. Reproducibility therefore requires users to obtain each dataset directly from the official provider and comply with the corresponding license, data-use agreement, credentialing, and non-clinical-use restrictions.

\subsection{VinDr-CXR}
\label{app:vindr_taxonomy}

Figure~\ref{fig:vindr_taxonomy} shows the annotation-compatible tree used for VinDr-CXR after introducing synthetic grouping nodes and excluding \texttt{No finding} from training and evaluation.

\begin{figure}[h]
\centering
\scalebox{0.9}{
\begin{forest}
for tree={
    draw=none,
    font=\footnotesize,
    edge={semithick},
    anchor=west,
    grow'=east,
    child anchor=west,
    parent anchor=east,
    l sep=4pt,
    s sep=4pt
}
[ROOT
  [No finding]
  [Abnormality
    [Pulmonary abnormality
      [Parenchymal abnormality
        [Opacity-related finding
          [Lung Opacity]
          [Consolidation]
          [Infiltration]
          [Edema]
        ]
        [Interstitial / fibrotic finding
          [ILD]
          [Pulmonary fibrosis]
        ]
        [Obstructive airway finding
          [Emphysema]
          [COPD]
        ]
        [Infectious pulmonary finding
          [Pneumonia]
          [Tuberculosis]
        ]
        [Atelectasis]
        [Lung cavity]
        [Lung cyst]
      ]
      [Nodule / mass
        [Nodule/Mass]
        [Lung tumor]
      ]
    ]
    [Pleural abnormality
      [Pleural effusion]
      [Pleural thickening]
      [Pneumothorax]
    ]
    [Cardiac / vascular abnormality
      [Cardiomegaly]
      [Aortic enlargement]
      [Enlarged PA]
    ]
    [Mediastinal abnormality [Mediastinal shift]]
    [Musculoskeletal abnormality
      [Rib fracture]
      [Clavicle fracture]
    ]
    [Other abnormality
      [Calcification]
      [Other lesion]
      [Other diseases]
    ]
  ]
]
\end{forest}
}
\caption{Annotation-compatible chest X-ray taxonomy for the VinDr-CXR dataset. Problematic raw-to-raw implication edges from the original taxonomy are replaced by synthetic internal grouping nodes. The raw source taxonomy contains a \texttt{No finding} node under \texttt{ROOT}, but all reported experiments exclude \texttt{No finding} from training and evaluation.}
\label{fig:vindr_taxonomy}
\end{figure}
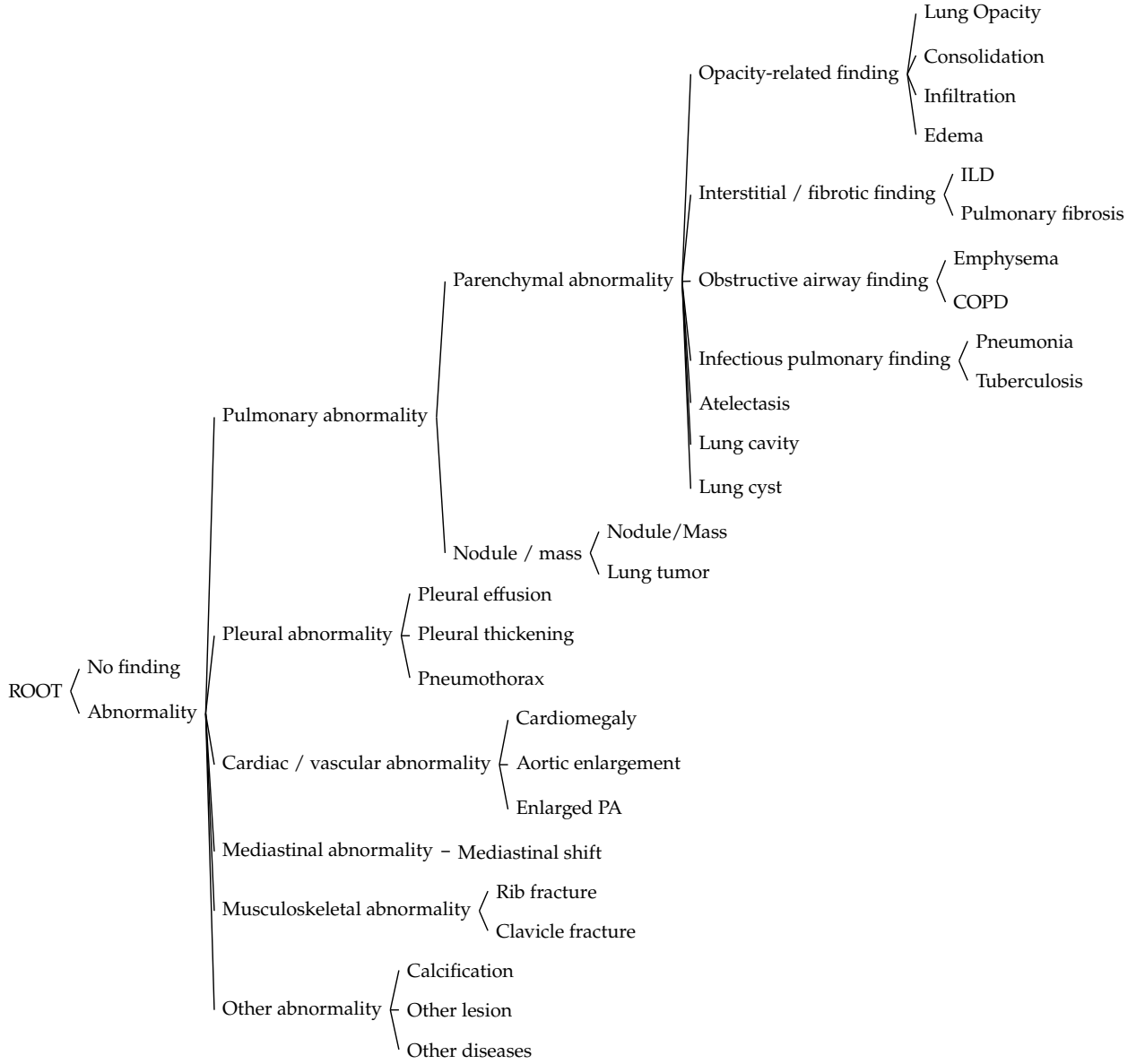

\subsection{CheXpert}
\label{app:chexpert_taxonomy}

Figure~\ref{fig:chexpert_taxonomy} shows the CheXpert taxonomy used to define hierarchy constraints and parent--child handoff coherence.


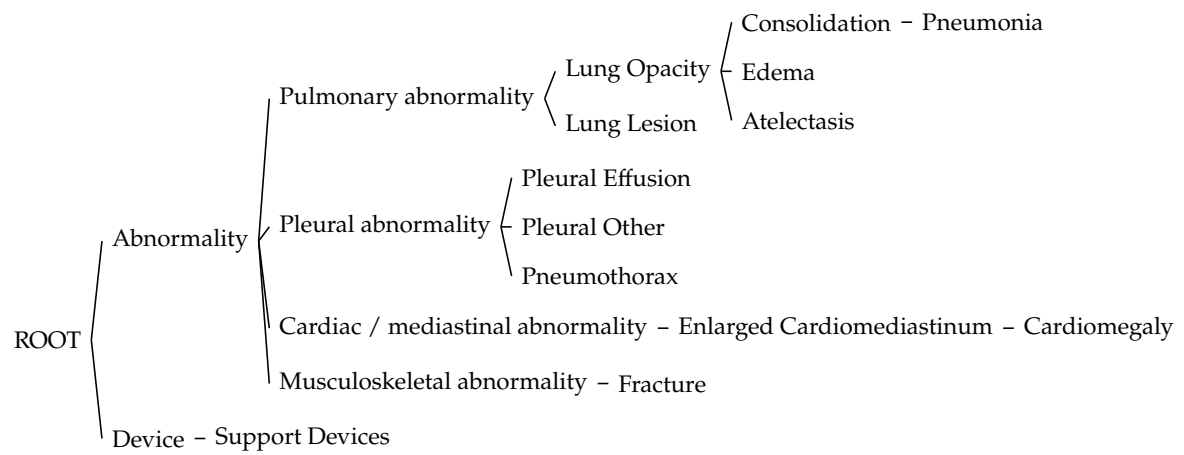
\begin{figure}[h]
\centering
\scalebox{1.0}{ 
\begin{forest}
for tree={
    draw=none,
    font=\footnotesize,
    edge={semithick},
    anchor=west,
    grow'=east,
    child anchor=west,
    parent anchor=east,
    l sep=4pt,
    s sep=4pt
}
[ROOT
  [Abnormality
    [Pulmonary abnormality
      [Lung Opacity
        [Consolidation [Pneumonia]]
        [Edema]
        [Atelectasis]
      ]
      [Lung Lesion]
    ]
    [Pleural abnormality
      [Pleural Effusion]
      [Pleural Other]
      [Pneumothorax]
    ]
    [Cardiac / mediastinal abnormality
      [Enlarged Cardiomediastinum
        [Cardiomegaly]
      ]
    ]
    [Musculoskeletal abnormality
      [Fracture]
    ]
  ]
  [Device
    [Support Devices]
  ]
]
\end{forest}
}
\caption{Taxonomy utilised for the CheXpert dataset.}
\label{fig:chexpert_taxonomy}
\end{figure}

\subsection{PadChest}
\label{app:padchest_taxonomy}

Figure~\ref{fig:padchest_taxonomy} shows the larger PadChest taxonomy used as a controlled-expert stress test for taxonomy scale.

\begin{figure}[h]
\centering
\scalebox{0.75}{
\begin{forest}
for tree={
    draw=none,
    font=\footnotesize,
    edge={semithick},
    anchor=west,
    grow'=east,
    child anchor=west,
    parent anchor=east,
    l sep=4pt,
    s sep=4pt
}
[ROOT
  [Abnormality
    [Pulmonary abnormality
      [Chronic pulmonary finding
        [Obstructive pulmonary finding
          [copd signs]
          [air trapping]
          [emphysema]
          [flattened diaphragm]
        ]
        [Interstitial / fibrotic pulmonary finding
          [interstitial pattern]
          [fibrotic band]
          [pulmonary fibrosis]
        ]
        [Airway structural finding
          [bronchiectasis]
        ]
        [chronic changes]
      ]
      [Parenchymal opacity / volume-loss finding
        [Opacity-related pulmonary finding
          [infiltrates]
          [alveolar pattern]
          [increased density]
        ]
        [Atelectatic / volume-loss finding
          [atelectasis]
          [laminar atelectasis]
          [volume loss]
          [hypoexpansion]
        ]
      ]
      [Nodular pulmonary finding
        [nodule]
        [pseudonodule]
        [calcified granuloma]
      ]
    ]
    [Pleural / diaphragmatic abnormality
      [Pleural fluid-related finding
        [pleural effusion]
        [costophrenic angle blunting]
      ]
      [Pleural thickening finding
        [pleural thickening]
        [apical pleural thickening]
      ]
      [Diaphragmatic finding
        [hemidiaphragm elevation]
        [diaphragmatic eventration]
      ]
    ]
    [Cardiomediastinal / vascular abnormality
      [Hilar finding
        [vascular hilar enlargement]
        [hilar enlargement]
      ]
      [Aortic finding
        [aortic elongation]
        [aortic atheromatosis]
      ]
      [cardiomegaly]
    ]
    [Musculoskeletal abnormality
      [Spinal degenerative / curvature finding
        [scoliosis]
        [kyphosis]
        [vertebral degenerative changes]
      ]
      [Thoracic osseous injury finding
        [callus rib fracture]
      ]
    ]
    [Hiatal abnormality
      [hiatal hernia]
    ]
  ]
  [Procedure / device finding
    [Surgical change
      [sternotomy]
      [suture material]
    ]
    [pacemaker]
  ]
]
\end{forest}
}
\caption{Taxonomy utilised for the PadChest dataset.}
\label{fig:padchest_taxonomy}
\end{figure}
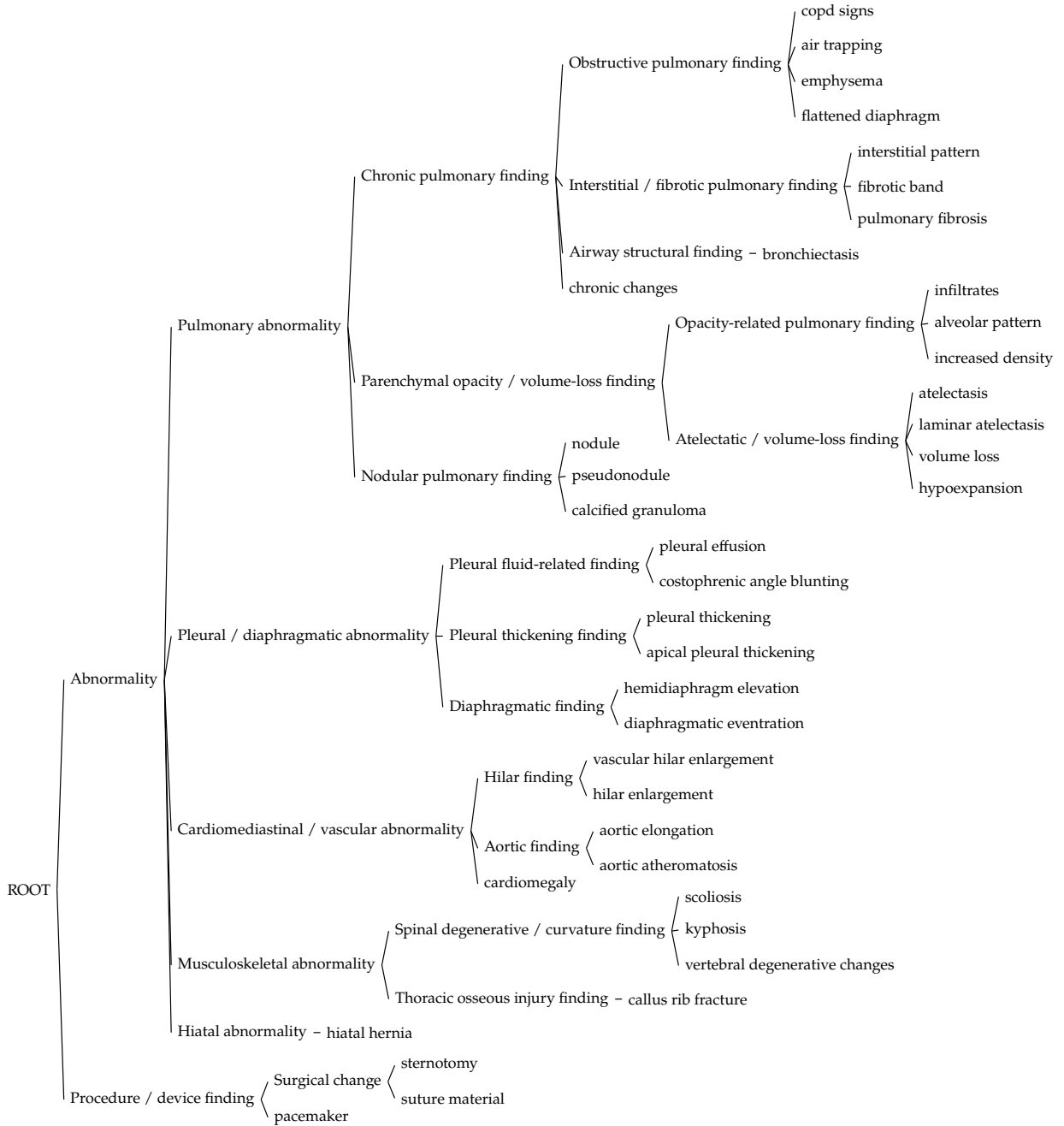

\subsection{ADPv2}
\label{app:adpv2_taxonomy}

Figure~\ref{fig:adpv2_taxonomy} shows the ADPv2 tissue taxonomy used to test the framework outside chest radiography.

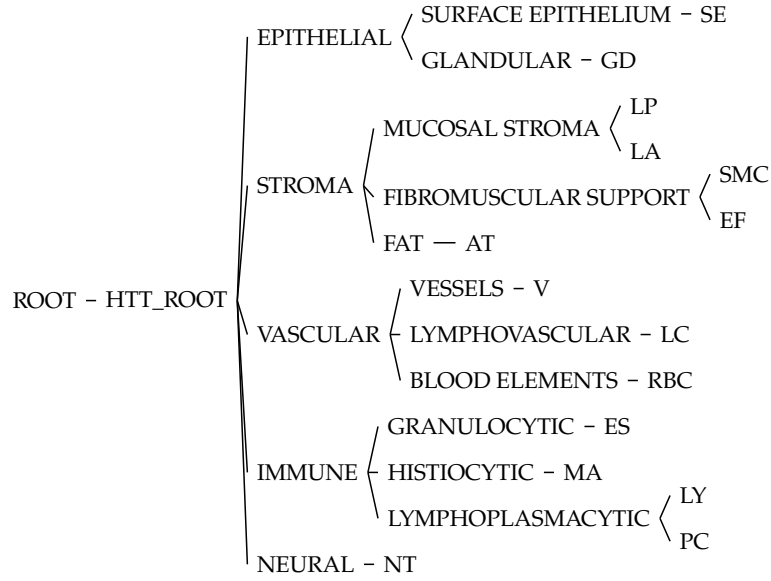
\begin{figure}[h]
\centering
\scalebox{0.95}{
\begin{forest}
for tree={
    draw=none,
    font=\footnotesize,
    edge={semithick},
    anchor=west,
    grow'=east,
    child anchor=west,
    parent anchor=east,
    l sep=4pt,
    s sep=4pt
}
[ROOT
  [HTT\_ROOT
    [EPITHELIAL
      [SURFACE EPITHELIUM
        [SE]
      ]
      [GLANDULAR
        [GD]
      ]
    ]
    [STROMA
      [MUCOSAL STROMA
        [LP]
        [LA]
      ]
      [FIBROMUSCULAR SUPPORT
        [SMC]
        [EF]
      ]
      [FAT
        [AT]
      ]
    ]
    [VASCULAR
      [VESSELS
        [V]
      ]
      [LYMPHOVASCULAR
        [LC]
      ]
      [BLOOD ELEMENTS
        [RBC]
      ]
    ]
    [IMMUNE
      [GRANULOCYTIC
        [ES]
      ]
      [HISTIOCYTIC
        [MA]
      ]
      [LYMPHOPLASMACYTIC
        [LY]
        [PC]
      ]
    ]
    [NEURAL
      [NT]
    ]
  ]
]
\end{forest}
}
\caption{Taxonomy utilised for the ADPv2 dataset.}
\label{fig:adpv2_taxonomy}
\end{figure}

\section{Experimental Details}
\label{app:exp_details}

This appendix provides implementation details for dataset construction, taxonomy structure, preprocessing, model architecture, optimisation, projection decoding, budget-swept evaluation, and statistical testing. All details below were checked against the released code and saved experiment configurations.

\subsection{Dataset Construction and Taxonomy Summaries}
\label{app:exp_details_datasets}

\begin{table}[h]
\centering
\caption{\textbf{Summary of the real-reader datasets and raw taxonomies used in the main experiments.}
For VinDr-CXR, the number of usable studies depends on which radiologist is treated as the expert, because we require exactly three reader annotations and retain only studies containing the target reader. Raw non-root taxonomy counts are reported here; after low-prevalence filtering, the VinDr-CXR experiments retained 34--35 labels depending on the expert, while CheXpert retained all 19 labels. The controlled-expert PadChest and ADPv2 protocols are detailed in the text immediately below.}
\label{tab:dataset_taxonomy_summary}
\scriptsize
\setlength{\tabcolsep}{5pt}
\renewcommand{\arraystretch}{1.15}
\begin{tabular}{@{}lcccccc@{}}
\toprule
\textbf{Dataset} & \textbf{Studies} & \textbf{Readers} & \textbf{Expert Tasks} & \textbf{Nodes} & \textbf{Internal} & \textbf{Leaves} \\
\midrule
VinDr-CXR  & 6,125--6,582 per task & 3 & 3 & 36 & 13 & 23 \\
CheXpert   & 500 & 5 & 5 & 19 & 9 & 10 \\
\bottomrule
\end{tabular}
\end{table}

\textbf{VinDr-CXR.}
We use the VinDr-CXR v1.0.0 training annotations from PhysioNet, specifically the study-level label file \texttt{image\_labels\_train.csv} together with the radiographs under \texttt{images/train/}. The label column \texttt{No finding} is excluded. The dataset contains three reader annotations per retained study. For each expert task, one radiologist is treated as the expert $m$, while the remaining two radiologists are averaged label-wise to form the reference label $y$. Thus, the reference consensus scores are soft values in $\{0, 0.5, 1\}$ at the original leaf labels before hierarchy expansion. The three expert tasks are the three radiologists present in the current pipeline, namely \texttt{R8}, \texttt{R9}, and \texttt{R10}. The resulting sample counts are 6,582 for \texttt{R8}, 6,125 for \texttt{R9}, and 6,467 for \texttt{R10}. The stratified train/validation/test splits are:
\[
\texttt{R8}: 3948/1317/1317,\qquad
\texttt{R9}: 3675/1225/1225,\qquad
\texttt{R10}: 3879/1294/1294.
\]
In the main text, VinDr-CXR results are averaged across these three expert tasks.

\textbf{CheXpert.}
We use the radiologist-comparison subset distributed through the repository layout under \texttt{data/chexpert/}, with reader-specific annotation CSVs \texttt{bc1\_gt.csv}, \texttt{bc2\_gt.csv}, \texttt{bc3\_gt.csv}, \texttt{bc5\_gt.csv}, together with the associated study directories under \texttt{data/chexpert/chexlocalize/}. We exclude \texttt{No Finding} and use the 13 CheXpert pathology/device labels listed in the code as base labels. For each expert task, one reader is treated as the expert $m$, while the remaining four readers are averaged label-wise to form the reference label $y$. Thus, the reference labels are soft scores in $\{0,0.25,0.5,0.75,1\}$ at the base labels before hierarchy expansion. The four expert tasks are \texttt{bc1}, \texttt{bc2}, \texttt{bc3}, and \texttt{bc5}. Each task contains 500 studies, split by multilabel-stratified sampling into 300 train, 100 validation, and 100 test examples. The raw taxonomy contains 19 non-root nodes, of which 9 are internal and 10 are leaves; in the reported runs, all 19 labels are retained after preprocessing.

\textbf{PadChest.}
We construct the PadChest controlled-expert cohort from cached expert-prediction CSVs together with the PadChest master label CSV. Only \texttt{PA} studies with \texttt{MethodLabel = Physician} are retained, and any overlap with the RNN-label pool is removed by image ID before splitting. Ground-truth labels are parsed from the PadChest label strings and then upward-closed through the evaluation taxonomy. The resulting retained cohort contains 6{,}320 studies/images. Splits are patient-disjoint and multilabel-stratified: first 20\% of patients are assigned to test, and then 1/8 of the remaining patients are assigned to validation, yielding an overall 70/10/20 patient split. Patient counts are stable across seeds (4{,}298 train / 615 val / 1{,}229 test), while study counts vary slightly because the split is patient-level; for example, seed 42 yields 4{,}413/640/1{,}267 studies. After applying the PadChest low-prevalence filter (\texttt{PADCHEST\_MIN\_POSITIVE\_SAMPLES}=10), the reported runs retain 61 labels (38 leaves and 23 internal nodes). The expert is a precomputed synthetic physician-style predictor loaded from cached per-image predictions rather than trained online during L2D; the codebase supports \texttt{pc\_carzero} and \texttt{pc\_carzero\_tuned} expert IDs.

\textbf{ADPv2.}
We construct the ADPv2 controlled-expert cohort from the \texttt{colon\_annotations\_batch\_*} and \texttt{colon\_metadata\_batch\_*} CSVs, merged on patch ID, with image paths resolved as \texttt{images/<PatchID>\_<PatchName>.png}. Ground-truth labels are upward-closed through the ADPv2 taxonomy. Splits are WSI-disjoint by \texttt{wsi\_name}: 20\% of WSIs are assigned to test and then 25\% of the remaining WSIs to validation, giving an approximate 60/20/20 split by WSI. Exact train/validation patch counts were not preserved in the current saved run summaries, but archived diagnostics confirm test split sizes between 3{,}985 and 4{,}000 patches across the reported seeds and 25 retained labels in all reported runs after applying \texttt{ADPV2\_MIN\_POSITIVE\_SAMPLES}=10. The ADPv2 experts are synthetic and built online from handcrafted morphology, colour, and texture features using truth-conditioned profile-specific calibration; available expert IDs are \texttt{generalist}, \texttt{epi}, \texttt{stroma}, and \texttt{immune}, with five-fold cross-fitting in the expert-construction pipeline.

\textbf{Taxonomy definitions.}
The taxonomies are implemented as explicit child-to-parent maps, i.e.\ dictionaries of the form
\[
\texttt{taxonomy[label]} = \texttt{parent\_label},
\]
with \texttt{"ROOT"} used as the sentinel top node. Auxiliary structures, including immediate-children maps, index-level ancestry maps, descendant matrices, topological orders, and reverse topological orders, are derived from this parent map at runtime. These are dataset-specific engineering taxonomies for coherent deferral experiments: they preserve clinically meaningful implication structure where possible, but may also introduce synthetic grouping nodes to obtain annotation-compatible trees. VinDr-CXR has a single root child (\texttt{Abnormality}); CheXpert has two root children (\texttt{Abnormality} and \texttt{Device}), so the kept structure is a forest. The main-text \textit{Lung Opacity} example is illustrative rather than a literal copy of every dataset taxonomy; in particular, the VinDr-CXR evaluation taxonomy uses an \textit{Opacity-related finding} grouping node and places \textit{Lung Opacity} as a sibling of \textit{Consolidation}, \textit{Infiltration}, and \textit{Edema}.

\subsection{Hierarchy Consistency of Reference and Expert Labels}
\label{app:exp_details_label_handling}

Both the reference labels $y$ and the expert labels $m$ are converted into hierarchy-consistent binary label vectors before hierarchical training and coherence evaluation. The preprocessing pipeline is the following.

\begin{itemize}[topsep=0pt,itemsep=0pt,parsep=0pt,leftmargin=*]
    \item For each expert task, we first construct the expert label matrix $m$ from the target reader and the soft reference matrix $y$ by averaging the remaining readers.
    \item We threshold the soft reference matrix at $0.5$ to obtain a binary copy used for hierarchy expansion.
    \item We add missing internal taxonomy nodes and enforce upward closure by propagating positive labels from children to all ancestors.
    \item For expert labels, we likewise add missing internal nodes, threshold any non-binary values at $0.5$, and apply the same upward-closure procedure.
    \item We then build the final reference matrix by keeping the original soft values on the original annotated labels and inserting hierarchy-expanded binary values for the newly introduced internal nodes.
    \item Finally, we enforce monotonicity of the reference scores by clipping each parent score from below by the maximum of its kept children.
\end{itemize}

Thus, upward closure is applied to both the expert labels and a binarised copy of the reference labels. Raw hierarchy violations are not discarded; instead, they are corrected by closure. The current pipeline does not apply any additional remapping for uncertain labels beyond the numeric labels already present in the input CSVs. The \texttt{No finding}/\texttt{No Finding} columns are excluded before model construction.

The exact order is:
\[
\begin{aligned}
&\text{raw reader annotations} \rightarrow \text{single-expert split} \rightarrow \\
&\text{reader averaging for } y \rightarrow \text{thresholded copy for hierarchy expansion} \rightarrow \\
&\text{internal-node insertion and upward closure} \rightarrow \text{parent-monotonicity correction for } y.
\end{aligned}
\]
After splitting, we apply low-prevalence filtering on the \emph{training} split only. A label is retained if it has at least 5 positive training examples on VinDr-CXR and at least 3 on CheXpert, where positivity is defined by thresholding at $0.5$. All ancestors of retained labels are also retained. This yields 34 kept labels for \texttt{R9} and 35 kept labels for \texttt{R8}/\texttt{R10} on VinDr-CXR, and 19 kept labels for every CheXpert task.

\textbf{Soft versus hard targets.}
After reader averaging, each node has a soft consensus score $s_v \in [0,1]$. For discrete correctness and evaluation, we define the hard majority label
\[
Y_v := \mathbb I[s_v \ge 0.5].
\]
In the reported experiments, the classification term of the defer-aware training loss uses $s_v$, while expert-correctness terms, correctness-based masking, and all reported F1-, balanced-accuracy-, and accuracy-based metrics use $Y_v$. The binary label notation used in the theory therefore corresponds to the thresholded consensus label $Y_v$ in experiments; the soft score $s_v$ is used only in the class-fit term during training.

\subsection{Model Architecture and Training Hyperparameters}
\label{app:exp_details_training}

All methods on a given dataset use the same image encoder and the same per-node ternary action heads.

\textbf{Architecture.}
For VinDr-CXR, the encoder is a \texttt{torchvision} ResNet-18 backbone, initialised with \texttt{ImageNet1K\_V1} weights. For CheXpert, the code is written to use a \texttt{torchxrayvision} DenseNet-121 backbone with weights \texttt{densenet121-res224-chex}. For PadChest, the encoder is an \texttt{XRVBackbone} built on the chest-X-ray-pretrained TorchXRayVision DenseNet-121. For ADPv2, the encoder is a \texttt{ResNet18Backbone}. The X-ray-pretrained backbones are frozen by default in the current codepath.

On top of the encoder, we use a shared projection block followed by one three-way linear head per node. Concretely, if the encoder output is $h(x)$, we compute
\[
z(x)=\mathrm{LayerNorm}\!\big(\mathrm{Dropout}(\mathrm{ReLU}(Wh(x)+b))\big),
\]
and then attach one linear map $\mathbb R^{256}\to\mathbb R^3$ per taxonomy node.

Additional architectural details are:
\begin{itemize}[topsep=0pt,itemsep=0pt,parsep=0pt,leftmargin=*]
    \item shared hidden dimension: 256
    \item dropout: 0.1
    \item ternary action heads: one linear head per kept node
\end{itemize}

\textbf{Image preprocessing.}
VinDr-CXR images are loaded as RGB JPEGs, resized to a cached size of $256\times256$, and then transformed as follows:
\begin{itemize}[topsep=0pt,itemsep=0pt,parsep=0pt,leftmargin=*]
    \item training: random resized crop to $224\times224$ with scale range $(0.8,1.0)$, random horizontal flip with probability $0.5$, ImageNet normalisation
    \item validation/test: centre crop to $224\times224$, ImageNet normalisation
\end{itemize}
The ImageNet normalisation constants are
\[
\mu=(0.485,0.456,0.406),\qquad \sigma=(0.229,0.224,0.225).
\]

CheXpert images are loaded as greyscale images, resized directly to $224\times224$, randomly flipped horizontally during training with probability $0.5$, converted to a single-channel tensor, and normalised to the range $[-1,1]$ via
\[
x \mapsto \frac{x/255 - 0.5}{0.5}.
\]
When the \texttt{torchvision} DenseNet fallback is used, the single channel is repeated to three channels inside the backbone wrapper before feature extraction.

PadChest uses the same greyscale $224\times224$ X-ray preprocessing and $[-1,1]$ normalisation as the XRV-based CheXpert pipeline.

ADPv2 patches are loaded as RGB images. Training uses \texttt{RandomResizedCrop(224)}, random horizontal flip, and colour jitter, while validation/test use deterministic resize preprocessing.

\textbf{Optimisation.}
Unless otherwise stated, all models are optimised with:
\begin{itemize}[topsep=0pt,itemsep=0pt,parsep=0pt]
    \item optimiser: AdamW
    \item base learning rate: $10^{-3}$
    \item weight decay: $10^{-4}$
    \item batch size: 128
    \item number of seeds: 5
\end{itemize}

\textbf{Checkpoint selection.}
The balanced-accuracy-stopped BR-family runs reported in the main text use validation \textbf{AU-SysBalAcc} as the stopping metric. Earlier runs in the repository also include AU-SysF1-I-stopped checkpoints, but we do not mix stopping rules within the main utility comparison. Early stopping uses patience 25 and minimum delta $10^{-4}$. For each method family, the reported model is the checkpoint with the best validation value of the designated stopping metric under the inference semantics corresponding to that family.

\textbf{Stage I / Stage II schedule.}
For BR-L2D and BR (cont.), Stage~I optimises the local BR objective
\[
\mathcal J_{\mathrm{BR}}
=
\sum_{v\in\mathcal V}\mathcal L_\phi(\etaLoc_v(x), y_v, m_v).
\]
For BR+RPO, Stage~I uses the same objective, and Stage~II fine-tunes through TBP using
\[
\mathcal J_{\mathrm{RPO}}
=
\sum_{v\in\mathcal V}\mathcal L_\phi(\muGlob_v(x), y_v, m_v).
\]
The exact schedule in the released runs is:
\begin{itemize}[topsep=0pt,itemsep=0pt,parsep=0pt]
    \item Stage~I epochs: 100
    \item Stage~II epochs: 100
    \item learning-rate drop at Stage~II: factor 0.1
    \item continuation schedule: identical to Stage~II in epoch count and LR change
\end{itemize}

\textbf{Loss function.}
The defer-aware surrogate $\mathcal L_\phi$ used in the reported experiments is the default \texttt{mozannar} loss in the code, i.e.\ the Mozannar--Sontag style defer-aware cross-entropy implemented in \texttt{lce\_eq7\_per\_sample}. The codebase also contains Mao-style and RS-style alternatives, but these were not used in the main reported runs.

\subsection{Exact Coherent Projection and Dynamic Programs}
\label{app:exp_details_projection}

BR+Projection is an inference-only structured method: training remains standard
BR-L2D, but both defer ranking and final action decoding are performed under the
coherent projection model induced by the Selective-Exclusion contract.

\textbf{Coherent feasible set.}
Let
\[
\mathcal C_{\mathrm{SE}}
=
\left\{
a \in \act^|v| :
\forall (p \to c)\in\mathcal E,\;
a_p=\absent \Rightarrow a_c=\absent,\;
a_p=\defer \Rightarrow a_c\in\{\absent,\defer\}
\right\}.
\]
This is exactly the set of joint action vectors satisfying taxonomic satisfiability, deductive closure, and the Selective-Exclusion parent-deferral contract from Proposition~\ref{prop:local_characterization}. For notational
simplicity we present the tree case below; for a forest, the dynamic programs
factor independently across roots and the root scores are summed.

\textbf{Full coherent projection.}
Ignoring the budget constraint, coherent projection is the MAP decode
\begin{equation}
\label{eq:proj_full_map}
\hat{\mathbf a}^{\mathrm{Proj}}(x)
\in
\arg\max_{a \in \mathcal C_{\mathrm{SE}}}
\sum_{v \in \mathcal V}\log \etaLoc_v(a_v \mid x).
\end{equation}

This dynamic program has the same form as the coherent Bayes decoder in Eq.~\eqref{eq:bayes_coherent}, but the scores are different. Eq.~\eqref{eq:proj_full_map} uses learned log primitive scores, whereas the Bayes decoder uses conditional risks or utilities. Thus projection is an exact coherent MAP decoder under the learned local action model; it is Bayes-risk optimal only when the local scores are calibrated as Bayes utilities for the target system loss.

\textbf{Structured action scores for defer ranking.}
To rank node-level deferrals under a budget, we do not use the local BR
marginals directly. Instead, for each node $t\in\mathcal V$ and action
$a^\star\in\{\absent,\present,\defer\}$, we define the constrained MAP action score
\begin{equation}
\label{eq:proj_action_value_app}
V_t(a^\star \mid x)
=
\max_{\mathbf a \in \mathcal C_{\mathrm{SE}}:\, a_t=a^\star}
\sum_{v \in \mathcal V}\log \etaLoc_v(a_v \mid x).
\end{equation}
This score measures, under the coherent projection model itself, how compatible
each action at node $t$ is with the best globally coherent configuration.

We then define the projection defer-priority score
\begin{equation}
\label{eq:proj_defer_score_app}
s_{\mathrm{Proj}}(x,t)
=
V_t(\defer\mid x)
-
\max\!\Big(V_t(\absent\mid x),\,V_t(\present\mid x)\Big).
\end{equation}
At budget $b$, the top $b\cdot N_{\text{total}}$ node-level decisions according
to \eqref{eq:proj_defer_score_app} are selected for deferral.

\textbf{Dynamic program for constrained action scores.}
Because the taxonomy is a tree, \eqref{eq:proj_action_value_app} is solved
exactly by dynamic programming. Let $\mathrm{Ch}(v)$ denote the immediate
children of node $v$, and let the default admissible child-action sets under
Selective-Exclusion be
\[
A(i)
=
\begin{cases}
\{\absent\}, & i=\absent,\\
\{\absent,\present,\defer\}, & i=\present,\\
\{\absent,\defer\}, & i=\defer.
\end{cases}
\]
For a fixed target node $t$ and target action $a^\star$, define the clamped
admissible set
\[
A_v(i;t,a^\star)
=
\begin{cases}
A(i)\cap\{a^\star\}, & v=t,\\
A(i), & v\neq t.
\end{cases}
\]
If $A_v(i;t,a^\star)=\varnothing$, the corresponding DP value is set to $-\infty$.

For non-root nodes, define
\[
G_v(j;t,a^\star)
=
\log \etaLoc_v(j\mid x)
+
\sum_{u\in\mathrm{Ch}(v)} F_u(j;t,a^\star),
\]
and the Bellman recursion
\[
F_v(i;t,a^\star)
=
\max_{j\in A_v(i;t,a^\star)} G_v(j;t,a^\star).
\]
For the root node $r$, the admissible set is
\[
A_r(t,a^\star)
=
\begin{cases}
\{a^\star\}, & r=t,\\
\{\absent,\present,\defer\}, & r\neq t.
\end{cases}
\]
The constrained action value is then
\[
V_t(a^\star\mid x)
=
\max_{j\in A_r(t,a^\star)} G_r(j;t,a^\star).
\]
In practice, this dynamic program is run once for each pair
$(t,a^\star)\in\mathcal V\times\{\absent,\present,\defer\}$.

\textbf{Projection under a deferral budget.}
At budget $b$, let $D_b(x)\subseteq\mathcal V$ denote the set of nodes selected
for deferral by \eqref{eq:proj_defer_score_app}. The final budgeted coherent
decode is
\begin{equation}
\label{eq:budget_proj_obj_new}
\hat{\mathbf a}^{\mathrm{Proj}}_b(x)
\in
\arg\max_{\mathbf a \in \mathcal C_{\mathrm{SE}} \cap \mathcal F_b(x)}
\sum_{v \in \mathcal V}\log \etaLoc_v(a_v \mid x),
\end{equation}
where
\[
\mathcal F_b(x)
=
\left\{
a\in\act^{|\mathcal V|} :
a_v=\defer \ \forall v\in D_b(x),
\quad
a_v\in\{\absent,\present\}\ \forall v\notin D_b(x)
\right\}.
\]
Thus, selected nodes are clamped to $\defer$, while all remaining nodes are
restricted to autonomous actions.

\textbf{Dynamic program for budgeted coherent decoding.}
For non-root nodes, define the budget-constrained admissible action set
\[
B_v(i;D_b(x))
=
\begin{cases}
\{\defer\}\cap A(i), & v\in D_b(x),\\
\{\absent,\present\}\cap A(i), & v\notin D_b(x).
\end{cases}
\]
For the root node $r$,
\[
B_r(D_b(x))
=
\begin{cases}
\{\defer\}, & r\in D_b(x),\\
\{\absent,\present\}, & r\notin D_b(x).
\end{cases}
\]
Define
\[
G_v^{b}(j)
=
\log \etaLoc_v(j\mid x)
+
\sum_{u\in\mathrm{Ch}(v)} F_u^{b}(j),
\]
and
\[
F_v^{b}(i)
=
\max_{j\in B_v(i;D_b(x))} G_v^{b}(j).
\]
For the root,
\[
F_r^{b}
=
\max_{j\in B_r(D_b(x))} G_r^{b}(j).
\]
The maximising actions are then recovered by backtracking. This produces the
exact MAP action vector under both the coherence constraints and the externally
imposed defer budget.

\textbf{Implementation note.}
In the released implementation, the raw top-\(k\) defer set selected by the global ranking is first passed through a deterministic feasibility closure before the exact dynamic program is run. The raw mask can be infeasible when it forces both an ancestor and a descendant to defer but leaves an intermediate node on that path unforced. Under Selective-Exclusion, such an intermediate node cannot be asserted present below a deferred parent; if it is not forced to defer, the budget mask restricts it to an autonomous action, so it is forced absent, making the lower deferred descendant infeasible. The closure fixes exactly this case by replacing \(D_b(x)\) with the minimal superset \(D_b^{\mathrm{cl}}(x)\supseteq D_b(x)\) that connects forced-deferred ancestor--descendant pairs by deferred paths. Concretely, if \(p\in D_b^{\mathrm{cl}}(x)\) and the subtree rooted at child \(c\) contains any node in \(D_b^{\mathrm{cl}}(x)\), then \(c\) is added to \(D_b^{\mathrm{cl}}(x)\), with this rule applied top-down until no more nodes are added.

\textbf{Complexity.}
Since $|\act|=3$ is fixed and the taxonomy is a tree, both dynamic programs are
linear in the number of nodes up to a small constant factor. The budgeted decode
is
\[
O(|\mathcal V|\,|\act|^2),
\]
which is linear in practice. The constrained action-score computation requires
one such DP for each $(t,a^\star)$ pair, yielding
\[
O(|\mathcal V|^2\,|\act|^3)
\]
in the straightforward exact implementation, which is acceptable for this
offline baseline.

\subsection{Budget-Swept Utility and Coherence Metrics}
\label{app:exp_details_metrics}

\textbf{Budget sweep.}
For each image $x_i$ and node $v$, we compute a method-specific defer-priority score
\[
s(x_i,v)
=
r_{iv}(\defer)
-
\max\big(r_{iv}(\absent),\,r_{iv}(\present)\big),
\]
where $r$ is the method-specific ranking value:
\begin{itemize}[topsep=0pt,itemsep=0pt,parsep=0pt,leftmargin=*]
    \item for BR-L2D and BR (cont.), $r_{iv}=\etaLoc_v(\cdot\mid x_i)$;
    \item for BR+Projection, $r_{iv}=V_v(\cdot\mid x_i)$, the constrained MAP action scores from \eqref{eq:proj_action_value_app};
    \item for BR+RPO, $r_{iv}=\muGlob_v(\cdot\mid x_i)$.
\end{itemize}
Flattening all $(i,v)$ pairs across the evaluation set gives a global ranking over
node-level decisions. At budget $b$, the top $b\cdot N_{\text{total}}$ ranked
decisions are clamped to $\defer$, where $N_{\text{total}}=N\cdot L$.

\textbf{Threshold discretisation and integration.}
The released runs use \texttt{NUM\_BUDGET\_THRESHOLDS}=101. Concretely, the implementation samples an integer threshold grid by taking the unique values in
\[
\mathrm{linspace}(0,N_{\text{total}},101+1),
\]
and always including both endpoints $0$ and $N_{\text{total}}$. Thus the number of evaluated thresholds is at most 102, and can be smaller when duplicate integer thresholds arise. All AUC-style quantities are then computed by trapezoidal integration using \texttt{sklearn.metrics.auc} over budget fraction.

\textbf{Why we emphasise balanced metrics.}
Chest X-ray labels are sparse and strongly negative-dominant. Metrics based on balanced accuracy and macro-F1 therefore provide a more faithful picture of comparative system utility than raw correctness averaged over all study-label decisions. For that reason, the main comparison emphasises AU-SysBalAcc and macro-F1 metrics. These summaries are descriptive, however, and should not be read as a replacement for a deployment-specific cost-sensitive risk model. In particular, the Bayes characterisation in Section~\ref{sec:bayes_coherent} assumes additive nodewise system risk, whereas F1 and balanced-accuracy AUCs are evaluation summaries over a budget sweep.

\textbf{AU-SysBalAcc.}
Let $\hat Y^{\mathrm{sys}}(b)\in\{0,1\}^{N\times L}$ denote the system predictions at budget $b$, and flatten the $(N\times L)$ study-label decisions into one binary vector. Let
\[
\mathrm{TPR}(b)=\frac{\mathrm{TP}(b)}{\mathrm{TP}(b)+\mathrm{FN}(b)},
\qquad
\mathrm{TNR}(b)=\frac{\mathrm{TN}(b)}{\mathrm{TN}(b)+\mathrm{FP}(b)}.
\]
The system balanced accuracy at budget $b$ is
\[
\mathrm{BalAcc}^{\mathrm{sys}}(b)
=
\frac{1}{2}\bigl(\mathrm{TPR}(b)+\mathrm{TNR}(b)\bigr).
\]
We then define
\[
\mathrm{AU\mbox{-}SysBalAcc}
=
\int_0^1 \mathrm{BalAcc}^{\mathrm{sys}}(b)\,db.
\]

\textbf{AU-SysF1-L (pooled label-wise F1).}
Let $\hat Y^{\mathrm{sys}}(b)\in\{0,1\}^{N\times L}$ denote the system predictions at budget $b$. In the current implementation, AU-SysF1-L is computed by flattening all study-label decisions into one long binary vector and evaluating a single F1 score:
\[
\mathrm{SysF1}_L(b)
=
\mathrm{F1}\!\left(\mathrm{vec}\bigl(\hat Y^{\mathrm{sys}}(b)\bigr),\,\mathrm{vec}(Y)\right),
\]
where $Y$ is the reference label matrix and $\mathrm{vec}(\cdot)$ denotes flattening. This is therefore a pooled or micro-style F1 over study-label decisions rather than a macro average over labels. We then define
\[
\mathrm{AU\mbox{-}SysF1\mbox{-}L}
=
\int_0^1 \mathrm{SysF1}_L(b)\,db,
\]
approximated numerically by the sampled threshold grid and trapezoidal integration described above.

\textbf{AU-F1-Lmacro.}
The macro label-wise F1 treats each label equally. At budget $b$,
\[
\mathrm{F1}_{L,\mathrm{macro}}(b)
=
\frac{1}{L}
\sum_{v=1}^{L}
\mathrm{F1}\!\left(\hat Y^{\mathrm{sys}}_{:v}(b),\,Y_{:v}\right),
\]
and
\[
\mathrm{AU\mbox{-}F1\mbox{-}Lmacro}
=
\int_0^1 \mathrm{F1}_{L,\mathrm{macro}}(b)\,db.
\]

\textbf{AU-SysF1-I / AU-F1-Imacro.}
The instance-wise system F1 computes one F1 score per study and averages across studies:
\[
\mathrm{SysF1}_I(b)
=
\frac{1}{N}
\sum_{i=1}^{N}
\mathrm{F1}\!\left(\hat Y^{\mathrm{sys}}_{i:}(b),\,Y_{i:}\right),
\]
and
\[
\mathrm{AU\mbox{-}SysF1\mbox{-}I}
=
\int_0^1 \mathrm{SysF1}_I(b)\,db.
\]
In the current codepath, AU-F1-Imacro is populated from this same per-study F1 computation, so AU-F1-Imacro and AU-SysF1-I are numerically identical. We use the AU-SysF1-I name in the main text because it is the more direct description of the reported quantity.

\textbf{Neighbourhood-partition incoherence.}
Let $\mathcal P\subseteq \mathcal V$ be the set of kept parent nodes with at least one kept immediate child. For each image $i$, parent node $p\in\mathcal P$, and budget $b$, define the local neighbourhood action tuple
\[
a^{(i)}_{N(p)}(b)
=
\big(a_{ip}(b), \{a_{ic}(b)\}_{c\in C(p)}\big).
\]
Using the partition induced by Proposition~\ref{prop:local_characterization}, define indicator functions
\[
\delta_{\mathrm{tax}}(a^{(i)}_{N(p)}(b)),\quad
\delta_{\mathrm{ded}}(a^{(i)}_{N(p)}(b)),\quad
\delta_{\mathrm{del}}(a^{(i)}_{N(p)}(b)),
\]
which take value $1$ if the local neighbourhood falls into the corresponding incoherent class and $0$ otherwise. The ``any incoherence'' indicator is
\[
\delta_{\mathrm{any}}
=
\max\{\delta_{\mathrm{tax}},\,\delta_{\mathrm{ded}},\,\delta_{\mathrm{del}}\}.
\]
The neighbourhood-partition rates at budget $b$ are then
\[
R_{\mathrm{type}}(b)
=
\frac{1}{N|\mathcal P|}
\sum_{i=1}^{N}
\sum_{p\in\mathcal P}
\delta_{\mathrm{type}}(a^{(i)}_{N(p)}(b)),
\qquad
\mathrm{type}\in\{\mathrm{tax},\mathrm{ded},\mathrm{del},\mathrm{any}\}.
\]
In the implementation, the partition uses immediate children only and ``first violated predicate wins'':
\begin{enumerate}[topsep=0pt,itemsep=0pt,parsep=0pt,leftmargin=*]
    \item taxonomic contradiction: parent absent and any child present
    \item delegation violation: parent defer and any child present
    \item deductive defect: parent absent, no child present, and any child defer
    \item coherent: all remaining neighbourhoods
\end{enumerate}
Finally,
\[
\mathrm{AU\mbox{-}Neigh\mbox{-}Any}
=
\int_0^1 R_{\mathrm{any}}(b)\,db,
\]
and similarly for the defect-specific neighbourhood-partition rates used throughout the paper.

\textbf{Edge-weighted incoherence.}
In addition to the neighbourhood partition, we also compute an edge-weighted view in which the unit of analysis is the immediate parent--child edge. Let $\mathcal E_{\text{kept}}$ be the set of kept immediate edges. Then the denominator is
\[
N\,|\mathcal E_{\text{kept}}|,
\]
and the edge-weighted defect rates at budget $b$ are
\[
R^{\mathrm{edge}}_{\mathrm{tax}}(b)
=
\frac{1}{N|\mathcal E_{\text{kept}}|}
\sum_{i=1}^{N}
\sum_{(p\to c)\in\mathcal E_{\text{kept}}}
\mathbbm{1}[a_{ip}(b)=\absent,\ a_{ic}(b)=\present],
\]
\[
R^{\mathrm{edge}}_{\mathrm{ded}}(b)
=
\frac{1}{N|\mathcal E_{\text{kept}}|}
\sum_{i=1}^{N}
\sum_{(p\to c)\in\mathcal E_{\text{kept}}}
\mathbbm{1}[a_{ip}(b)=\absent,\ a_{ic}(b)=\defer],
\]
\[
R^{\mathrm{edge}}_{\mathrm{del}}(b)
=
\frac{1}{N|\mathcal E_{\text{kept}}|}
\sum_{i=1}^{N}
\sum_{(p\to c)\in\mathcal E_{\text{kept}}}
\mathbbm{1}[a_{ip}(b)=\defer,\ a_{ic}(b)=\present],
\]
and
\[
R^{\mathrm{edge}}_{\mathrm{any}}(b)
=
\frac{1}{N|\mathcal E_{\text{kept}}|}
\sum_{i=1}^{N}
\sum_{(p\to c)\in\mathcal E_{\text{kept}}}
\mathbbm{1}\!\left[
\begin{array}{l}
(a_{ip}(b)=\absent,\ a_{ic}(b)=\present)\ \vee\\
(a_{ip}(b)=\absent,\ a_{ic}(b)=\defer)\ \vee\\
(a_{ip}(b)=\defer,\ a_{ic}(b)=\present)
\end{array}
\right].
\]
Finally,
\[
\mathrm{AU\mbox{-}EW\mbox{-}Any}
=
\int_0^1 R^{\mathrm{edge}}_{\mathrm{any}}(b)\,db,
\]
and similarly for the defect-specific edge-weighted AUCs used in the appendix decomposition tables.

\subsection{Statistical Testing}
\label{app:exp_details_stats}

All hypothesis tests are paired one-sided Wilcoxon signed-rank tests. The main-text comparisons use the $15$ matched balanced-accuracy-stopped VinDr-CXR runs pooled across R8/R9/R10 and the $20$ matched balanced-accuracy-stopped CheXpert runs pooled across bc1/bc2/bc3/bc5. For utility metrics, the alternative hypothesis is that the first method has higher performance; for incoherence metrics, the alternative is that the first method has lower incoherence. All reported $p$-values are unadjusted.

\begin{table*}[h]
\centering
\caption{\textbf{Pairwise one-sided Wilcoxon tests for the VinDr-CXR main comparison.}
Each cell reports the direction favoured by the test together with the unadjusted $p$-value. ``ns'' denotes that the reported one-sided comparison was not significant at $p<0.05$.}
\label{tab:vindr_balacc_stats}
\scriptsize
\setlength{\tabcolsep}{4pt}
\renewcommand{\arraystretch}{1.10}
\resizebox{\linewidth}{!}{%
\begin{tabular}{@{}lcccccc@{}}
\toprule
\textbf{Pair}
& \textbf{AU-BalAcc}
& \textbf{AU-F1-Lmacro}
& \textbf{AU-SysF1-I}
& \textbf{AU-SysF1-L}
& \textbf{EW Any}
& \textbf{NP Any} \\
\midrule
BR vs Projection
& Projection $(0.0473)$
& Projection $(0.0473)$
& Projection $(0.0177)$
& ns
& Projection $(3.1\times 10^{-5})$
& Projection $(3.1\times 10^{-5})$ \\
BR vs cont.
& BR $(3.1\times 10^{-5})$
& cont. $(0.0151)$
& ns
& cont. $(0.0010)$
& BR $(0.0090)$
& BR $(0.0108)$ \\
BR vs RPO
& ns
& RPO $(3.1\times 10^{-5})$
& ns
& RPO $(9.2\times 10^{-5})$
& RPO $(3.1\times 10^{-5})$
& RPO $(3.1\times 10^{-5})$ \\
Projection vs cont.
& Projection $(3.1\times 10^{-5})$
& cont. $(0.0365)$
& ns
& cont. $(5.8\times 10^{-4})$
& Projection $(3.1\times 10^{-5})$
& Projection $(3.1\times 10^{-5})$ \\
Projection vs RPO
& ns
& RPO $(3.1\times 10^{-5})$
& ns
& RPO $(9.2\times 10^{-5})$
& Projection $(3.1\times 10^{-5})$
& Projection $(3.1\times 10^{-5})$ \\
cont. vs RPO
& RPO $(9.2\times 10^{-5})$
& RPO $(6.1\times 10^{-5})$
& ns
& RPO $(0.0277)$
& RPO $(3.1\times 10^{-5})$
& RPO $(3.1\times 10^{-5})$ \\
\bottomrule
\end{tabular}%
}
\end{table*}

\begin{table}[h]
\centering
\caption{\textbf{Per-expert mean utility winners for the VinDr-CXR main comparison.}}
\label{tab:vindr_by_expert_best}
\small
\setlength{\tabcolsep}{5pt}
\renewcommand{\arraystretch}{1.10}
\begin{tabular}{@{}lcccc@{}}
\toprule
\textbf{Expert}
& \textbf{Best AU-BalAcc}
& \textbf{Best AU-F1-Lmacro}
& \textbf{Best AU-SysF1-I}
& \textbf{Best AU-SysF1-L} \\
\midrule
R8  & Projection $0.8421$ & RPO $0.4850$ & Projection $0.4914$ & RPO $0.7519$ \\
R9  & RPO $0.8689$ & RPO $0.5175$ & RPO $0.5510$ & RPO $0.7867$ \\
R10 & Projection $0.8483$ & RPO $0.4924$ & RPO $0.5061$ & RPO $0.7644$ \\
\bottomrule
\end{tabular}
\end{table}

\subsection{Additional CheXpert Statistics}
\label{app:chexpert_partial}

The main-text CheXpert rows in Table~\ref{tab:main_results} pool the four completed real-reader experts bc1, bc2, bc3, and bc5, with five matched seeds per expert ($20$ matched runs total); the same pooling is used for the incoherence decompositions. We collect the corresponding pooled significance tests and per-expert utility winners here.

\begin{table*}[h]
\centering
\caption{\textbf{Pairwise one-sided Wilcoxon tests for the pooled CheXpert real-reader sweep.}
Each cell reports the direction favoured by the test together with the unadjusted $p$-value. The comparison pools the $20$ matched runs from bc1/bc2/bc3/bc5. ``ns'' denotes that the reported one-sided comparison was not significant at $p<0.05$.}
\label{tab:chexpert_partial_stats}
\scriptsize
\setlength{\tabcolsep}{4pt}
\renewcommand{\arraystretch}{1.10}
\begin{tabular}{@{}lcccccc@{}}
\toprule
\textbf{Pair}
& \textbf{AU-BalAcc}
& \textbf{AU-F1-Lmacro}
& \textbf{AU-SysF1-I}
& \textbf{AU-SysF1-L}
& \textbf{EW Any}
& \textbf{NP Any} \\
\midrule
BR vs Projection
& ns
& ns
& ns
& ns
& Projection $(9.5\times 10^{-7})$
& Projection $(4.4\times 10^{-5})$ \\
BR vs cont.
& cont.\ $(0.0086)$
& cont.\ $(2.4\times 10^{-4})$
& cont.\ $(0.0133)$
& cont.\ $(0.0266)$
& ns
& ns \\
BR vs RPO
& RPO $(3.1\times 10^{-5})$
& RPO $(3.5\times 10^{-4})$
& RPO $(0.0012)$
& RPO $(1.8\times 10^{-5})$
& RPO $(9.5\times 10^{-7})$
& RPO $(9.5\times 10^{-7})$ \\
Projection vs cont.
& cont.\ $(4.3\times 10^{-4})$
& cont.\ $(6.0\times 10^{-4})$
& cont.\ $(0.0181)$
& cont.\ $(0.0047)$
& Projection $(4.4\times 10^{-5})$
& Projection $(9.5\times 10^{-7})$ \\
Projection vs RPO
& RPO $(3.1\times 10^{-5})$
& RPO $(8.4\times 10^{-4})$
& RPO $(0.0036)$
& RPO $(1.8\times 10^{-5})$
& Projection $(0.0139)$
& Projection $(0.0139)$ \\
cont.\ vs RPO
& RPO $(0.0200)$
& ns
& ns
& RPO $(0.0148)$
& RPO $(4.4\times 10^{-5})$
& RPO $(9.5\times 10^{-7})$ \\
\bottomrule
\end{tabular}
\end{table*}

\begin{table}[h]
\centering
\caption{\textbf{Per-expert mean utility winners for the pooled CheXpert real-reader sweep.}
Each entry reports the best method for that expert, averaged over the five matched seeds.}
\label{tab:chexpert_partial_by_expert}
\small
\setlength{\tabcolsep}{5pt}
\renewcommand{\arraystretch}{1.10}
\begin{tabular}{@{}lcccc@{}}
\toprule
\textbf{Expert}
& \textbf{Best AU-BalAcc}
& \textbf{Best AU-F1-Lmacro}
& \textbf{Best AU-SysF1-I}
& \textbf{Best AU-SysF1-L} \\
\midrule
bc1 & RPO $0.8299$ & RPO $0.5019$ & RPO $0.6599$ & RPO $0.7668$ \\
bc2 & RPO $0.8388$ & RPO $0.5428$ & RPO $0.6775$ & RPO $0.7810$ \\
bc3 & RPO $0.8353$ & RPO $0.5111$ & RPO $0.6278$ & RPO $0.7633$ \\
bc5 & RPO $0.8198$ & cont.\ $0.5291$ & RPO $0.6426$ & RPO $0.7462$ \\
\bottomrule
\end{tabular}
\end{table}

\subsection{Reproducibility Details}
\label{app:exp_details_repro}

Unless a seed list is explicitly supplied, the code generates the five seeds used in the released experiments via
\[
42,\ 153,\ 264,\ 375,\ 486,
\]
i.e.\ \texttt{42 + 111*i} for \texttt{i=0,\dots,4}. The reported experiments were run on a workstation with an AMD Ryzen 9 7950X CPU (16 cores / 32 threads), 62 GiB RAM, and a single NVIDIA GeForce RTX 4090 GPU with 24 GiB VRAM. The code runs on CUDA when \texttt{torch.cuda.is\_available()} is true and otherwise falls back to CPU.

The software stack in the current release environment is:
\begin{itemize}[topsep=0pt,itemsep=0pt,parsep=0pt,leftmargin=*]
    \item Python 3.12.2
    \item PyTorch 2.6.0
    \item torchvision 0.21.0
    \item scikit-learn 1.8.0
    \item NumPy 2.2.6
    \item pandas 2.2.3
    \item iterative-stratification 0.1.9
\end{itemize}

Code-level implementation details not central to the main text include:
\begin{itemize}[topsep=0pt,itemsep=0pt,parsep=0pt,leftmargin=*]
    \item VinDr-CXR training uses random resized crop and horizontal flip; validation/test use centre crop.
    \item CheXpert training uses random horizontal flip only; validation/test are deterministic.
    \item The budget-sweep AUCs use 101 sampled intervals (at most 102 evaluated thresholds including endpoints).
    \item Projection decoding clamps local probabilities below by $10^{-12}$ before taking logs.
\end{itemize}

\section{Exact Decoding under the TBP Joint Model}
\label{app:tbp_exactdecode}

The main-text RPO results use a fast budgeted decoder based on TBP marginals. This decoder is efficient and yields only negligible residual incoherence, but it is not itself guaranteed coherent. To separate coherence of the TBP joint model from coherence of the external decoder, we additionally evaluate an exact MAP decoder under the TBP model.

\textbf{Exact TBP decode.}
Let
\[
P_{\mathrm{TBP}}(\mathbf a\mid x)
=
\etaLoc_r(a_r\mid x)
\prod_{v\neq r}\mathbf T_v(x)_{a_{pa(v)},a_v}
\]
be the TBP joint model from \eqref{eq:tbp_joint}. At budget \(b\), let \(D_b^{\mathrm{raw}}(x)\) be the set of nodes selected for deferral by the same marginal ranking used in the main-text RPO experiments. As in Appendix~\ref{app:exp_details_projection}, we apply the deterministic feasibility closure that adds exactly the intermediate nodes needed to connect forced-deferred ancestor--descendant pairs by deferred paths, yielding \(D_b^{\mathrm{cl}}(x)\). We then decode
\begin{equation}
\label{eq:tbp_exact_decode_app}
\hat{\mathbf a}^{\mathrm{Exact}}_b(x)
\in
\arg\max_{\mathbf a\in\mathcal C_{\mathrm{SE}}\cap \mathcal F_b^{\mathrm{cl}}(x)}
\log P_{\mathrm{TBP}}(\mathbf a\mid x),
\end{equation}
where
\[
\mathcal F_b^{\mathrm{cl}}(x)
=
\left\{
a\in\act^{|\mathcal V|} :
a_v=\defer\ \forall v\in D_b^{\mathrm{cl}}(x),
\quad
a_v\in\{\absent,\present\}\ \forall v\notin D_b^{\mathrm{cl}}(x)
\right\}.
\]
Because \eqref{eq:tbp_exact_decode_app} maximises over the coherent feasible set itself, the resulting deterministic action vector is exactly coherent.
As with projection, this is exact MAP decoding under the learned TBP joint model, not necessarily Bayes-risk decoding for the additive system loss unless the joint scores are calibrated to that risk.

\begin{table}[h]
\centering
\caption{\textbf{Exact decoding under the TBP joint model.}
Values are mean $\pm$ standard deviation over matched runs.}
\label{tab:tbp_exactdecode}
\small
\setlength{\tabcolsep}{6pt}
\renewcommand{\arraystretch}{1.10}
\begin{tabular}{@{}lccc@{}}
\toprule
\textbf{Dataset}
& \textbf{AU-SysF1-L}
& \textbf{AU-SysF1-I}
& \textbf{AU-Neigh-Any} \\
\midrule
VinDr-CXR
& $0.7806 \pm 0.0141$
& $0.5175 \pm 0.0247$
& $0.0000 \pm 0.0000$ \\
CheXpert
& $0.6983 \pm 0.0185$
& $0.6004 \pm 0.0303$
& $0.0000 \pm 0.0000$ \\
\bottomrule
\end{tabular}
\end{table}

As expected, exact decoding drives both edge-weighted and neighbourhood-partition incoherence to zero on both datasets. Its utility trade-off is dataset-dependent: on VinDr-CXR, exact decoding is comparable to the fast marginal RPO decoder and slightly improves the reported AU-SysF1-L and AU-SysF1-I values in this comparison, whereas on CheXpert it reduces utility. Thus, exact deterministic coherence under the TBP model is feasible, but the fast marginal decoder remains the preferred high-throughput operating point when its small residual incoherence is acceptable.

\textbf{Closure diagnostics.}
Table~\ref{tab:tbp_exactdecode_closure} quantifies the effect of the feasibility closure. In both datasets it is rare and negligible: closure is activated in fewer than $0.4\%$ of budgeted example-threshold pairs on VinDr-CXR and fewer than $0.05\%$ on CheXpert, with mean added deferred nodes below $0.005$ and realised-to-raw defer ratios essentially equal to one.

\begin{table}[h]
\centering
\caption{\textbf{Feasibility-closure diagnostics for exact TBP decoding.}}
\label{tab:tbp_exactdecode_closure}
\small
\setlength{\tabcolsep}{5pt}
\renewcommand{\arraystretch}{1.10}
\begin{tabular}{@{}lcccc@{}}
\toprule
\textbf{Dataset}
& \makecell{\textbf{Closure}\\\textbf{activation rate}}
& \makecell{\textbf{Mean added}\\\textbf{deferred nodes}}
& \makecell{\textbf{Max added}\\\textbf{deferred nodes}}
& \makecell{\textbf{Realised/raw}\\\textbf{defer ratio}} \\
\midrule
VinDr-CXR & $0.00380$ & $0.00412$ & $3$ & $1.00024$ \\
CheXpert  & $0.000424$ & $0.000424$ & $1$ & $1.000045$ \\
\bottomrule
\end{tabular}
\end{table}

\section{Inference Runtime}
\label{app:exp_details_runtime}

To quantify the computational cost of coherent inference, we benchmarked per-study inference time and full budget-sweep evaluation time for three structured decoders: the fast marginal RPO decoder, the exact TBP MAP decoder from Appendix~\ref{app:tbp_exactdecode}, and exact coherent projection. All timings were measured on CPU, exclude data-loading overhead, and are intended as representative runtime profiles rather than deployment-time service estimates. For exact projection, runtime includes both computation of the constrained MAP action scores $V_v(a\mid x)$ for all nodes and actions and the final budgeted coherent decode. For exact TBP decoding, runtime includes the TBP forward pass and a single exact MAP decode under the budgeted feasible set.

For per-study decode timing, we use a fixed within-study defer fraction of
$10\%$. This is only for runtime benchmarking; the main paper's evaluation still
uses the global dataset-level budget sweep. We additionally report the full
runtime of the offline budget-swept evaluation with 101 budget thresholds.

\begin{table*}[h]
\centering
\caption{\textbf{Runtime of coherent decoding variants.}
Per-study timings are measured on CPU and exclude data-loading overhead. Budget-sweep timings correspond to the full offline evaluation on the test split with 101 budget thresholds. Entries in the final column are reported as per-study speedup / budget-sweep speedup relative to exact projection.}
\label{tab:runtime}
\small
\setlength{\tabcolsep}{6pt}
\renewcommand{\arraystretch}{1.10}
\begin{tabular}{@{}llcccc@{}}
\toprule
\textbf{Dataset} & \textbf{Method}
& \makecell{\textbf{Total}\\\textbf{ms/study}}
& \makecell{\textbf{Budget Sweep}\\\textbf{sec/split}}
& \makecell{\textbf{Slowdown}\\\textbf{vs.\ fast RPO}}
& \makecell{\textbf{Speedup}\\\textbf{vs.\ Projection}\\\textbf{(study / sweep)}} \\
\midrule
\multirow{3}{*}{VinDr-CXR}
& BR-L2D+RPO
& 0.225
& 0.699
& 1.00$\times$
& 1074$\times$ / 1287$\times$ \\
& BR-L2D+RPO+Exact Decode
& 1.178
& 97.21
& 5.25$\times$
& 205$\times$ / 9.25$\times$ \\
& BR+Projection
& 241.71
& 899.36
& 1074$\times$ / 1287$\times$
& 1.00$\times$ \\
\midrule
\multirow{3}{*}{CheXpert}
& BR-L2D+RPO
& 0.611
& 0.300
& 1.00$\times$
& 138$\times$ / 123$\times$ \\
& BR-L2D+RPO+Exact Decode
& 1.045
& 4.891
& 1.71$\times$
& 80$\times$ / 7.5$\times$ \\
& BR+Projection
& 84.287
& 36.862
& 138$\times$ / 123$\times$
& 1.00$\times$ \\
\bottomrule
\end{tabular}
\end{table*}

These timings show that exact TBP decoding occupies a useful middle ground. It is slower than the fast marginal decoder, especially for full budget-sweep evaluation, but it remains dramatically cheaper than exact projection while eliminating incoherence exactly. This supports the main practical interpretation of the paper: fast RPO is the preferred deployment-time operating point because it delivers strong utility at near-BR cost while maintaining negligible incoherence, and exact TBP decoding provides a coherent deterministic decoder under the learned TBP joint model that is far less expensive than projection.

\section{Extension Beyond Trees: DAG Taxonomies}
\label{app:dag_extension}

The main paper assumes that the taxonomy is a tree. This assumption is useful because it makes coherence local and gives exact linear-time dynamic programs for both Bayes decoding and coherent projection. Many medical vocabularies, however, are more naturally represented as directed acyclic graphs (DAGs), where a finding may have multiple parents. This appendix sketches how the same conceptual framework can be extended to DAG taxonomies, and clarifies which parts of the method transfer directly and which require new inference machinery.

\subsection{DAG label space and action coherence}

Let
\[
\mathcal G=(\mathcal V,\mathcal E)
\]
be a directed acyclic taxonomy, where an edge $(p\to c)\in\mathcal E$ means that child concept $c$ implies parent concept $p$. As in the tree setting, labels obey upward implication:
\begin{equation}
\label{eq:dag_upward}
y_c=1 \Rightarrow y_p=1
\qquad
\forall (p\to c)\in\mathcal E .
\end{equation}
Equivalently, an absent parent forces all of its descendants absent. The taxonomy-consistent binary label space becomes
\[
\mathcal Y_{\mathcal G}
:=
\left\{
y\in\{0,1\}^{|\mathcal V|}:
\forall (p\to c)\in\mathcal E,\;
y_c=1\Rightarrow y_p=1
\right\}.
\]
For an action vector $a\in\act^{|\mathcal V|}$, the compatibility set is defined exactly as before:
\[
\compat{a}
:=
\left\{
y\in\mathcal Y_{\mathcal G}:
\forall v\in\mathcal V,\;
a_v\neq\defer \Rightarrow y_v=a_v
\right\}.
\]
Thus $a$ is taxonomically satisfiable if $\compat{a}\neq\emptyset$.

Under Selective-Exclusion, a deferred parent may selectively exclude descendants, but may not allow a positive child assertion that commits the deferred parent. For a DAG, the natural edge-wise contract is
\begin{equation}
\label{eq:dag_se_edge}
a_p=\defer \Rightarrow a_c\in\{\absent,\defer\}
\qquad
\forall (p\to c)\in\mathcal E .
\end{equation}
Together with deductive closure, this yields the following DAG analogue of the tree constraints:
\begin{equation}
\label{eq:dag_local_forbidden}
\begin{aligned}
a_p=\absent,\ a_c=\present
&\qquad \text{taxonomic contradiction},\\
a_p=\defer,\ a_c=\present
&\qquad \text{delegation violation},\\
a_p=\absent,\ a_c=\defer
&\qquad \text{deductive defect},
\end{aligned}
\qquad
\forall (p\to c)\in\mathcal E .
\end{equation}
These are the same local forbidden parent--child patterns as in Proposition~\ref{prop:local_characterization}. The difference is that, in a DAG, a node may have multiple parents, so its feasible action is determined by the conjunction of all parent constraints.

\subsection{Multiple-parent feasibility}

Let $P(v)=\{p:(p\to v)\in\mathcal E\}$ denote the set of parents of node $v$. In the tree case, the parent action $a_{pa(v)}$ determines the feasible child-action set through
\[
\Gamma_{\mathrm{SE}}(\absent)=\{\absent\},\qquad
\Gamma_{\mathrm{SE}}(\present)=\act,\qquad
\Gamma_{\mathrm{SE}}(\defer)=\{\absent,\defer\}.
\]
For a DAG, the corresponding feasible set for node $v$ is the intersection of all parent-imposed feasible sets:
\begin{equation}
\label{eq:dag_gamma_intersection}
\Gamma_{\mathrm{DAG}}(a_{P(v)})
:=
\bigcap_{p\in P(v)}
\Gamma_{\mathrm{SE}}(a_p).
\end{equation}
Thus:
\[
\Gamma_{\mathrm{DAG}}(a_{P(v)})=
\begin{cases}
\{\absent\}, & \exists p\in P(v): a_p=\absent,\\
\{\absent,\defer\}, & \nexists p:a_p=\absent \text{ and } \exists p:a_p=\defer,\\
\act, & \forall p\in P(v): a_p=\present.
\end{cases}
\]
This has the desired semantics. If any parent is asserted absent, then $v$ is forced absent. If no parent is absent but at least one parent is deferred, then $v$ may be absent or deferred, but cannot be asserted present. Only when all parents are asserted present may $v$ freely choose among $\absent,\present,\defer$.

The coherent DAG action set can therefore be written as
\begin{equation}
\label{eq:dag_coherent_set}
\mathcal C_{\mathrm{SE}}^{\mathrm{DAG}}
=
\left\{
a\in\act^{|\mathcal V|}:
\forall v\in\mathcal V,\;
a_v\in\Gamma_{\mathrm{DAG}}(a_{P(v)})
\right\},
\end{equation}
with the convention that roots have no parent restriction and may take any action in $\act$.

\subsection{Coherent projection on a DAG}

The exact projection objective transfers directly. Given local primitive distributions $\etaLoc_v(a\mid x)$, the DAG coherent projection is
\begin{equation}
\label{eq:dag_projection}
\hat{\mathbf a}^{\mathrm{Proj}}_{\mathrm{DAG}}(x)
\in
\arg\max_{\mathbf a\in\mathcal C_{\mathrm{SE}}^{\mathrm{DAG}}}
\sum_{v\in\mathcal V}\log \etaLoc_v(a_v\mid x).
\end{equation}
The distinction is computational. On a tree, the factor graph induced by the constraints has treewidth one, so dynamic programming gives exact linear-time decoding. On a general DAG, multiple parents create factor loops in the moralised constraint graph, and exact MAP inference need not decompose into independent subtrees.

Several practical routes are available.

\textbf{Bounded-treewidth exact decoding.}
If the moralised graph of the taxonomy has small treewidth, Eq.~\eqref{eq:dag_projection} can be solved exactly by junction-tree or variable-elimination methods. The local factors are small because each node still has only three possible actions. This would retain the exact zero-incoherence guarantee, with complexity exponential in graph treewidth rather than in taxonomy size.

\textbf{Integer-programming decoding.}
For arbitrary DAGs, coherent projection can be formulated as a small integer linear program. Introduce binary variables
\[
z_{v,a}\in\{0,1\},
\qquad
v\in\mathcal V,\ a\in\act,
\]
where $z_{v,a}=1$ means $a_v=a$, and impose
\[
\sum_{a\in\act}z_{v,a}=1
\qquad
\forall v\in\mathcal V .
\]
The three forbidden edge patterns in Eq.~\eqref{eq:dag_local_forbidden} become linear constraints:
\begin{align}
z_{p,\absent}+z_{c,\present} &\le 1,
\label{eq:dag_ilp_tax}\\
z_{p,\defer}+z_{c,\present} &\le 1,
\label{eq:dag_ilp_del}\\
z_{p,\absent}+z_{c,\defer} &\le 1,
\label{eq:dag_ilp_ded}
\end{align}
for all $(p\to c)\in\mathcal E$. The projection problem is then
\begin{equation}
\label{eq:dag_ilp_obj}
\max_{z}
\sum_{v\in\mathcal V}
\sum_{a\in\act}
z_{v,a}\log \etaLoc_v(a\mid x)
\quad
\text{s.t. Eqs.~\eqref{eq:dag_ilp_tax}--\eqref{eq:dag_ilp_ded}.}
\end{equation}
This formulation gives an exact coherent decoder for DAGs and is especially attractive for offline evaluation, moderate-sized taxonomies, or safety-critical settings where deterministic zero incoherence is required.

\textbf{Approximate message passing or local repair.}
For large DAGs, one could use loopy max-product, dual decomposition, or Lagrangian relaxation to obtain approximate coherent decodes. A simpler engineering alternative is to first decode nodewise, then iteratively repair violated edges until all constraints in Eq.~\eqref{eq:dag_local_forbidden} are satisfied. Such repairs would be heuristic unless paired with an exact optimisation certificate, but they may be useful as fast deployment-time approximations.

\subsection{TBP-style models on a DAG}

The tree version of TBP defines a top-down Markov action model in which each child has one parent. A direct DAG analogue must combine evidence from multiple parents. One natural extension is to replace the single-parent transition by the multiple-parent feasible set in Eq.~\eqref{eq:dag_gamma_intersection}. Let $\etaLoc_v(\cdot\mid x)$ be the local primitive at node $v$. Conditional on the parent actions $a_{P(v)}$, define
\begin{equation}
\label{eq:dag_tbp_conditional}
P(a_v=j\mid a_{P(v)},x)
=
\frac{
\etaLoc_v(j\mid x)\,
\mathbbm{1}\{j\in\Gamma_{\mathrm{DAG}}(a_{P(v)})\}
}{
\sum_{k\in\Gamma_{\mathrm{DAG}}(a_{P(v)})}
\etaLoc_v(k\mid x)
}.
\end{equation}
This conditional is well-defined because \(\absent\in\Gamma_{\mathrm{SE}}(a_p)\) for every parent action \(a_p\), and therefore \(\absent\in\Gamma_{\mathrm{DAG}}(a_{P(v)})\) for every parent-action configuration. Since \(\etaLoc_v(\absent\mid x)>0\) under the softmax parameterisation, the normalising denominator is strictly positive.

This is the same renormalisation principle used in the tree TBP transition kernel: local primitive mass is projected onto the admissible action face induced by the parents. For roots, $P(a_v\mid x)=\etaLoc_v(a_v\mid x)$. Since $\mathcal G$ is acyclic, these conditionals define a valid Bayesian-network-style joint action model:
\begin{equation}
\label{eq:dag_tbp_joint}
P_{\mathrm{DAG}}(\mathbf a\mid x)
=
\prod_{v\in\mathcal V}
P(a_v\mid a_{P(v)},x),
\end{equation}
where root nodes use their primitive distributions.

By construction, Eq.~\eqref{eq:dag_tbp_joint} assigns zero probability to every action vector outside $\mathcal C_{\mathrm{SE}}^{\mathrm{DAG}}$. Thus the coherent-support property of TBP extends naturally from trees to DAGs. However, marginal computation is no longer a single top-down pass, because a node's parents may be statistically dependent through shared ancestors. Exact marginal inference in Eq.~\eqref{eq:dag_tbp_joint} would again require junction-tree or variable-elimination methods, while approximate inference could use loopy belief propagation, variational methods, or sampling.

\subsection{RPO for DAG taxonomies}

Recursive Policy Optimisation also extends at the objective level. Stage I would still learn local primitives with a nodewise L2D loss:
\[
\mathcal J_{\mathrm{StageI}}
=
\sum_{v\in\mathcal V}
\mathcal L_\phi\big(\etaLoc_v(x),y_v,m_v\big).
\]
Stage II would replace the tree-composed marginals $\muGlob_v$ with marginals under the DAG joint model:
\[
\muGlob_v(a\mid x)
:=
P_{\mathrm{DAG}}(a_v=a\mid x),
\]
and optimise
\begin{equation}
\label{eq:dag_rpo}
\mathcal J_{\mathrm{RPO}}^{\mathrm{DAG}}
=
\sum_{v\in\mathcal V}
\mathcal L_\phi\big(\muGlob_v(x),y_v,m_v\big).
\end{equation}
When exact marginals are tractable, this gives a direct DAG analogue of RPO. When they are not, Eq.~\eqref{eq:dag_rpo} could be optimised with differentiable approximate inference. Alternatively, one could train against coherent MAP actions or structured surrogate losses derived from the ILP decoder, though this would require handling the non-smoothness of discrete optimisation.

\subsection{Practical interpretation}

The tree assumption in the main paper is therefore not a conceptual requirement of deferral coherence. The three central ideas all extend to DAGs:

\begin{enumerate}[leftmargin=*]
    \item deferral coherence is still defined over ternary handoff actions rather than binary labels;
    \item Selective-Exclusion still forbids positive child assertions under absent or deferred parents;
    \item coherent projection and coherent-support joint action models can still be formulated exactly.
\end{enumerate}

What changes is the inference problem. Trees permit simple linear-time dynamic programs. DAGs require either exact bounded-treewidth inference, integer-programming decoders, or approximate structured inference. Thus, extending the present framework to realistic DAG medical vocabularies is a natural direction for future work: the semantics of coherent handoff are essentially unchanged, but the computational machinery must move from tree recursions to more general structured prediction.

\section{Additional Experiments and Ablations}

\subsection{Extension to Multiple Experts}
\label{apd:multi-expert-extension}

Our primary focus in the main paper is the single-expert setting. We include this subsection to show that the same coherent hierarchical recipe extends naturally to multiple experts in practice. We emphasise that we view this as an empirical extension of the framework, rather than as a new theoretical result.

Let \(E\) denote the number of available experts. We replace the single deferral action \(\perp\) with expert-indexed deferral actions \(\perp_1,\ldots,\perp_E\), so that each node acts on the augmented action space
\[
A_E \;:=\; \{0,1,\perp_1,\ldots,\perp_E\}.
\]
The corresponding coherent Bayes rule is obtained from Eq.~\eqref{eq:bayes_coherent} by replacing the single deferral risk with expert-indexed risks
\[
\rho_v(\perp_e\mid x)=P(m_{v,e}\neq Y_v\mid x)+\lambda_{v,e},
\]
and optimising over the multi-expert coherent action set.

At each node \(v\in V\), we use a dedicated head
\[
f_{\theta}^{(v)}(x)\in\mathbb{R}^{2+E},
\]
which induces the local primitive policy
\[
\pi_v(a\mid x)
\;:=\;
\mathrm{Softmax}\!\bigl(f_{\theta}^{(v)}(x)\bigr)_a,
\qquad a\in A_E.
\]

\textbf{Stage I: local multi-expert primitives.}
In Stage I, we train each node independently with the multi-expert classifier--rejector surrogate. Let \(y_v\in\{0,1\}\) denote the ground-truth label at node \(v\), and let \(m_{v,1},\ldots,m_{v,E}\in\{0,1\}\) denote the corresponding predictions of the \(E\) experts. We define the local loss
\[
\mathcal{L}_{\mathrm{ME}}^{(v)}(x,y_v,m_{v,1:E})
=
-\log \pi_v(y_v\mid x)
-
\sum_{e=1}^{E}
\mathbb{I}(m_{v,e}=y_v)\,
\log \pi_v(\perp_e\mid x),
\]
and optimise the Stage I objective
\[
J_{\mathrm{base}}^{\mathrm{ME}}
=
\sum_{v\in V}
\mathbb{E}\!\left[
\mathcal{L}_{\mathrm{ME}}^{(v)}(x,y_v,m_{v,1:E})
\right].
\]

\textbf{Multi-expert Selective-Exclusion contract.}
We retain the same Selective-Exclusion semantics as in the single-expert setting, but treat the deferral contract as expert-agnostic. Concretely, for every edge \((\mathrm{pa}(v)\to v)\),
\[
a_{\mathrm{pa}(v)} \in \{\perp_1,\ldots,\perp_E\}
\;\Longrightarrow\;
a_v \in \{0,\perp_1,\ldots,\perp_E\}.
\]
Together with taxonomic satisfiability, this implies that any positive child assertion must have a positive parent assertion:
\[
a_v = 1 \;\Longrightarrow\; a_{\mathrm{pa}(v)} = 1.
\]
Thus, if a parent is deferred to any expert, a child may still be excluded or deferred, but it may not be asserted present.

This is an expert-agnostic handoff contract: a parent deferred to one expert may have a child deferred to another expert. A stricter same-expert contract would instead require
\[
a_{\mathrm{pa}(v)}=\perp_e
\Rightarrow
a_v\in\{0,\perp_e\},
\]
which may be preferable when responsibility for a finding family should remain with one reader. We leave this workflow choice to deployment.

\textbf{Multi-expert TBP.}
For each edge \((\mathrm{pa}(v)\to v)\), we define a transition kernel
\[
T_v(x)\in[0,1]^{(2+E)\times(2+E)},
\]
with rows and columns indexed by \(A_E=\{0,1,\perp_1,\ldots,\perp_E\}\). The kernel is specified by the following cases:
\[
(T_v)_{0,0}=1,
\qquad
(T_v)_{0,a}=0 \;\; \text{for all } a\neq 0,
\]
\[
(T_v)_{1,a}=\pi_v(a\mid x),
\qquad a\in A_E,
\]
and for each deferred parent state \(\perp_e\),
\[
(T_v)_{\perp_e,1}=0,
\qquad
(T_v)_{\perp_e,a}
=
\frac{\pi_v(a\mid x)}{Z_v(x)},
\qquad
a\in\{0,\perp_1,\ldots,\perp_E\},
\]
where
\[
Z_v(x)
=
\pi_v(0\mid x)
+
\sum_{e'=1}^{E}\pi_v(\perp_{e'}\mid x).
\]
When \(Z_v(x)=0\), we set \((T_v)_{\perp_e,0}=1\) and \((T_v)_{\perp_e,a}=0\) for all \(a\neq 0\).

Let \(\Pi_v(x)\in\Delta^{1+E}\) denote the unconditional marginal action distribution at node \(v\). As in the single-expert setting, TBP propagates these marginals top-down via
\[
\Pi_v(x)
=
\Pi_{\mathrm{pa}(v)}(x)^{\top}T_v(x).
\]
By construction, this recursion assigns zero probability to both taxonomic contradictions (positive child under absent parent) and delegation violations (positive child under deferred parent).

\textbf{Stage II: recursive policy optimisation.}
In Stage II, we fine-tune end-to-end through the TBP recursion using the same surrogate applied to the propagated unconditional marginals. Writing \(\Pi_v(a\mid x)\) for the \(a\)-coordinate of \(\Pi_v(x)\), we optimise
\[
J_{\mathrm{RPO}}^{\mathrm{ME}}
=
\sum_{v\in V}
\mathbb{E}\!\left[
-\log \Pi_v(y_v\mid x)
-
\sum_{e=1}^{E}
\mathbb{I}(m_{v,e}=y_v)\,
\log \Pi_v(\perp_e\mid x)
\right].
\]
This is the direct multi-expert analogue of the Stage II fine-tuning procedure in the main method: Stage I learns strong local primitives, and Stage II calibrates them under the inference-time hierarchy semantics induced by TBP.

\textbf{Inference.}
At test time, we decode top-down using the admissible TBP transitions rather than taking independent nodewise argmaxes of the unconditional marginals. This yields a coherent joint action assignment by construction under the expert-agnostic Selective-Exclusion contract.

\begin{figure}[t]
    \centering
    \includegraphics[width=\linewidth]{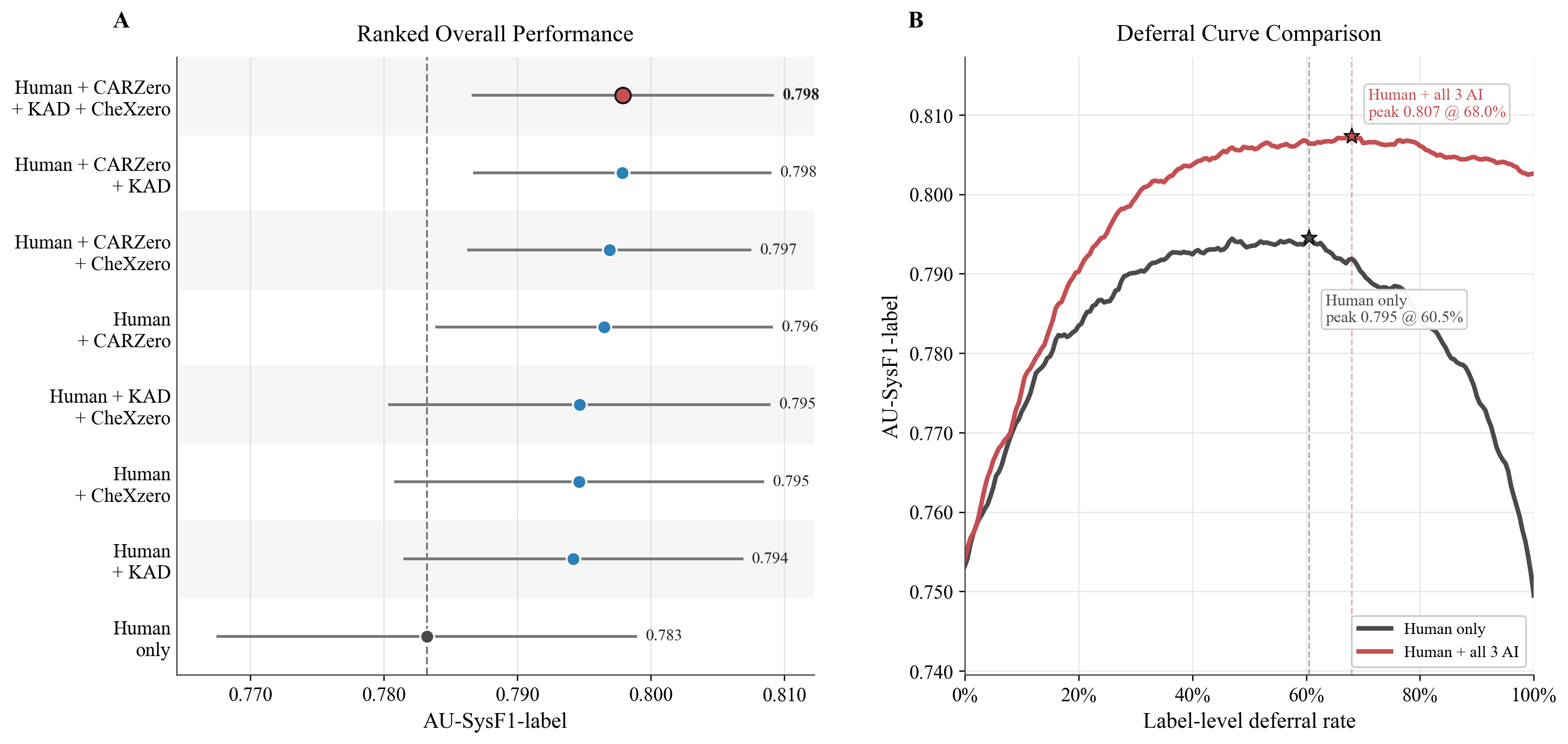}
    \caption{\textbf{CheXpert multiple-expert results.}
    \textbf{(A)} Overall AU-SysF1-label for eight human/AI expert combinations across 15 runs per combination (5 readers \(\times\) 3 seeds); points show means and bars show \(\pm 1\) SD.
    \textbf{(B)} Mean label-level deferral curves comparing Human only and Human + CARZero + KAD + CheXzero; stars mark the peak of each mean curve.}
    \label{fig:multi_l2d}
\end{figure}

\noindent\textbf{Setup.}
On CheXpert, we treat one radiologist as the designated human expert and allow deferral to a selected subset of three AI experts: CARZero \cite{lai2024carzero}, KAD \cite{zhang2023knowledge}, and CheXzero \cite{tiu2022expert}. This yields \(2^3=8\) expert combinations, ranging from Human only to Human + CARZero + KAD + CheXzero. For each chosen expert set, we instantiate the multi-expert extension above using the same taxonomy, the same Selective-Exclusion contract, and the same two-stage training recipe as in the single-expert setting. We evaluate all eight combinations across all five radiologists and three random seeds, yielding 15 runs per combination. We report AU-SysF1-label, and for the best-performing configuration we additionally plot label-level deferral curves against the human-only baseline.

\noindent\textbf{Results.}
All AI-augmented combinations outperform the human-only baseline of \(0.783\) AU-SysF1-label. The best mean performance is achieved by Human + CARZero + KAD + CheXzero at \(0.798\), narrowly ahead of Human + CARZero + KAD at \(0.798\). Figure~\ref{fig:multi_l2d}(B) further shows that the full four-expert system yields a stronger deferral curve than Human only, with a peak AU-SysF1-label of \(0.807\) at \(68.0\%\) label-level deferral, compared with \(0.795\) at \(60.5\%\) for Human only. These results suggest that the coherent hierarchical L2D framework transfers cleanly to multiple-expert routing and can improve system utility in this setting.

\subsection{Sensitivity to the Stage I objective}
\label{app:stage1_sensitivity}

\begin{table}[h]
\centering
\small
\setlength{\tabcolsep}{6pt}
\caption{\textbf{Sensitivity to the Stage I objective on VinDr-CXR expert R10: utility.}
We compare alternative Stage I objectives and post-Stage I training strategies over five random seeds for reader \texttt{R10}. Values are mean \(\pm\) SD. Higher is better for all metrics.}
\label{tab:stage1_sensitivity_r10_utility}
\begin{tabular}{lccc}
\toprule
Method & AU-SysF1(I) \(\uparrow\) & AU-SysF1(L) \(\uparrow\) & AU-SysBalAcc \(\uparrow\) \\
\midrule
\multicolumn{4}{l}{\textit{RPO}} \\
MC + RPO         & \textbf{0.5199 \(\pm\) 0.0050} & 0.7892 \(\pm\) 0.0009 & 0.8530 \(\pm\) 0.0018 \\
Focal + RPO      & 0.5198 \(\pm\) 0.0052 & \textbf{0.7899 \(\pm\) 0.0024} & 0.8536 \(\pm\) 0.0023 \\
BR-L2D + RPO     & 0.5191 \(\pm\) 0.0052 & 0.7835 \(\pm\) 0.0032 & 0.8503 \(\pm\) 0.0032 \\
\addlinespace[2pt]
\multicolumn{4}{l}{\textit{Continue}} \\
MC + Continue    & 0.5154 \(\pm\) 0.0047 & 0.7831 \(\pm\) 0.0025 & 0.8434 \(\pm\) 0.0025 \\
Focal + Continue & 0.5131 \(\pm\) 0.0055 & 0.7818 \(\pm\) 0.0028 & 0.8431 \(\pm\) 0.0043 \\
BR (cont.)       & 0.5145 \(\pm\) 0.0056 & 0.7817 \(\pm\) 0.0021 & 0.8422 \(\pm\) 0.0012 \\
\addlinespace[2pt]
\multicolumn{4}{l}{\textit{Projection}} \\
MC + Projection     & 0.5170 \(\pm\) 0.0049 & 0.7529 \(\pm\) 0.0112 & 0.8541 \(\pm\) 0.0035 \\
Focal + Projection  & 0.3786 \(\pm\) 0.0072 & 0.4736 \(\pm\) 0.0107 & 0.7221 \(\pm\) 0.0101 \\
BR-L2D + Projection & 0.5158 \(\pm\) 0.0041 & 0.7691 \(\pm\) 0.0077 & 0.8449 \(\pm\) 0.0053 \\
\addlinespace[2pt]
\multicolumn{4}{l}{\textit{Stage I only}} \\
MC Pretrain      & 0.5168 \(\pm\) 0.0056 & 0.7589 \(\pm\) 0.0083 & 0.8515 \(\pm\) 0.0044 \\
Focal Pretrain   & 0.5117 \(\pm\) 0.0052 & 0.7596 \(\pm\) 0.0060 & \textbf{0.8602 \(\pm\) 0.0048} \\
BR-L2D           & 0.5148 \(\pm\) 0.0038 & 0.7683 \(\pm\) 0.0074 & 0.8439 \(\pm\) 0.0057 \\
\bottomrule
\end{tabular}
\end{table}

\begin{table}[h]
\centering
\small
\setlength{\tabcolsep}{5pt}
\caption{\textbf{Sensitivity to the Stage I objective on VinDr-CXR expert R10: edge-weighted incoherence.}
Values are mean \(\pm\) SD over five random seeds for reader \texttt{R10}. Lower is better for all metrics. Tax = taxonomic contradiction, Ded = deductive defect, Del = delegation violation, and Any = any incoherence.}
\label{tab:stage1_sensitivity_r10_incoherence}
\begin{tabular}{lcccc}
\toprule
Method & Tax \(\downarrow\) & Ded \(\downarrow\) & Del \(\downarrow\) & Any \(\downarrow\) \\
\midrule
\multicolumn{5}{l}{\textit{RPO}} \\
MC + RPO         & 0.0000 \(\pm\) 0.0000 & 0.0006 \(\pm\) 0.0002 & 0.0001 \(\pm\) 0.0000 & 0.0007 \(\pm\) 0.0002 \\
Focal + RPO      & 0.0000 \(\pm\) 0.0000 & 0.0007 \(\pm\) 0.0001 & 0.0013 \(\pm\) 0.0002 & 0.0020 \(\pm\) 0.0003 \\
BR-L2D + RPO     & 0.0000 \(\pm\) 0.0000 & 0.0002 \(\pm\) 0.0001 & 0.0001 \(\pm\) 0.0001 & 0.0003 \(\pm\) 0.0001 \\
\addlinespace[2pt]
\multicolumn{5}{l}{\textit{Continue}} \\
MC + Continue    & 0.0000 \(\pm\) 0.0000 & 0.0780 \(\pm\) 0.0041 & 0.0114 \(\pm\) 0.0020 & 0.0894 \(\pm\) 0.0044 \\
Focal + Continue & 0.0000 \(\pm\) 0.0000 & 0.0755 \(\pm\) 0.0049 & 0.0129 \(\pm\) 0.0024 & 0.0885 \(\pm\) 0.0066 \\
BR (cont.)       & 0.0000 \(\pm\) 0.0000 & 0.0639 \(\pm\) 0.0143 & 0.0144 \(\pm\) 0.0039 & 0.0783 \(\pm\) 0.0177 \\
\addlinespace[2pt]
\multicolumn{5}{l}{\textit{Projection}} \\
MC + Projection     & 0.0000 \(\pm\) 0.0000 & 0.0000 \(\pm\) 0.0000 & 0.0000 \(\pm\) 0.0000 & 0.0000 \(\pm\) 0.0000 \\
Focal + Projection  & 0.0000 \(\pm\) 0.0000 & 0.0000 \(\pm\) 0.0000 & 0.0000 \(\pm\) 0.0000 & 0.0000 \(\pm\) 0.0000 \\
BR-L2D + Projection & 0.0000 \(\pm\) 0.0000 & 0.0000 \(\pm\) 0.0000 & 0.0000 \(\pm\) 0.0000 & 0.0000 \(\pm\) 0.0000 \\
\addlinespace[2pt]
\multicolumn{5}{l}{\textit{Stage I only}} \\
MC Pretrain      & 0.0004 \(\pm\) 0.0002 & 0.0777 \(\pm\) 0.0264 & 0.0001 \(\pm\) 0.0000 & 0.0781 \(\pm\) 0.0264 \\
Focal Pretrain   & 0.0000 \(\pm\) 0.0000 & 0.0020 \(\pm\) 0.0015 & 0.0040 \(\pm\) 0.0011 & 0.0060 \(\pm\) 0.0025 \\
BR-L2D           & 0.0000 \(\pm\) 0.0000 & 0.0433 \(\pm\) 0.0180 & 0.0142 \(\pm\) 0.0039 & 0.0575 \(\pm\) 0.0185 \\
\bottomrule
\end{tabular}
\end{table}

Tables~\ref{tab:stage1_sensitivity_r10_utility} and~\ref{tab:stage1_sensitivity_r10_incoherence} test whether the RPO gains depend on the particular Stage I objective. On VinDr-CXR expert \texttt{R10}, RPO is robust across the tested initialisations: MC + RPO, Focal + RPO, and BR-L2D + RPO all achieve strong utility while reducing edge-weighted incoherence to near zero.

The comparison also separates RPO from two alternative explanations. Simply continuing Stage I training does not remove incoherence and gives weaker utility than RPO, showing that the effect is not just longer optimisation. Projection guarantees zero incoherence for every Stage I objective, but can reduce utility substantially when the local primitives are poorly aligned with the coherent decoder, most notably for Focal + Projection. Overall, Stage I provides a useful initialisation, while Stage II RPO is the component that aligns the learned primitives with the TBP inference semantics and improves the utility--coherence trade-off.

\end{document}